\documentclass[a4paper]{article}
\usepackage{geometry}
\usepackage{times}
\usepackage{epsfig}
\usepackage{graphicx}
\usepackage{amsmath}
\usepackage{amssymb}
\usepackage{amsthm}
\usepackage[tight,footnotesize]{subfigure}
\usepackage{algorithm}
\usepackage{algorithmic}
\usepackage{array}
\usepackage{multirow}
\usepackage[noblocks]{authblk}
\usepackage{url}

\newcommand{\mb}[1]{\mathbf{#1}}
\newcommand{\tb}[1]{\textbf{#1}}
\DeclareMathOperator{\tr}{tr} \DeclareMathOperator{\rank}{rank}
\DeclareMathOperator{\diag}{diag}
\DeclareMathOperator{\st}{s.t.}

\newtheorem{theorem}{Theorem}%[section]

\newtheorem{proposition}[theorem]{Proposition}
\newtheorem{corollary}[theorem]{Corollary}
\newtheorem{definition}[theorem]{Definition}

\begin{document}

\title{Spectral Sparse Representation for Clustering: Evolved from PCA, K-means, Laplacian Eigenmap, and Ratio Cut}

%\author{\name Zhenfang Hu \email fancij@zju.edu.cn \\
%%\addr College of Computer Science and Technology\\
%%       Zhejiang University\\
%%       Hangzhou, 310027, China
%       \AND
%       \name Gang Pan \email gpan@zju.edu.cn \\
%       \addr College of Computer Science\\
%       Zhejiang University\\
%       Hangzhou, 310027, China}
%
%\editor{xxx}

\author[*]{Zhenfang~Hu}
\author[*]{Gang~Pan}
\author[]{Yueming~Wang}
\author[*]{Zhaohui Wu}

\affil[*]{College of Computer Science and Technology, Zhejiang
University, China, \authorcr fancij@zju.edu.cn, gpan@zju.edu.cn, wzh@zju.edu.cn}
\affil[]{Qiushi Academy for Advanced Studies, Zhejiang University, China, \authorcr ymingwang@zju.edu.cn}

\maketitle

\begin{abstract}%   <- trailing '%' for backward compatibility of .sty file
Dimensionality reduction, cluster analysis, and sparse representation are basic components in machine learning.
However, their relationships have not yet been fully investigated. In this paper, we
find that the spectral graph theory underlies a series of these elementary methods and can unify them into a
complete framework. The methods include PCA, K-means, Laplacian eigenmap (LE), ratio cut (Rcut), and a new sparse
representation method developed by us, called spectral sparse representation (SSR). Further, extended relations to conventional over-complete sparse representations (e.g., method of optimal directions, KSVD), manifold learning (e.g., kernel PCA, multidimensional scaling, Isomap, locally linear embedding), and subspace clustering (e.g., sparse subspace clustering, low-rank representation) are incorporated. We show that, under an ideal condition from the spectral graph theory, PCA, K-means, LE, and Rcut are unified together. And when the condition is relaxed, the unification evolves to SSR, which lies in the intermediate between PCA/LE and K-mean/Rcut. An efficient algorithm, NSCrt, is developed to solve the sparse codes of SSR. SSR combines merits of both sides: its sparse codes reduce dimensionality of data  meanwhile revealing cluster structure. For its inherent relation to cluster analysis, the codes of SSR can be directly used for clustering. Scut, a clustering approach derived from SSR reaches the state-of-the-art performance in the spectral clustering family. The one-shot solution obtained by Scut is comparable to the optimal result of K-means that are run many times. Experiments on various data sets demonstrate the properties and strengths of SSR, NSCrt, and Scut.
\end{abstract}

\tb{Keywords:}
  sparse representation, spectral graph, dimensionality reduction, cluster analysis,
  PCA, K-means, spectral clustering, Laplacian eigenmap.

\section{Introduction}\label{sec:introduction}
As the rise of information age, we are overwhelmed by a large amount
of data generated daily, e.g., images, videos, speeches, text,
financial data, and biomedical data. But the information contained
in these data is usually not explicit in their original forms. Data understanding becomes gradually urgent. In this
paper, we focus on unsupervised learning, which tries to find
hidden structure in unlabeled data. The intrinsic
structure of data is often explored by representing the data in another form. Typically used
methods include dimensionality reduction, cluster analysis, and sparse representation, which are among the cornerstones of machine learning.

However, the relationships among these methods have not been fully explored. We believe that the basic
parts of machine learning deserve extensive investigation. This paper devotes to an attempt in this direction.

\subsection{Dimensionality Reduction and Cluster Analysis}
Dimensionality reduction and cluster analysis are two of the most traditional unsupervised learning methods, having wide-spread applications.

\tb{Dimensionality reduction}. It aims at representing data by low-dimensional codes. On the one hand, it saves storage, considering many data we encounter are of high dimensions whose intrinsic dimensions, however, are much lower. On the other hand, noise may be reduced and the structure of data becomes prominent. The codes will preserve the relation of data as much as possible, e.g., distance between data points. The widely applied methods include: principal component analysis (PCA)
\cite{jolliffe2002principal}, multidimensional scaling (MDS) \cite{cox2001multidimensional}, kernel PCA \cite{scholkopf1998nonlinear}, nonnegative
matrix factorization (NMF) \cite{lee1999learning},
Isomap \cite{tenenbaum2000global}, locally linear embedding (LLE) \cite{roweis2000nonlinear},
Laplacian eigenmap (LE) \cite{belkin2003laplacian}, and locality
preserving projections (LPP) \cite{he2003locality}.

\tb{Cluster analysis}. It
tries to partition the data into disjoint
groups such that similar data points are assigned to the same group and
dissimilar data points are separated into different groups. In this
way, the cluster structure of data is revealed. Various methods have
been proposed, for example, centroid-based K-means clustering
\cite{macqueen1967some} and its distribution-based version:
clustering via Gaussian mixture models (GMM)
\cite{bishop2006pattern}, connectivity-based hierarchical
clustering \cite{hastie2005elements}, density-based DBSCAN
\cite{ester1996density}, and graph-based spectral clustering
\cite{von2007tutorial, dhillon2004unified} such as ratio cut
(Rcut) \cite{chan1994spectral}, normalized cut (Ncut)
\cite{shi2000normalized, yu2003multiclass, ng2002spectral}.

The two kinds of methods appear distinct: dimensionality reduction is concerned with fidelity where the data relation should be preserved faithfully, while cluster analysis focuses on semantic where the classes of data should be made clear. However, in a broad sense, both of them can be seen as code-based data representation methods. For cluster analysis, each data point is represented by one cluster, and its codes are indicator vector consisting of zeros and a 1, the index of which indicates the cluster membership. The codes of dimensionality reduction are compact, due to low dimensionality and fidelity, while the codes of cluster analysis are sparse, due to only one nonzero.

The above methods, according to the form of input data they work
with, can be categorized into two types. One works with original
data, e.g., PCA and K-means. The other works with similarity matrix,
which stores the pair-wise relations of data, e.g., LE and spectral
clustering. We can call the second type kernel methods, since
usually the similarity matrix is constructed by some kernel function
\cite{von2007tutorial, shawe2004kernel}.

Among the various methods of dimensionality reduction and cluster
analysis, PCA, K-means, LE, and spectral clustering are the
representative ones. They are based on principled mathematical
formulations and are effective in practice. In this paper, we
focus on these four methods. They appear very different at first
sight. Some of them are dimensionality reduction methods while the
others are cluster analysis methods. Some of them work with original
data while the others work with similarity matrix. In fact, they are
closely related. Some pair-wise connections have been found in the
literature, including LE and Ncut \cite{belkin2003laplacian}, PCA and K-means \cite{ding2004k}, spectral
clustering and K-means \cite{dhillon2004unified}. However, the relations remain pair-wise,
they are not yet rigorously integrated into a unified framework.

\subsection{Sparse Representation (SR)}
Sparse representation is a more recently developed code-based representation method, which represents data with a dictionary and sparse codes. In existing SRs, the dictionary is generally over-complete, i.e., the size of dictionary (number of words/atoms) is larger than the data dimension. The codes are called sparse for only a few nonzero entries exist or dominate.

In the past several years, the communities of signal processing,
computer vision, and pattern recognition had witnessed the great
success of various SRs \cite{bruckstein2009sparse, elad2010role,
wright2010sparse, rubinstein2010dictionaries}, e.g., compressed
sensing theory \cite{donoho2006compressed, candes2008introduction,
bruckstein2009sparse}, over-complete dictionary learning, including
\cite{olshausen1996emergence}, method of optimal directions (MOD) \cite{Engan1999Method}, KSVD \cite{aharon2006img}. In
pattern recognition, sparse codes are frequently
used as features for classification, e.g., face recognition
\cite{wright2009robust}, image classification
\cite{yang2009linear, boureau2010learning, coates2011importance}.

Almost all prevalent SRs are over-complete SRs (OSRs). Except the special application in compressed sensing, OSRs are
not related to dimensionality reduction. On the other hand, comparing with the popular applications in classification, the applications in clustering are few, the most well-known one is sparse subspace clustering (SSC) \cite{elhamifar2013sparse}, which deals with clusters lying in a union of subspaces.

\subsection{Our Work}
In this paper, we deepen and complete the relations among PCA,
K-means, LE,\footnote{The LE we refer to hereafter is slightly
different from the original one \cite{belkin2003laplacian}. The
definition will be given in later section. We call the original LE
to be normalized LE for a reason that will be clear later.} and
Rcut, so that they are unified together. Then, a new SR is developed,
called spectral sparse representation (SSR), which evolves from the
unification of the four methods, bearing
inherent relations to dimensionality reduction and cluster analysis. The spectral graph theory underlies all the methods and integrates them into a framework.

The idea can be briefly described as follows. PCA, LE, K-means, and Rcut are written into forms working with the same matrix, the Laplacian matrix from the spectral graph theory. The low-dimensional codes of PCA and LE are the leading eigenvectors of this matrix. When an ideal graph condition is met, which implies perfectly separable clusters exist in the data, the indicator vectors of K-means and Rcut become the leading eigenvectors, thus PCA, K-means, LE, and Rcut are unified. When the condition is relaxed, the leading eigenvectors become noisy indicator vectors, which are sparse and still contain some cluster information of the data, the unification then evolves to SSR.

The first step to the unification is to establish the bilateral
conversion between PCA and LE, and that between K-means and Rcut. In
effect, we convert the objective of one method into the form of the
other. It turns out that PCA and K-means are equivalently working
with a similarity matrix built by the Gram matrix of data, i.e.,
linear kernel matrix. So they can be converted to the forms of
kernel methods. Roughly speaking, PCA and K-means are linear LE and
linear Rcut respectively. Conversely, the similarity matrix used by
LE and Rcut can be converted to a Gram matrix of some virtual data.
Thus the objectives of LE and Rcut can be written into the forms of
PCA and K-means respectively. Our theory includes two
versions: one works with the original data, called linear
version; the other works with the similarity matrix, called kernel
version.

The second step to the unification is to bridge the link between Rcut and LE under an ideal graph condition. This is
done through spectral graph theory \cite{chung1997spectral, mohar1991laplacian, mohar1997some}, as
\cite{belkin2003laplacian} did. The solution of LE is a set of eigenvectors corresponding to the zero eigenvalues, and
that of Rcut is a set of indicator vectors corresponding to the best partition. Under the condition, the two sets of
vectors span the same subspace, and they are equivalent in the sense of a rotation transform. In brief, the indicator
vectors become the leading eigenvectors of LE, that is Rcut and LE are unified. Based on the results, the equivalence between PCA and K-means is automatically built. Therefore, the unification of PCA, K-means, LE, and Rcut is established.

The unification evolves to SSR when the ideal condition is relaxed,
i.e., in a noisy case. In this case, the leading eigenvectors of PCA
and LE span the same subspace as the noisy version of the
indicator vectors. We call the noisy indicator vectors sparse codes,
since the codes of each data point is usually dominated by one
entry, and the ``noise'' indeed quantitatively reflects the
overlapping status of the clusters. We develop an algorithm, called NSCrt,
to find a rotation matrix, by which the eigenvectors turn into the
noisy indicator vectors. Furthermore, as an application of SSR, a
clustering algorithm, sparse cut (Scut), is developed, which
determines the cluster membership of each data point by simply
checking the maximal entry of its code vector.

SSR can establish extended relations to a series of methods, including 1) OSRs, e.g., \cite{olshausen1996emergence}, MOD, KSVD, 2) manifold learning, e.g., kernel PCA, MDS, Isomap, LLE, and 3) subspace clustering, e.g., SSC, low-rank representation (LRR) \cite{liu2013robust}.

A diagram of the relations between PCA, linear SSR, OSR, and K-means is shown in Figure~\ref{fig:diagram} (it will be discussed in details in Section~\ref{sec:OSR vs SSR}). The codes of the methods change from compact to sparse, and then to extreme sparse, leading to clustering. The main
framework of our work is shown in Figure~\ref{fig:framework}.

\begin{figure}[h]
\center{
\includegraphics[width=12cm]{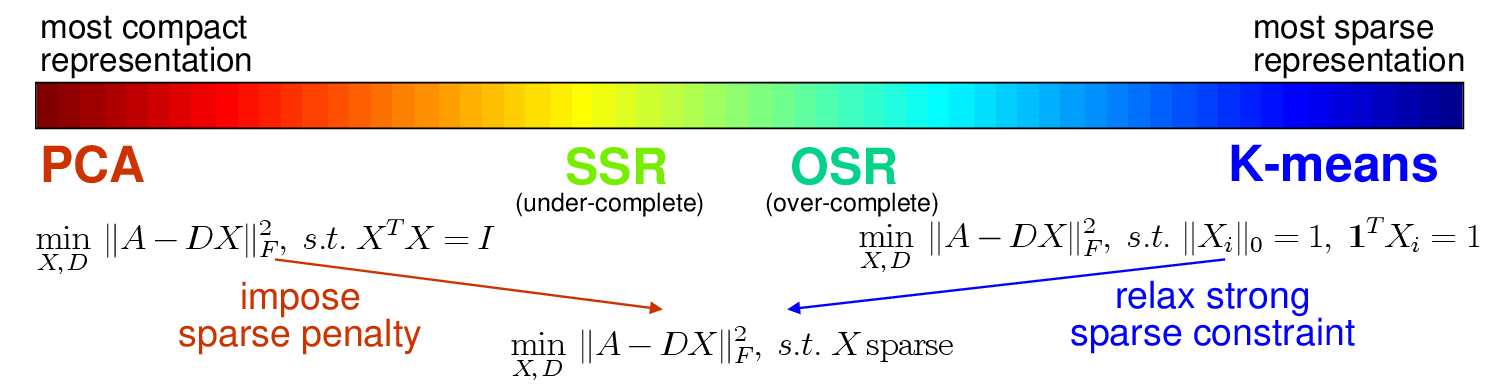}
}\caption{A diagram of the relations between dimensionality reduction, sparse representation, and cluster analysis,
represented by PCA, linear SSR, OSR, and K-means.}\label{fig:diagram}
\end{figure}

\begin{figure*}[htbp]
\centering{ \subfigure[linear
version]{\label{fig:framework_ori}\includegraphics[width=7.4cm]{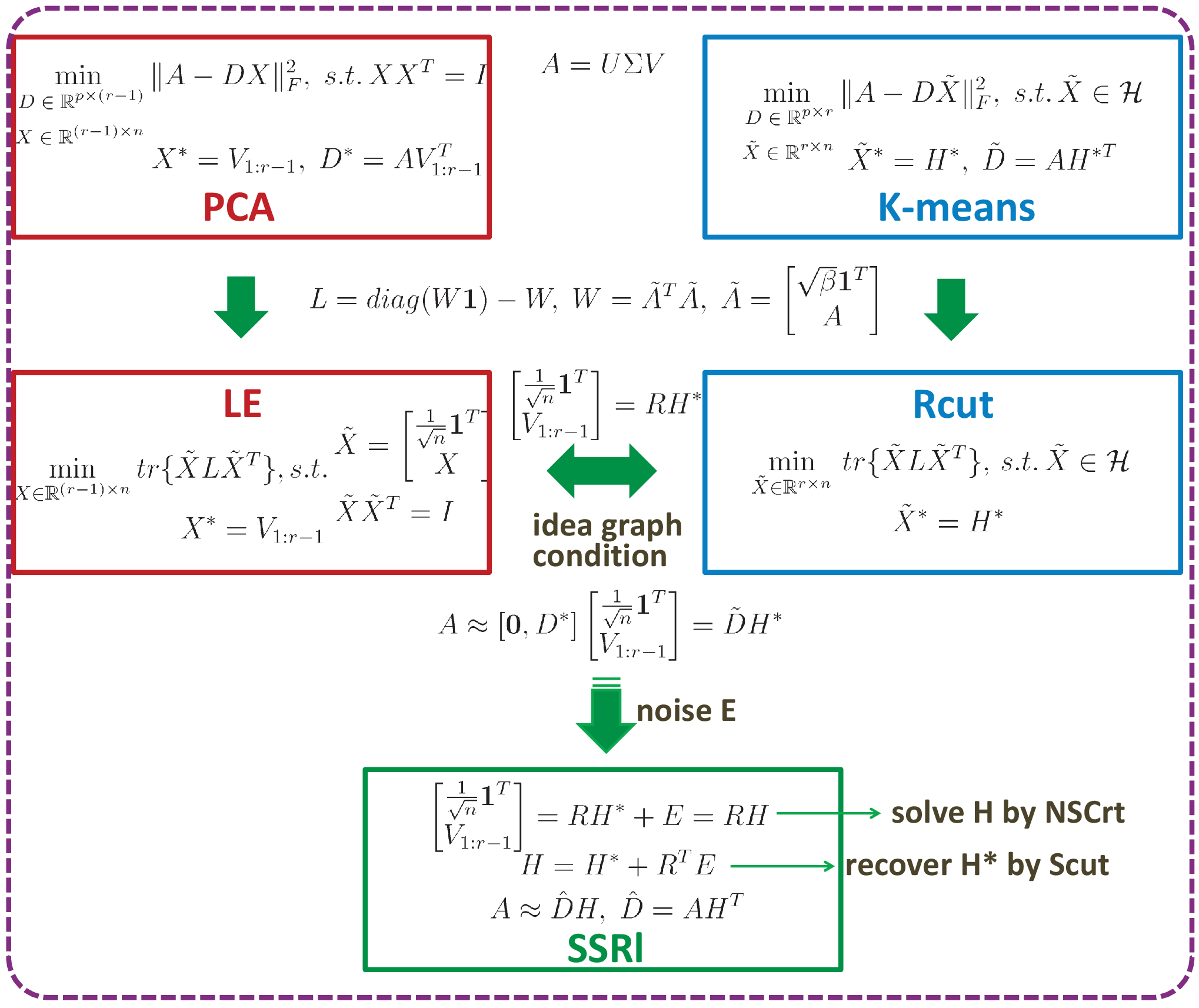}}\hspace{3mm} \subfigure[kernel
version]{\label{fig:framework_ker}\includegraphics[width=7.4cm]{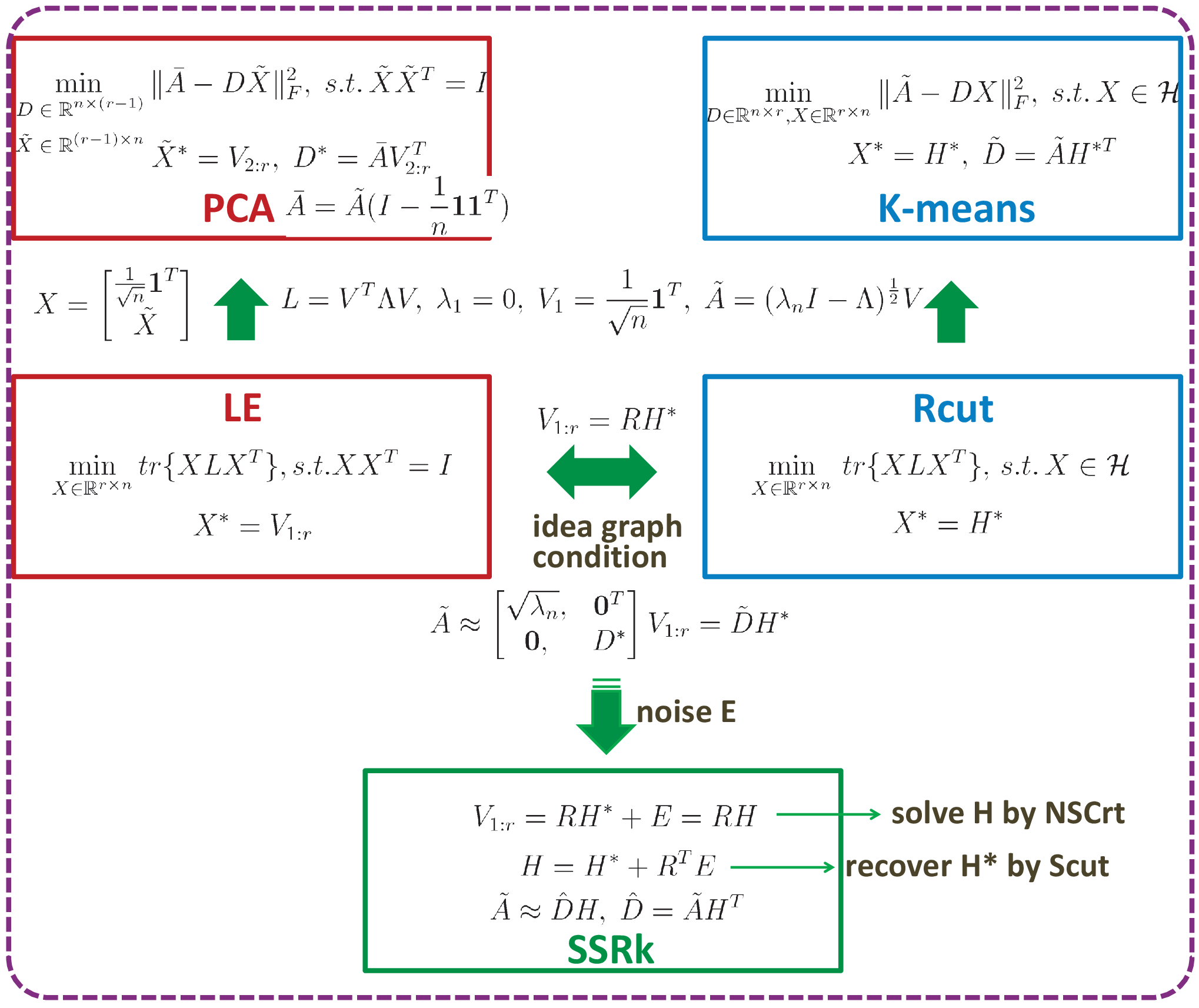}} } \caption{Main framework of
the paper.}\label{fig:framework}
\end{figure*}

The features and contributions of the paper are four-fold:

\begin{enumerate}
  \item A spectral graph theory-based framework unifying dimensionality reduction, cluster analysis, and sparse representation has been established, including inherent relations between PCA, K-means, LE, Rcut, SSR, and extended relations to OSRs (\cite{olshausen1996emergence}, MOD, KSVD), manifold learning (kernel PCA, MDS, Isomap, LLE), and subspace clustering (SSC, LRR).

  \item A new sparse representation, SSR, is developed. It is inherently related to dimensionality reduction and
  cluster analysis, and shares broad relations to many other methods. Lying in the intermediate between dimensionality reduction and cluster analysis, SSR combines merits of both sides. It achieves dimensionality reduction with the same fidelity as PCA/LE meanwhile revealing cluster structure of data.
      In contrast to the hard clustering nature of K-means/Rcut, SSR is soft and descriptive: it reveals the
      underlying clusters, and also describes the overlapping status of them. In contrast to OSRs, the sparse
      codes of SSR can be directly used for clustering, and the sparsity in SSR is implicitly determined by the
      data structure rather than imposed explicitly. If the clusters overlap less, the codes become sparser,
      and vice versa. Finally, SSR has two versions, a linear version, which is under-complete, complementing
      to OSRs, and a more powerful kernel version. For both versions, sparse codes of new data that are out of the
      sample set can be easily obtained.

  \item An algorithm, NSCrt, is developed to solve the sparse codes of SSR. It is simple yet efficient, having a linear computational complexity about the data size. Experimental results demonstrated that, it can effectively recover the underlying solutions.

  \item A clustering algorithm, Scut, is derived from SSR, which performs clustering by checking the maximal entry of each sparse code-vector. Owing to the good performance of NSCrt, Scut outperforms K-means based spectral clustering methods, which depend on initialization and easily get trapped in local minima. The one-shot solution obtained by Scut is comparable to the optimal result of K-means that are run many times.
\end{enumerate}

The rest of the paper is arranged as follows. Section~\ref{sec:related work} introduces the related work.
Section~\ref{sec:review} reviews PCA, K-means, Rcut, and LE, and introduces the most appropriate forms for the
unification. Section~\ref{sec:bilateral conversions} establishes the bilateral conversions between PCA and LE,
K-means and Rcut. Section~\ref{sec:unification} presents the equivalence relation of LE and Rcut, and unifies the
four methods. Section~\ref{sec:SSR} proposes SSR, out-of-sample extensions, NSCrt, and Scut.
Section~\ref{sec:experiments} shows the experimental results. Section~\ref{sec:discussions} investigates the extended relations to other methods. The paper is concluded with Section~\ref{sec:future work}.

\tb{Notations}. The major notations used are listed in
Table~\ref{tab:notations}.

\begin{table*}[t]
\caption{Notations.}\label{tab:notations} \vskip 0.1in
\begin{center}
%\vskip -0.1in
\begin{scriptsize} %
\begin{tabular}{|c||m{8.7cm}|}

\hline Notation & Interpretation \\\hline\hline

$\mb{1}$ & A vector of uniform value 1.\\\hline

$\text{null}(A)$ & Null space of matrix $A$.\\\hline

$\text{span}(A)$ & Subspace spanned by columns (or rows, when clear
under the context) of $A$.\\\hline

$\diag(v)$ & A diagonal matrix with the diagonal being vector
$v$.\\\hline

$A=[A_1,\dots,A_n]\in \mathbb{R}^{p\times n}$ & Data matrix with $n$
samples of dimension $p$ arranged column-wise.\\\hline

$F\in \mathbb{R}^{K\times n}$  & Indicator vectors/matrix. The $k$th
row $F_k$ is the indicator vector of the $k$th cluster $C_k$:
$F_{ki}=1$ if $A_i\in C_k$, and $F_{ki}=0$ otherwise. $F_i$ denote
the $i$th column of $F$. Note that $FF^T=\diag([n_1,\dots,n_K])$,
i.e., $F_{k_1}F_{k_2}^T=n_k$, where $n_k$ is the size of $C_k$, if
$k_1=k_2$; and $F_{k_1}F_{k_2}^T=0$ otherwise.\\\hline

$\mathcal{F}=\{F\in \mathbb{R}^{K\times
n}|\,\|F_i\|_0=1,\mb{1}^TF_i=1, \forall i\}$ & All possible
$K$-partitions of $n$ samples.\\\hline

$H^*\in \mathbb{R}^{K\times n}$  & Normalized indicator
vectors/matrix. $H^*=(FF^T)^{-1/2}F$. It implies
$H^*_k=\frac{1}{\sqrt{n_k}}F_k$, $\|H^*_k\|_2=1$, and $H^*H^{*T}=I$.
In SSR, we use $H$ to denote the noisy version of $H^*$, i.e.,
sparse codes.
\\\hline

$\mathcal{H}=\{H^*|\,H^*=(FF^T)^{-1/2}F,F\in \mathcal{F}\}$ &
Normalized version of $\mathcal{F}$.\\\hline

$V_{1:r}$ & The leading $r$ eigenvectors if $V$ is an eigenvector
matrix.\\\hline

\end{tabular}
\end{scriptsize}
\end{center}
\vskip -0.1in
\end{table*}
%-------------------------------------------------------------------------
\section{Related Work}\label{sec:related work}
First of all, we should clarify that the unifying framework we investigate is different to many frameworks or unified
views in the literature, e.g., \cite{bengio2004out, Yan2007Graph, de2012least}, which usually focus on studying
the general objective function shared by a set of methods, whose contents as well as solutions may be unrelated to each
other. Our framework investigates the inherent relations, e.g., the equivalences of solutions and the conditions when
they hold. Besides, although there are combined applications of dimensionality reduction, cluster analysis, and SR, we
do not consider heuristic hybrid models, e.g., different methods are added together or form a pipeline to accomplish
certain task. We will introduce the related work below: pair-wise relations of PCA, K-means, LE, and spectral
clustering; OSRs; and clustering methods close to Scut.

\subsection{Pair-wise Relations of PCA, K-means, LE, and Spectral
Clustering}\label{sec:related work:unification}

\tb{Normalized LE and Ncut} \cite{belkin2003laplacian}: their objectives can be written into similar forms except that
spectral clustering additionally requires the variable to be an indicator matrix for the purpose of clustering. The
solution of normalized LE is a set of generalized eigenvectors, while that of Ncut is indicator vectors. When a
condition is met, which is in fact the ideal graph condition in this paper, the indicator vectors become the leading
generalized eigenvectors, so the two methods become equivalent. When the condition is nearly met, the generalized
eigenvectors are a rotation of some noisy indicator vectors. Thus, clustering is usually done by postprocessing the
generalized eigenvectors, e.g., applying K-means on the generalized eigenvectors. In addition to Normalized LE and
Ncut, there is a more elementary counterpart, LE and Rcut, which will make the connections to PCA and K-means
straightforward. However, they have been largely overlooked. In this paper, we will focus on LE and Rcut.

\tb{PCA and K-means} \cite{ding2004k}: their objectives can be written into similar forms except that PCA drops a
constant component and K-means constrains a variable to be discrete indicator. PCA is thus viewed as a relaxation
of K-means. The relaxation relation indicates that clustering may be done by applying K-means on the normalized
principal components of PCA, i.e., some leading eigenvectors. It is easy to see that this pair shares many common
features with the first pair. But this analogy may be ignored. In this paper, we will convert PCA and K-means to LE and
Rcut respectively, and then with the help of spectral graph theory, it is discovered that PCA and K-means are
inherently related, beyond sharing similar forms. They are even exactly equivalent under an ideal graph condition. It
further reveals that the heuristic clustering scheme of applying K-means on the normalized PCs is in fact a linear Rcut
algorithm. Similar clustering strategies that apply dimensionality reduction as preprocessing have been studied
\cite{Cohen2014Dimensionality} and frequently applied, however, the preprocessing is mostly considered for the reasons
of efficiency, denoising, etc. rather than inherent relation.

\tb{Spectral clustering and kernel K-means}
\cite{dhillon2004unified}: the trace minimization or maximization
form of spectral clustering can be converted to the trace
maximization form of (weighted) kernel K-means, so spectral
clustering can be solved by (weighted) kernel K-means.
However, on the one hand, the conversion stops at kernel K-means and does not go
further to the dictionary-based representation. In this
paper, we will show that the dictionary form can enrich our
understanding of spectral clustering. On the other hand, the
conversion from K-means to spectral clustering is
absent.\footnote{Among the spectral clustering, Ncut corresponds to
weighted K-means, while Rcut corresponds to K-means. The conversion
from K-means to Rcut is feasible. However, in the same rigorous
sense, the conversion from weighted K-means to Ncut seems
infeasible. Hence, we focus on K-means and Rcut in this paper.} We
will supplement this part and show that this conversion will lead to
a linkage between dimensionality reduction and cluster analysis.

The above relations remain pair-wise and have not yet been
integrated into a unified framework.

\subsection{Over-complete Sparse Representation (OSR)}\label{sec:related work:SR}
The representative work of OSRs include compressed sensing theory, sparse and redundant dictionary
learning and their applications in pattern recognition.

Compressed sensing theory
\cite{donoho2006compressed, candes2008introduction,
bruckstein2009sparse} states that if a signal is sparse under some
basis, then far fewer measurements than Shannon theorem indicates
are required to reconstruct the signal. Over-complete dictionary
is designed and should satisfy good mutual coherence property, i.e.,
atoms of the dictionary are not close to each other. If the signal
is intrinsically sparse enough, the sparse codes are guaranteed to
be solved exactly by an $\ell_0$-norm based greedy algorithm, orthogonal matching pursuit
(OMP) \cite{pati1993orthogonal}, or by an $\ell_1$-norm based
convex optimization, basis pursuit (BP) \cite{chen1998atomic}. In
this case, the signal can be exactly reconstructed. The compressed
sensing theory is concerned with signal compression and
reconstruction, in this paper, we will focus on the semantic aspect, i.e., cluster structure.

%It was found that the atoms of the dictionary are similar to the receptive fields in primary visual cortex.

\cite{olshausen1996emergence}, MOD \cite{Engan1999Method}, and KSVD \cite{aharon2006img} learn an over-complete
dictionary and sparse codes, based on different sparsity penalties and optimizations. Taking advantage of the
reconstruction power of SR, KSVD has achieved good performance on a series of image processing problems
\cite{elad2010role}, e.g., image compression, denoising, deblurring, inpainting, and super resolution. The model and
optimization process of KSVD are generalized from those of K-means \cite{aharon2006img}. However, effective
application of KSVD in clustering is hardly found.

In pattern recognition, after OSRs are solved, the sparse codes are frequently used as features for classification,
e.g., face recognition \cite{wright2009robust}, image classification \cite{yang2009linear, boureau2010learning,
coates2011importance}. \cite{gao2010kernel} developed a kernel SR extending the work of \cite{wright2009robust} and
\cite{yang2009linear}. The kernel trick is applied on the dictionary rather than the data. Since the dictionary is
unknown beforehand, the optimization is complex. On the other hand, there are few work on clustering. Except in model
combination way, e.g., \cite{ramirez2010classification,Dong2011Sparsity}, the most famous one may be sparse subspace
clustering (SSC) \cite{elhamifar2013sparse}, which deals with data that lie in a union of independent low-dimensional
linear/affine subspaces, with each subspace corresponding to a cluster.

In summary, OSRs are mainly
applied to signal compression, image representation, and image
classification. They are not related to cluster analysis and
dimensionality reduction. The applications in clustering are few.
%OSRs were originally derived from signal (image) processing domain where signal compression and reconstruction are the
%concerns. There, sparsity and over-completeness are vital for the theoretical guarantee and empirical success. However,
%when they are applied to pattern recognition, class semantic becomes the prime concerns. The justifications of sparsity
%and over-completeness become less obvious. For classification, controversies once arose
%\cite{rigamonti2011sparse,shi2011face,zhang2011sparse}. For clustering, the applications are few.
%Finally, almost all
%SRs insist on over-completeness, we suggest that it should be reexamined. First, although over-completeness ensures up
%to perfect reconstruction, reconstruction is not the focus of classification and clustering; second, according to the
%objectives of most OSRs, over-completeness is a choice rather than a necessary working condition. The procedures run as
%usual if the dictionary is set to be under-complete.

\subsection{Spectral Clustering Methods}\label{sec:related work:Scut}
Spectral clustering usually consists of two steps: first some eigenvectors are solved, then some post-processing
techniques are employed to recover the discrete indicator vectors from the eigenvectors, i.e., finishing clustering.
Scut follows this manner. Besides the most simple post-processing technique, K-means, as classical spectral clusterings
used \cite{von2007tutorial}, there are some other efforts have been made.

\cite{zelnik2004self} is the closet to Scut. It tried to find a rotation matrix through which rows of the eigenvectors
(column-wise) best align with the canonical coordinate system. Then as Scut, non-maximum suppression was applied to
finish clustering. However, the rotation matrix is solved by gradient descent under some objective, whose computational
cost is high when the number of clusters becomes somewhat large. \cite{yu2003multiclass} directly found a set of
discrete indicator vectors and a rotation matrix such that the discrete indicator vectors approximate the
row-normalized eigenvectors after the rotation. Rather than finding a rotation matrix, \cite{zass2005unifying} found a
set of soft indicator vectors by nonnegative factorization of a transformed similarity matrix, then non-maximum
suppression was applied.

The post-processing step is important for spectral clustering.
Although many recent variants of spectral clustering were proposed,
producing various relaxed indicator matrices, e.g.,
\cite{wang2009clustering, nie2011spectral, yang2012clustering},
they still rely on the above techniques, especially K-means, to
finish the clustering. Little progress was observed in dealing with
the post-processing.

\section{Reviews: PCA, K-means, Rcut, and LE}\label{sec:review}

The section will introduce the classical formulations of PCA, K-means, Rcut, and LE, as well as the most appropriate forms for the unification: the dictionary representation forms for PCA and K-means, the trace optimization forms for Rcut and LE. Along with them, important properties and interpretations are elaborated.

\subsection{Principal Component Analysis (PCA)}\label{sec:pca}
Given mean-removed data matrix $A\in \mathbb{R}^{p\times n}$
($A\mb{1}=\mb{0}$), PCA \cite{jolliffe2002principal} approximates
the data by representing them in another basis, called loadings
$Q=[Q_1,\dots,Q_r]\in \mathbb{R}^{p\times r}$ ($r\ll p$):
$A_i\approx QY_i$, $\forall i$, where $Y_i\in \mathbb{R}^r$ are the
low-dimensional codes, called principal components (PCs). The
loadings are computed sequentially by maximizing the scaled
covariance matrix $C=AA^T$:
\begin{equation}
\max_{Q_i}\;Q_i^TCQ_i,\;\st\;\|Q_i\|_2=1,\;Q_i^TQ_j=0,\;j=1,\dots,i-1.
\end{equation}
Let $A=U\Sigma V$ be the compact SVD of $A$,\footnote{By default we
use the compact form of SVD in this paper, and assume a descending
order of the singular values.} where $\Sigma\in\mathbb{R}^{m\times
m}$, and $m$ is the rank of $A$. Then $C=U\Sigma^2U^T$, and the
solution is $Q^*=U_{1:r}$, $Y^*=Q^{*T}A=\Sigma_{1:r}V_{1:r}$, where
$\Sigma_{1:r}$ is a diagonal matrix containing the first $r$
singular values.

PCA can also be derived from the well-known Eckart-Young theorem
\cite{eckart1936approximation}, where it takes a dictionary
representation form:

\begin{theorem}\label{theo:Eckart-Young}
(\textbf{Eckart-Young Theorem})
%\footnote{It holds for general matrix not necessarily mean-removed.}
Let $A=U\Sigma V$ be the compact SVD of $A$. A particular solution
of the rank $r$ ($r\leq \rank(A)$) approximation of $A$
\begin{equation}\label{equ:Eckart-Young}
\min_{D,X} \|A-DX\|^2_F,\;\st\,XX^T=I,
\end{equation}
where $D\in\mathbb{R}^{p\times r}$ and $X\in\mathbb{R}^{r\times n}$,
is provided by
\begin{equation}\label{equ:solution of Eckart-Young}
X^*=V_{1:r},\;D^*=AX^{*T}=U_{1:r}\Sigma_{1:r}.
\end{equation}
\end{theorem}
In fact, for any rotation matrix $R\in\mathbb{R}^{r\times r}$,
$R^TR=I$,  $RX^*$ and $D^*R^T$ also constitutes a solution.
%It is
%easy to see that if we impose normality on $D$ rather than on $X$,
%we get $D^*=U_{1:r}$, i.e., loadings of PCA, and
%$X^*=\Sigma_{1:r}V_{1:r}$, i.e., PCs of PCA.
%Henceforth, we view (\ref{equ:Eckart-Young}) as the main formulation
%of PCA.
$V_{1:r}$ are the normalized PCs, where the weights $\Sigma_{1:r}$
have been transferred to the loadings, and $U_{1:r}\Sigma_{1:r}$ becomes
the scaled loadings. If normality were imposed on $D$ rather than $X$, PCA could be exactly recovered.

In this paper, we take (\ref{equ:Eckart-Young}) as the formulation
of PCA. It has the following interpretations: 1) each sample
is approximated by a dictionary representation, e.g., $A_i\approx
DX_i$, namely a linear combination of the atoms (columns) of
dictionary $D$ with the codes $X_i$. 2) The dictionary in turn comes from
the linear combinations of the samples with transpose of the codes, $D^*=AX^{*T}$. 3) Eliminating
the dictionary, we have $A_i\approx A(X^{*T}X_i^*)$, i.e., $A_i$ can
be approximated by a linear combination of the whole data set with
weight vector $X^{*T}X_i^*$. The Gram matrix of codes, $X^{*T}X^*$,
encodes a linear relationship of the data within rank-$r$ limitation.

\subsection{K-means}\label{sec:k-means}
Given data matrix $A\in \mathbb{R}^{p\times n}$,
%\footnote{K-means is translation invariant, we may assume $A$ is mean-removed.}
K-means \cite{macqueen1967some, ding2004k} aims to partition the
data into $K$ clusters via minimizing the within-cluster variance:
\begin{equation}
\min_{C_1,\dots,C_K,D} \sum_{k=1}^{K}\sum_{A_i\in
C_k}\|A_i-D_k\|_2^2,
\end{equation}
where $C_k$ is the $k$th cluster, $D=[D_1,\dots,D_K]\in
\mathbb{R}^{p\times K}$ are the $K$ cluster centers.
%The global optimal solution is hard to obtain.
After initializing $D$, K-means finds a locally optimal solution via
alternating two steps: given $D$, assign each sample to the nearest
cluster center; given the current partition, update $D_k$ by the
mean of its members.

Using indicator matrix, the above objective can be written
in a dictionary representation form \cite{ding2004k}:
\begin{equation}\label{equ:k-means A-DF}
\begin{aligned}
\min_{D,F}\,&\sum_{i=1}^n\|A_i-DF_i\|_2^2=\|A-DF\|^2_F,\st\,F\in
\mathcal{F}.
\end{aligned}
\end{equation}
It can be solved alternately. Given $D$, solving $F_i$ corresponds
to the nearest-center search; given $F$, it becomes a linear
regression problem, and $D^*=AF^{\dag}=AF^T(FF^T)^{-1}$. Since
$FF^T$ is a diagonal matrix of the cluster sizes, the atoms of $D^*$ are
the averages of cluster members. Hence, the solution process is
exactly identical to that of traditional K-means.
%The problem can be solved alternately with respect to $D$ and $F$.
%It can be shown that the solution process turns out to be exactly
%the above K-means algorithm. We would not elaborate.

Substituting $D^*=AF^T(FF^T)^{-1}$ into (\ref{equ:k-means A-DF}), we
obtain $\min_{F}\,\|A-AF^T(FF^T)^{-1}F\|^2_F$, $\st\,F\in
\mathcal{F}$. Using the normalized indicator instead, it is
equivalent to
\begin{equation}\label{equ:k-means A-AX^TX}
\begin{aligned}
\min_{X}\,\|A-AX^TX\|^2_F,\,\st\,X\in \mathcal{H},
\end{aligned}
\end{equation}
and
\begin{equation}\label{equ:k-means A-DX}
\begin{aligned}
\min_{D,X}\,\|A-DX\|^2_F,\,\st\,X\in \mathcal{H},
\end{aligned}
\end{equation}
with $D^*=AX^{*T}$.\footnote{For simplicity, we will frequently use
the symbol $D$ to represent dictionaries in this paper. However,
they may not be the same variable when appearing in different
objectives, e.g., (\ref{equ:k-means A-DF}) and (\ref{equ:k-means
A-DX}).}

In this paper, we take (\ref{equ:k-means A-DX}) as the formulation
of K-means. It has the following interpretations: 1) Each sample is
allowed to be represented by only one atom, $A_i\approx DX_i$,
$\|X_i\|_0=1$. The representation error is thus large. 2) Each atom
is a weighted average of the cluster members, $D^*_k=AX_k^{*T}$. But note that it is not a proper cluster center now, since $X_k$ does not sum to 1. 3)
Eliminating the dictionary, we get $A_i\approx A(X^{*T}X_i^*)$. The weight vector $X^{*T}X_i^*$ takes the form of, e.g.,
$[\frac{1}{n_k},\dots,\frac{1}{n_k},0,\dots,0]^T$, and they always
sum to 1, $\mb{1}^T(X^{*T}X_i^*)=1$. It implies that $A_i$ is
approximated by the mean of the cluster members, i.e., cluster
center. The Gram matrix of the codes $X^{*T}X^*$ reflects the
cluster structure of the data.

\subsection{Ratio Cut (Rcut)}\label{sec:ratio cut}

Given an undirected graph of $n$ vertices (data points), with the adjacency matrix defined to be a similarity matrix $W\in \mathbb{R}^{n\times n}$, measuring the pairwise similarities between data points, $W_{ij}=W_{ji}\geq 0$, Rcut \cite{chan1994spectral, von2007tutorial} seeks a partition via minimizing the one-versus-rest weights:
\begin{equation}
\min_{C_1,\dots,C_K} \sum_{k=1}^K \frac{g(C_k,\bar{C}_k)}{n_k},
\end{equation}
where $C_k$ is the $k$th cluster, and $\bar{C}_k$ is the complement
of it, $g(C_k,\bar{C}_k)=\sum_{i\in C_k} \sum_{j\in \bar{C}_k}
W_{ij}$, $n_k$ is the size of $C_k$.

The objective can be expressed more explicitly by the indicator
notation together with the graph Laplacian matrix
\cite{von2007tutorial}:
%\begin{equation}
\begin{align}
&\min_{C_1,\dots,C_K} \sum_{k=1}^K
\frac{g(C_k,\bar{C}_k)}{n_k}\label{equ:ratio cut 1}\\
=&\min_{F} \sum_{k=1}^K \frac{F_kLF_k^T}{F_kF_k^T},\,\st\,
F\in \mathcal{F}\label{equ:ratio cut 2}\\
=&\min_{X} \sum_{k=1}^K
X_kLX_k^T,\,\st\,X_k=(F_kF_k^T)^{-1/2}F_k\\
=&\min_{X}\;\tr\{XLX^T\},\,\st\,X\in \mathcal{H}.\label{equ:ratio
cut}
\end{align}
%\end{equation}
$L$ is the Laplacian matrix defined as $L=S-W$, where $S$ is a
diagonal degree matrix with the $i$th diagonal element being the sum
of weights on the $i$th row, $s_i=\sum_{j=1}^n W_{ij}$. In the above
objectives, the equivalence of (\ref{equ:ratio cut 1}) and
(\ref{equ:ratio cut 2}) as well as the solution of (\ref{equ:ratio
cut}) depend on some properties of the Laplacian matrix
\cite{mohar1991laplacian, mohar1997some} and the ideal graph
condition. These properties and the condition are fundamental to the
unification and SSR. We now introduce.

Firstly, there is an elementary property of Laplacian matrix:
\begin{equation}\label{equ:hLh}
\forall f\in \mathbb{R}^n,\;f^TLf=\sum_{i<j}(f_i-f_j)^2W_{ij} \geq
0.
\end{equation}
Based on this elementary property, the following properties can be
derived:
\begin{enumerate}
\item{} $L$ is positive semi-definite.

\item{} When $f$ is an indicator vector for $C_k$,
$f^TLf=g(C_k,\bar{C}_k)$.

\item{} $\text{null}(L)\neq \{0\}$. Vector $\mb{1}\in \mathbb{R}^n$ is always
an eigenvector of eigenvalue 0.
%Besides, $\mb{1}^T\in
%\text{span}(F)$, $\forall F\in\mathcal{F}$; $\mb{1}^T\in
%\text{span}(X)$ and $X^TX\mb{1}=\mb{1}$, $\forall X\in\mathcal{H}$.

\item{} The multiplicity of eigenvalue 0 equals the number of connected
components of the graph, and the $K$ indicator vectors of the
partition span the eigenspace of eigenvalue 0.
\end{enumerate}

The equivalence of (\ref{equ:ratio cut 1}) and (\ref{equ:ratio cut
2}) is due to property 2. Property 4 plays a key role in the
unification, and it is closely related to the ideal graph condition.
The condition was informally called ``ideal case'' in
the literature \cite{von2007tutorial,
ng2002spectral}. We define it precisely:

\begin{definition}\label{def:ideal}
\tb{(Ideal graph condition)} Targeting for $K$ clusters, if there
are exactly $K$ connected components in the graph, then the graph
(or similarity matrix) is called ideal (with respect to $K$
clusters).
\end{definition}

The condition implies the between-cluster weights are all zero: $W_{ij}=0$, if the $i$th and $j$th points are of different clusters. If members in the same cluster are arranged consecutively,
the similarity matrix would consist of $K$ diagonal blocks, and
there are no sub diagonal blocks in each block.

However, the condition is often met \emph{nearly} in
practice: some nonzero weights exist between clusters. If arranged orderly, the similarity
matrix would consist of $K$ noisy diagonal blocks. Furthermore,
according to matrix perturbation theory \cite{ding2001spectral,
von2007tutorial}: 1) the $K$ smallest eigenvalues are close to 0,
with the smallest one still being 0; 2) the eigenspace is spanned by
the $K$ \emph{noisy} indicator vectors, which do not deviate much
from those in the ideal case; 3) the $K$ eigenvectors, noted by $V\in \mathbb{R}^{K\times n}$, are a
rotation of the $K$ noisy indicator vectors $H$, i.e., $V=RH$, where
$R$ is a unitary matrix.

Based on the above properties, we now introduce the solution scheme
of Rcut. Objective (\ref{equ:ratio cut}) is hard to solve due to the
discrete nature of $X$. For this reason, it is relaxed to
\begin{equation}\label{equ:ratio cut relax}
\min_{X}\;\tr\{XLX^T\},\,\st\,XX^T=I,
\end{equation}
where the discrete constraint is ignored. By the Ky Fan theorem
\cite{fan1961generalization}, the solution set consists of all the
rotations of the $K$ smallest eigenvectors of $L$. Traditionally,
the clustering is finished by applying K-means on the columns of
these eigenvectors. The underlying rationale, which inspires SSR and
Scut, is explained below.

When the graph is ideal, the underlying indicator matrix $H^*$ would
be the smallest eigenvectors, thus it constitutes a solution. Note
that the $i$th column of $H^*$, $H^*_i\in\mathbb{R}^{K}$, is the
indicator vector of the $i$th data point: there is only one nonzero entry,
and the index of it signals the cluster membership. All members of a
cluster share the same pattern, and the indicator vectors of two
data points belonging to different clusters are orthogonal.
%But, any rotation of $H^*$ spans the same null space and
%constitutes an optimal solution too.
However, in numerical computation, the computed eigenvectors may be
a rotated version, $V=RH^*$. The indicator structure disappears.
Nevertheless, members of a cluster still share the same column
pattern, and orthogonality is preserved. K-means can be applied on
these columns to finish clustering. When the graph is nearly ideal,
the eigenvectors are rotated noisy indicators. In this case, to
finish clustering, some post-processing technique, e.g., K-means, is
more necessary.

%\begin{definition}\label{def:near ideal}
%\tb{(Near ideal graph condition)} If the data are nearly separable
%according to the graph of similarity matrix, and the underlying
%components are equivalent to the presumed clusters of data, the
%graph (or similarity matrix) is near ideal. In this case, the
%between-cluster weights are small or only a few of them are
%nonzero.\footnote{If the points in the same cluster are arranged
%consecutively, the similarity matrix is a noisy block-diagonal
%matrix.}
%\end{definition}

\subsection{Laplacian Eigenmap (LE)}\label{sec:laplacian eigenmap}

LE \cite{belkin2003laplacian} is a dimensionality reduction algorithm that works with the same
similarity matrix $W$ as Rcut. It finds low-dimensional codes $X_i\in\mathbb{R}^{r}$ for the $i$th data:
\begin{equation}\label{equ:laplacian eigenmap}
\min_{X}\;\frac{1}{2}\sum_{i,j}\,\|X_i-X_j\|_2^2W_{ij}=\tr\{XLX^T\},\,\st\,XX^T=I,
\end{equation}
where $X_i$ is the $i$th column of $X$, $L=S-W$ is the Laplacian
matrix of $W$. The constraint $XX^T=I$ is to avoid trivial
solution.\footnote{The original LE imposes $XSX^T=I$, whose
counterpart in spectral clustering is Ncut
\cite{shi2000normalized,belkin2003laplacian}. However, the
simplified version here, whose counterpart in spectral clustering is
Rcut, is more elementary, and it makes the connections to PCA and K-means feasible.} The left-hand side of (\ref{equ:laplacian
eigenmap}) implies that if a data pair is close ($W_{ij}$ being
large), their codes should be close too ($\|X_i-X_j\|_2^2$ being
small). In this way, the data relationship contained by $W$ is
approximately preserved by the codes. The equivalence of the
left-hand side and the right-hand side is due to property
(\ref{equ:hLh}) of the Laplacian matrix.\footnote{Note that the
right-hand side takes exactly the form of the relaxed Rcut
(\ref{equ:ratio cut relax}). However, the task here focuses on
representation rather than clustering. It does not care whether the
data has cluster structure or not.}

According to the right hand side of (\ref{equ:laplacian eigenmap}),
any rotation of the $r$ smallest eigenvectors of $L$ is a qualified
solution. Let $L=V^T\Lambda V$ be the full spectral decomposition of
$L$,\footnote{By default we use the full form of spectral
decomposition for the Laplacian matrix in this paper, which includes
all zero eigenvalues, and we assume an ascending order of the
eigenvalues.} where $V_1=\frac{1}{\sqrt{n}}\mb{1}^T$ is the smallest
eigenvector (see property 1 and 3 of the Laplacian matrix). For the
purpose of dimensionality reduction, the constant component $V_1$ is
often omitted. We can either restrict $X\mb{1}=\mb{0}$ to obtain $r$ dimensional codes $X^*=V_{2:r+1}$, or restrict $X=\begin{bmatrix}
\frac{1}{\sqrt{n}}\mb{1}^T\\
\tilde{X}\end{bmatrix}$ to obtain $r-1$ dimensional codes
$\tilde{X}^*=V_{2:r}$. In this paper, we adopt the latter scheme,
and sometimes also use (\ref{equ:laplacian eigenmap}).

%-------------------------------------------------------------------------
\section{Bilateral Conversions: PCA $\leftrightarrow$ LE, K-means $\leftrightarrow$ Rcut}\label{sec:bilateral conversions}
We begin the unification by establishing the bilateral conversions
between PCA and LE, K-means and Rcut. We will convert the objective
of one method into the form of the other. It will turn out that PCA
and K-means are equivalently working on a similarity matrix defined
by the Gram matrix of data. Consequently, they can be
converted to the forms of kernel methods. Briefly, PCA and K-means are linear cases of LE and Rcut respectively. Conversely, the similarity matrix used by LE and Rcut can be
converted to a Gram matrix of some virtual data, enabling the objectives
of them to be written into the forms of PCA and K-means respectively.

\subsection{PCA $\rightarrow$ LE}\label{sec:pca to laplacian
eigenmap}

For consistency, we assume PCA is to find $r-1$ PCs hereafter.
\begin{proposition}\label{theo:pca to laplacian}
(\tb{PCA $\rightarrow$ LE}) PCA is a special LE that uses linear
kernel. The objective of PCA
\begin{equation}\label{equ:pca A-DX}
\min_{D,X} \|A-DX\|^2_F,\;\st\,XX^T=I,
\end{equation}
where $A\mb{1}=\mb{0}$, $X\in\mathbb{R}^{(r-1)\times n}$, $r-1\leq
\rank(A)$, can be converted to the form of LE
\begin{equation}\label{equ:pca-lap}
\min_{X}\,\tr\{\tilde{X}L\tilde{X}^T\},\,\st\,\tilde{X}=\begin{bmatrix}
\frac{1}{\sqrt{n}}\mb{1}^T\\
X\end{bmatrix},\,\tilde{X}\tilde{X}^T=I.
\end{equation}
$L$ is the Laplacian matrix of $W$, and $W=\tilde{A}^T\tilde{A}$,
where $\tilde{A}=
\begin{bmatrix}
\sqrt{\beta}\mb{1}^T\\
A
\end{bmatrix}$ is called augmented data, and
$\beta=-\min_{ij}\,(A^TA)_{ij}$ so that $W_{ij}\geq 0$, $\forall
i,j$. $W$ is a kernel matrix built by the linear kernel function
$\phi(A_i,A_j)=A_i^TA_j+\beta$.
\end{proposition}

\begin{proof}
In the following, we will eliminate the dictionary of PCA and
convert PCA to the trace form, after that we define a similarity
matrix to be the Gram matrix of data, and then apply a padding trick
to make it nonnegative, finally we convert it to the Laplacian form.

We now go into details. Substituting $D=AX^{T}$ into (\ref{equ:pca A-DX}), we
obtain $\min_{X}$ $\|A-AX^TX\|^2_F$, $\st$ $XX^T=I$. Since $XX^T=I$, $X^TX$ is a projection matrix,
%it means to find a $r$
%dimensional subspace spanned by the rows of $X$ so that after
%projecting rows of $A$ onto it, the reconstructions $AX^TX$ are
%closest to their original ones, i.e., the lengths are best
%preserved.
then by the Pythagorean theorem,
\begin{equation}\label{equ:Pythagorean}
\|A-AX^TX\|^2_F=\|A\|^2_F-\|AX^TX\|^2_F=\tr\{A^TA\}-\tr\{XA^TAX^T\}.
\end{equation}
Hence the objective is equivalent to
%\footnote{Applying kernel trick on the Gram matrix $A^TA$ we get kernel PCA \cite{scholkopf1998nonlinear}.}
$\min_{X}$ $\tr\{X(-A^TA)X^T\}$, $\st\,XX^T=I$. In this form, PCA is close to
LE. Define the similarity matrix to be $W=A^TA$, which is in fact
built by linear kernel. By $A\mb{1}=\mb{0}$,
$S=\diag(W\mb{1})=\mb{0}$, so $(-A^TA)$ is a Laplacian matrix,
%\footnote{Applying kernel trick on $L=\diag(A^TA\mb{1})-A^TA$, we get traditional LE (\ref{equ:laplacian eigenmap}).}
except that $W$ may not be nonnegative. Again, by
$A\mb{1}=\mb{0}$, $\mb{1}$ is a singular vector of $A$ with singular value 0. Then according to the Eckart-Young theorem, $X^*$ must be orthogonal to $\mb{1}$, so we can restrict
$X\mb{1}=\mb{0}$, and the objective is equivalent to
\begin{equation}\label{equ:pca to laplacian}
\begin{aligned}
&\min_X\,-\tr\{X(\beta
\mb{1}\mb{1}^T+A^TA)X^T\}-(\frac{1}{\sqrt{n}}\mb{1})^T(\beta
\mb{1}\mb{1}^T+A^TA)(\frac{1}{\sqrt{n}}\mb{1}),\,\st\,\begin{split}&X^TX=I\\&X\mb{1}=\mb{0}\end{split}\\
%=&\min_{X}\,-\tr\{\tilde{X}(\beta
%\mb{1}\mb{1}^T+A^TA)\tilde{X}^T\},\,\st\,\tilde{X}=\begin{bmatrix}
%\frac{1}{\sqrt{n}}\mb{1}^T\\
%X\end{bmatrix},\,\tilde{X}\tilde{X}^T=I\\
=&\min_{X}\,\tr\{\tilde{X}(-\tilde{A}^T\tilde{A})\tilde{X}^T\},\,\st\,\tilde{X}=\begin{bmatrix}
\frac{1}{\sqrt{n}}\mb{1}^T\\
X\end{bmatrix},\,\tilde{X}\tilde{X}^T=I.
\end{aligned}
\end{equation}
In above, firstly we padded $A^TA$ uniformly with
$\beta\geq 0$ to make $A^TA$ nonnegative, e.g., choosing
$\beta=-\min_{ij}\,(A^TA)_{ij}$ (Since $A\mb{1}=\mb{0}$,
surely $\min_{ij}\,(A^TA)_{ij}\leq 0$); secondly we augmented $X$ to
be $\tilde{X}$. The padding trick can be
interpreted as translating $A$ along a new dimension so that the
inner products of all data-pairs become nonnegative.

Finally, since $S=\diag(\tilde{A}^T\tilde{A}\mb{1})=\beta n I$ and
$\tilde{X}\tilde{X}^T=I$, we can turn (\ref{equ:pca to laplacian})
into a problem with respect to the Laplacian matrix $L=\beta n
I-\tilde{A}^T\tilde{A}$, which is (\ref{equ:pca-lap}), a LE that uses linear kernel.
\end{proof}

By the equivalence of the objectives, we can obviously obtain the
equivalence of the solutions.
\begin{corollary}\label{theo:pca to laplacian solution}
The solutions of $X$ in (\ref{equ:pca A-DX}) and (\ref{equ:pca-lap})
are equivalent. Let the SVD of $A$ be $A=U\Sigma V$, one common
solution is $X^*=V_{1:r-1}$.
\end{corollary}

During the conversions, the right singular vectors of $A$ are preserved. It can be shown that
\begin{equation}\label{equ:svd tildeA}
\tilde{A}=\tilde{U}\tilde{\Sigma}\tilde{V}=\begin{bmatrix}1&\\&U\end{bmatrix}\begin{bmatrix}\sqrt{\beta
n}&\\
&\Sigma\end{bmatrix}\begin{bmatrix}\frac{1}{\sqrt{n}}\mb{1}^T\\V\end{bmatrix},
\end{equation}
i.e., the SVD of $\tilde{A}$ is also an augmentation of that of $A$.
Further, we obtain the spectral decomposition of $L$:
\begin{equation}\label{equ:sd L}
\begin{split}
L=\beta n I-\tilde{A}^T\tilde{A}&=
\begin{bmatrix}\tilde{V}^T&\hat{V}^T\end{bmatrix}\begin{bmatrix}\beta n
I-\tilde{\Sigma}^2&\\
&\beta
nI\end{bmatrix}\begin{bmatrix}\tilde{V}\\\hat{V}\end{bmatrix}\\& =
\begin{bmatrix}\frac{1}{\sqrt{n}}\mb{1}&V^T&\hat{V}^T\end{bmatrix}\begin{bmatrix}0&&\\
&\beta nI-\Sigma^2&\\&&\beta
nI\end{bmatrix}\begin{bmatrix}\frac{1}{\sqrt{n}}\mb{1}^T\\V\\\hat{V}\end{bmatrix},
\end{split}
\end{equation}
%since $L$ is positive semi-definite we know that $\beta n-\sigma_1^2\geq0$.
where $\hat{V}$ is the complement of $\tilde{V}$, i.e., $[\tilde{V}^T, \hat{V}^T]$ forms a unitary matrix. By
(\ref{equ:sd L}), one solution of LE is $X^*=V_{1:r-1}$, equivalent to that of PCA.

\subsection{K-means $\rightarrow$ Rcut}\label{sec:k-means to ratio cut}

\begin{proposition}\label{theo:kmeans to ratio cut}
(\tb{K-means $\rightarrow$ Rcut}) Concerning objectives, K-means is
a special Rcut that uses linear kernel. The objective of K-means
\begin{equation}\label{equ:kmeans A-DX}
\min_{D,X} \|A-DX\|^2_F,\;\st\,X\in \mathcal{H},
\end{equation}
where $X\in\mathbb{R}^{K\times n}$, can be converted to the form of
Rcut
\begin{equation}\label{equ:kmeans-ratio cut}
\min_{X}\,\tr\{XLX^T\},\,\st\,X\in \mathcal{H},
\end{equation}
where $L$ is a Laplacian matrix defined as Proposition~\ref{theo:pca
to laplacian}.
\end{proposition}

\begin{proof}
It follows most of the steps of Proposition~\ref{theo:pca
to laplacian}. Since $XX^T=I$, $X^TX$ is a projection matrix, by
(\ref{equ:Pythagorean}), (\ref{equ:k-means A-AX^TX}) is equivalent
to $\max_{X}\, \tr\{XA^TAX^T\},\;\st\,X\in \mathcal{H}$. Applying the padding trick, and since $\tr\{X\mb{1}\mb{1}^TX^T\}=n$,
it is equivalent to
\begin{equation}\label{equ:k-means to ratio cut}
\begin{aligned}
\min_X\,\tr\{X(\beta n I-(\beta
\mb{1}\mb{1}^T+A^TA))X^T\},\;\st\,X\in \mathcal{H}
%=&\min_X\,\tr\{X(\beta n I-\tilde{A}^T\tilde{A})X^T\},\;\st\,X\in
%\mathcal{H},\,where\,\tilde{A}=
%\begin{bmatrix}
%\sqrt{\beta}\mb{1}^T\\
%A
%\end{bmatrix}\\
=\min_X\,\tr\{XLX^T\},\;\st\,X\in \mathcal{H},
\end{aligned}
\end{equation}
where $\beta=-\min_{ij}\,(A^TA)_{ij}$. This is exactly the form of
Rcut that uses linear kernel.
\end{proof}

\begin{corollary}\label{theo:kmeans to ratio cut solution}
The solutions of $X$ in (\ref{equ:kmeans A-DX}) and
(\ref{equ:kmeans-ratio cut}) are equivalent, whereas the algorithmic
results of applying K-means to (\ref{equ:kmeans A-DX}) and applying
Rcut to (\ref{equ:kmeans-ratio cut}) may be different. Let
$A=U\Sigma V$ be the SVD of $A$, where $\Sigma\in\mathbb{R}^{m\times
m}$. Assume $K\leq m+1$. In (\ref{equ:kmeans A-DX}), K-means
algorithm equivalently works with the full PCs $\Sigma V$. In
(\ref{equ:kmeans-ratio cut}), Rcut algorithm applies K-means
algorithm on the leading $K-1$ normalized PCs $V_{1:K-1}$.
\end{corollary}

\begin{proof}
Since K-means is rotation-invariant, working with $U\Sigma
V$ is equivalent to working with $\Sigma V$ for K-means. On the other hand, Rcut
algorithm applies K-means on the smallest eigenvectors of $L$, which
is $[\frac{1}{\sqrt{n}}\mb{1},V_{1:K-1}^T]^T$ by (\ref{equ:sd L}),
and since K-means is translation-invariant too, it is equivalent to
applying K-means on $V_{1:K-1}$, i.e., the first $K-1$ PCs with
singular values ignored.
\end{proof}

When the graph, embodied by $W=\tilde{A}^T\tilde{A}$, is ideal (Definition~\ref{def:ideal}), ignoring or keeping the
singular values makes no difference, since all of them are equal (by property 4 of Laplacian matrix and (\ref{equ:sd
L})). In this case, Rcut algorithm is equivalent to applying PCA to reduce the dimensionality of the data first, and
then applying K-means to finish clustering. However, in practice, this condition can hardly be met even nearly, which
essentially requires that after translating along a new dimension different clusters become near orthogonal. As a
consequence, the indicator structure underlying the PCs may deviate much from the ideal one, and the singular values
diverge. In this case, following exact PCA and applying K-means on the PCs, $\Sigma_{1:K-1}V_{1:K-1}$, may be
preferable.
%Note that, by the inherent relation, the reduced dimensionality $K-1$ is deterministic. It is interesting
%to find a counterpart in the supervised learning domain: LDA for classification \cite{duda2012pattern}, where the
%target dimension is also $K-1$. In contrast, in previous work that use dimensionality reduction as preprocessing, the
%target dimension is computed according to some considerations, e.g., fidelity \cite{Cohen2014Dimensionality}.

%There are two justifications. Firstly the
%singular values weight their corresponding singular vectors,
%%\footnote{The weighting is impossible in the original Rcut formulation, since the ``weights" include zero.}
%secondly the PCs are the best rank-$K$ approximation to the data by
%the Eckart-Young theorem.

\subsection{LE $\rightarrow$ PCA}

\begin{proposition}\label{theo:laplacian to pca}
(\tb{LE $\rightarrow$ PCA}) The objective of LE
\begin{equation}\label{equ:laplacian}
\min_{\tilde{X}}\,\tr\{XLX^T\},\,\st\,X=\begin{bmatrix}
\frac{1}{\sqrt{n}}\mb{1}^T\\
\tilde{X}\end{bmatrix},\,XX^T=I,
\end{equation}
where $X\in\mathbb{R}^{r\times n}$, can be converted to the form of
dictionary representation
\begin{equation}\label{equ:laplacian-dictionary}
\min_{D,\tilde{X}} \|\tilde{A}-DX\|^2_F,\;\st\,X=\begin{bmatrix}
\frac{1}{\sqrt{n}}\mb{1}^T\\
\tilde{X}\end{bmatrix},\,XX^T=I,
\end{equation}
and further to the form of PCA
\begin{equation}\label{equ:laplacian-pca}
\min_{D,\tilde{X}}
\|\bar{A}-D\tilde{X}\|^2_F,\;\st\,\tilde{X}\tilde{X}^T=I,
\end{equation}
Let $L=V^T\Lambda V$ be the spectral decomposition of $L$, where
$V_1=\frac{1}{\sqrt{n}}\mb{1}^T$ and $\lambda_n$ is the largest
eigenvalue, $\tilde{A}$ is defined as $\tilde{A}=(\lambda_n I-\Lambda)^{\frac{1}{2}}V$,
called virtual data, and $\bar{A}$ is the mean-removed version of $\tilde{A}$.
\end{proposition}

\begin{proof}
Since $\lambda_n I-L=V^T(\lambda_n
I-\Lambda)V=\tilde{A}^T\tilde{A}$, we have
\begin{equation}\label{equ:laplacian eigenmap to pca}
\begin{aligned}
\min_{X}\;\tr\{XLX^T\} \Leftrightarrow \max_{X}\;\tr\{X(\lambda_n
I-L)X^T\} =\max_{X}\;\tr\{X\tilde{A}^T\tilde{A}X^T\},\,\st\,XX^T=I.
\end{aligned}
\end{equation}
Further, by (\ref{equ:Pythagorean}), (\ref{equ:laplacian eigenmap to
pca}) is equivalent to $\min_{X}$ $\|\tilde{A}-\tilde{A}X^TX\|^2_F$, $\st$ $XX^T=I$,
and also $\min_{D,X}$ $\|\tilde{A}-DX\|^2_F$, $\st$ $XX^T=I$, with $D^*=\tilde{A}X^{*T}$. There remains minor difference to PCA: $\tilde{A}$ is not mean-removed. We now tackle this problem. Since $X=[
\frac{1}{\sqrt{n}}\mb{1}, \tilde{X}^T]^T$, the objective is equivalent to $\min_{\tilde{X}}$ $
\|\tilde{A}(I-\frac{1}{\sqrt{n}}\mb{1}\frac{1}{\sqrt{n}}\mb{1}^T)(I-\tilde{X}^T\tilde{X})\|^2_F
$ $=\min_{\tilde{X}}$ $\|\bar{A}-\bar{A}\tilde{X}^T\tilde{X}\|^2_F,\;\st\,\tilde{X}\tilde{X}^T=I$,\footnote{In precise, it holds when
$r\geq$ the null-space dimension of $L$ so that $\mb{1}$ must be a
component of the solution.}
where
$\bar{A}=\tilde{A}(I-\mb{1}\mb{1}^T/n)$
removes the mean of $\tilde{A}$. The objective finally
is equivalent to (\ref{equ:laplacian-pca}), for
$D^*=\bar{A}\tilde{X}^{*T}$.
\end{proof}

Keeping track of the eigenvectors of $L$ during the conversions, and by the Eckart-Young theorem, we have
\begin{corollary}\label{theo:laplacian to pca solution}
The solutions of $\tilde{X}$ in (\ref{equ:laplacian}),
(\ref{equ:laplacian-dictionary}), and (\ref{equ:laplacian-pca}) are
equivalent, one of them is $V_{2:r}$.
\end{corollary}
%Since $\lambda_n I-L=V^T(\lambda_n I-\Lambda)V$, the eigenvalues of
%$\lambda_n I-L$ reverse those of $L$. Especially, the eigenvalue 0
%in $L$ becomes the largest $\lambda_n$ in $\lambda_n I-L$.
%Nevertheless, $\lambda_n I-L$ remains positive semi-definite,
%%\footnote{Though any value larger than $\lambda_n$ meets our purpose, we have chosen the minimal one.}
%and the eigenvectors remain the same. Note that the SVD of
%$\tilde{A}$ is trivial: $\tilde{A}=I(\lambda_n
%I-\Lambda)^{\frac{1}{2}}V$. By the Eckart-Young theorem, one
%solution of $X$ in (\ref{equ:eigenmap A-DX}) is $V_{1:r}$. $\bar{A}$
%removes the largest component of $\tilde{A}$, its SVD is
%$\bar{A}=I(\lambda_n I-\Lambda_{2:n})^{\frac{1}{2}}V_{2:n}$. Again
%by the Eckart-Young theorem, one solution of PCA
%(\ref{equ:laplacian-pca}) is $V_{2:r}$, identical to that of LE.

\subsection{Rcut $\rightarrow$ K-means}\label{sec:ratio cut to k-means}

\begin{proposition}\label{theo:ratio to kmeans}
(\tb{Rcut $\rightarrow$ K-means}) The objective of Rcut
\begin{equation}\label{equ:ratio}
\min_{X}\,\tr\{XLX^T\},\,\st\,X\in \mathcal{H},
\end{equation}
can be converted to the form of K-means
\begin{equation}\label{equ:ratio-kmeans}
\min_{D,X} \|\tilde{A}-DX\|^2_F,\;\st\,X\in \mathcal{H},
\end{equation}
where $\tilde{A}$ is defined as Proposition~\ref{theo:laplacian to
pca}.
\end{proposition}

%\begin{proof}
The conversion resembles last section except that an
additional constraint $X\in \mathcal{H}$ is added. The proof is omitted. At last, we also have
%We briefly
%present them. (\ref{equ:ratio}) is equivalent to
%\begin{equation}\label{equ:ratio A-AX^TX}
%\min_{X} \|\tilde{A}-\tilde{A}X^TX\|^2_F,\;\st\,X\in \mathcal{H},
%\end{equation}
%and (\ref{equ:ratio-kmeans}) with $D^*=\tilde{A}X^{*T}$.
%\end{proof}

\begin{corollary}\label{theo:ratio to kmeans solution}
The solutions of $X$ in (\ref{equ:ratio}) and
(\ref{equ:ratio-kmeans}) are equivalent, whereas the algorithmic
results of applying Rcut to (\ref{equ:ratio}) and applying K-means
to (\ref{equ:ratio-kmeans}) may be different. Rcut algorithm relaxes
(\ref{equ:ratio}) to (\ref{equ:ratio cut relax}), and applies
K-means algorithm on the leading $K$ normalized PCs of $\tilde{A}$,
while K-means algorithm works with the full PCs, faithfully following the objective.
\end{corollary}

%The relation is similar to that in Section~\ref{sec:k-means to ratio
%cut}. K-means algorithm works with $\tilde{A}$, it exactly follows
%the original objective (\ref{equ:ratio}), since
%(\ref{equ:ratio-kmeans}) is equivalent to (\ref{equ:ratio}). In
%contrast, Rcut algorithm relaxes (\ref{equ:ratio}) to
%(\ref{equ:ratio cut relax}) and applies K-means algorithm on the
%smallest eigenvectors of $L$, or the leading normalized PCs of
%$\tilde{A}$. But unlike Section~\ref{sec:k-means to ratio cut}, the
%ideal graph condition can be met nearly or even perfectly here. When
%the data set is moderately large, the smallest eigenvalues of $L$
%are about the same (zero) compared with the largest eigenvalue. Thus
%the largest singular values of $\tilde{A}$ do not differ much, and
%applying K-means on the exact PCs dose not make much sense.

\section{Unification of PCA, K-means, LE, and Rcut}\label{sec:unification}

In Section~\ref{sec:bilateral conversions}, the bilateral relations between PCA and
LE, K-means and Rcut have been established. Now, under the ideal
graph condition (Definition~\ref{def:ideal}), we will establish the
equivalence of LE and Rcut, and then the equivalence of PCA and K-means
is automatically established. Therefore, the four methods are
unified together. Roughly speaking, the fundamental principle of the unification lies in that, under the condition,
the indicator vectors become the leading singular vectors, which means the target of the cluster analysis methods,
K-means and Rcut, coincides with that of the dimensionality reduction methods, PCA and LE.

Given a data set, we have two choices: one is to work with the
original data, the other is to work with the similarity matrix
constructed from the data. Therefore the theory has two versions: a
linear version and a kernel version. Both of the conditions
of the two versions concern the similarity matrices. The similarity matrix of the
kernel version is usually built by the K-Nearest Neighbor (KNN)
graph or the $\varepsilon$-neighborhood graph with Gaussian kernel
\cite{von2007tutorial}, while that of the linear version is
defined by the Gram matrix of data, in fact built by
linear kernel.\footnote{The two versions can be integrated into
one general kernel version, which works with a similarity matrix
that can be built by linear kernel, nonlinear kernel, or any other
nonnegative and symmetric similarity measure. However, to make things
clear, we distinguish them.}

We now go into details. By spectral graph theory
\cite{chung1997spectral, mohar1991laplacian, mohar1997some}, as
elaborated in Section~\ref{sec:ratio cut}, under the ideal graph
condition, the indicator vectors become the leading eigenvectors of
the Laplacian matrix, so the equivalence of LE and Rcut is manifest.
\footnote{By the same rationale, there is a counterpart in the
literature \cite{belkin2003laplacian}, i.e., normalized LE and
Ncut.}

\begin{theorem}\label{theo:ratio-laplacian}
(\textbf{LE $\Leftrightarrow$ Rcut}) If the ideal graph condition is satisfied and $K=r$, then the solution of LE (\ref{equ:laplacian}), $X^*$, attached with $\frac{1}{\sqrt{n}}\mb{1}$, and the solution of Rcut (\ref{equ:ratio}), $H^*$, are
related by a rotation transform: $\exists R\in\mathbb{R}^{r\times r}$, $R^TR=I$, such that
\begin{equation}\label{equ:X=RH}
X^*=RH^*.
\end{equation}
\end{theorem}

By Corollary~\ref{theo:pca to laplacian solution},
\ref{theo:kmeans to ratio cut solution}, and
Theorem~\ref{theo:ratio-laplacian}, we can obtain the linear
version of the unification.

\begin{theorem}\label{theo:unification original}
(\textbf{Unification: linear version}) Working with the original data, if the ideal graph condition is satisfied and $K=r\leq \rank(A)+1$, then the solution of PCA (\ref{equ:pca A-DX}) and LE (\ref{equ:pca-lap}), e.g., $X^*=V_{1:r-1}$, and the normalized indicator solution of K-means (\ref{equ:kmeans A-DX}) and Rcut (\ref{equ:kmeans-ratio cut}), $H^*$, are related by a rotation transform: $\exists R\in\mathbb{R}^{r\times r}$, $R^TR=I$, such that
\begin{equation}\label{equ:unification original}
\begin{bmatrix}\frac{1}{\sqrt{n}}\mb{1}^T\\V_{1:r-1}\end{bmatrix}=RH^*.
\end{equation}
\end{theorem}
The theorem establishes the exact equivalence of PCA and
K-means, together with the condition when it holds.
%It also suggests
%that, as an alternative to K-means, when $K\leq \rank(A)+1$, we can
%perform clustering by linear Rcut algorithm: applying K-means on the
%normalized PCs $V_{1:K-1}$. However, as argued in
%Section~\ref{sec:k-means to ratio cut}, the ideal condition can
%hardly be met even nearly, thus, except efficiency, it may not
%promise better performance.

Finally, by Corollary~\ref{theo:laplacian to pca solution},
\ref{theo:ratio to kmeans solution}, and
Theorem~\ref{theo:ratio-laplacian}, the kernel version of the
unification is obtained.

\begin{theorem}\label{theo:unification kernel}
(\textbf{Unification: kernel version}) Working with the similarity matrix, if the ideal graph condition is satisfied and $K=r$, then the solution of LE (\ref{equ:laplacian}) and PCA (\ref{equ:laplacian-pca}), e.g., $\tilde{X}^*=V_{2:r}$, and the normalized indicator solution of Rcut (\ref{equ:ratio}) and K-means (\ref{equ:ratio-kmeans}), $H^*$, are related by a rotation transform: $\exists R\in\mathbb{R}^{r\times r}$, $R^TR=I$, such that
\begin{equation}\label{equ:unification kernel}
\begin{bmatrix}\frac{1}{\sqrt{n}}\mb{1}^T\\V_{2:r}\end{bmatrix}=RH^*.
\end{equation}
\end{theorem}

A diagram of the framework is shown in Figure~\ref{fig:framework}.

\section{Spectral Sparse Representation (SSR)}\label{sec:SSR}
Under the ideal graph condition, PCA/LE, K-means/Rcut are unified.
When this condition is met nearly but may not exactly, it leads to
SSR (cf. Figure~\ref{fig:framework}).

We provide a brief overview first. Ignoring minor factors, PCA
and LE can be written in a form that finds a dictionary and codes to
approximately represent the data:
\begin{equation}
\min_{D,X} \|A-DX\|^2_F,\;\st\,XX^T=I,
\end{equation}
where $A$ is either transformed from the original data or factored
from the Laplacian matrix. K-means and Rcut additionally impose the
indicator constraint on $X$. Let $A=U\Sigma V$, one solution of
PCA/LE is $X^*=V_{1:r}$, and any rotation of these eigenvectors
constitutes a solution. When the graph is ideal, some rotation of
the eigenvectors turns into indicator vectors (assume $H^*$),
$V_{1:r}=RH^*$. Thus PCA/LE, K-means/Rcut are unified, and the data
can be represented as $A\approx D^*V_{1:r}=\tilde{D}H^*$, where
$D^*=AV_{1:r}^T$, $\tilde{D}=AH^{*T}$. The first representation form
$D^*V_{1:r}$ is of PCA/LE, while the second $\tilde{D}H^*$ is of
K-means/Rcut. When the graph is nearly ideal, the rotation of
eigenvectors can only lead to noisy indicator vectors (assume $H$),
$V_{1:r}=RH$. When the data is represented by these noisy indicator
vectors, which we call sparse codes, we have $A\approx
D^*V_{1:r}=\hat{D}H$, where $\hat{D}=AH^T$. The representation form
$\hat{D}H$ is the spectral sparse representation.

SSR can be seen as a new representation method. It represents data
in a dimensionality reduced way while achieving the same
representation fidelity as PCA/LE. Meanwhile, the codes are sparse
and approximate to the optimal indicator matrix of K-means/Rcut, so
the underlying cluster structure can be revealed. In contrast to the
hard clustering nature of K-means/Rcut, SSR is soft and descriptive.
It describes the underlying clusters as well as the overlapping
status of them. SSR lies in the intermediate of the two kinds of
methods and combines some of the merits of both sides.

The perturbed indicator matrix is called sparse codes, since each
column of it corresponds to a data point and is usually dominated by
a single entry. A measure of sparsity can be defined as follows. For
a vector $x\in \mathbb{R}^{n}\neq \mb{0}$,
\begin{equation}\label{equ:sparsity}
sparsity(x)=\|x\|_2/\|x\|_1.
\end{equation}
$1/\sqrt{n}\leq
sparsity(x) \leq 1$. The higher the value, the sparser the vector. The minimum is achieved when the magnitudes of
the entries are uniform. The indicator matrix has only one nonzero entry in each column, achieving the maximum sparsity 1.

We provide a qualitative interpretation of the cluster information
revealed by the sparse codes. The values of the code vector suggest
how likely the data point belongs to different clusters. 1) If there
is only one positive entry, then the index of it indicates its
cluster membership. 2) If there are several positive entries, then
the data is an overlapping point of some clusters. 3) If the sizes
of the clusters are similar, the larger the entry is, the closer the
point is to the corresponding cluster. We now analyze. Assume there
are two clusters $C_1$ and $C_2$, and exactly one overlapping point $A_m$. From the view of LE, we are to
minimize $\sum_{i,j}\,\|X_i-X_j\|^2W_{ij}$, $\st$, $XX^T=I$. There
is a sub term $\sum_{i\in \mathcal{N}_m}\|X_m-X_i\|^2W_{mi}$
($\mathcal{N}_m$ is the neighborhood of $A_m$), which requires that
$X_m^*=H_m$, should be close to both $[1/\sqrt{n_1}, 0]^T$
and $[0, 1/\sqrt{n_2}]^T$ (ideal indicator vectors of neighboring
points). Hence there will be a positive value in each row of $H_m$,
with magnitudes less than $1/\sqrt{n_1}$ and $1/\sqrt{n_2}$
respectively. In order to meet the orthogonality between rows of
$H$, small negative values must appear in the other columns of $H$.
%The largest values of each row will be paired with larger negative
%values.
This is our first impression to the sparse codes.

SSR has two versions: the linear version (SSRl) and the
kernel version (SSRk).

\subsection{Kernel Version (SSRk): Similarity Matrix as Input}\label{sec:sparse kernel version}
For convenience, we repeat some formulations. Let
$L=V^T\Lambda V$, where $V_1=\frac{1}{\sqrt{n}}\mb{1}^T$,
$\lambda_1=0$ and $\lambda_n$ is the largest eigenvalue. Define
virtual data $\tilde{A}=(\lambda_n I-\Lambda)^{\frac{1}{2}}V$. LE is
equivalent to
\begin{equation}\label{equ:kernel sparse A-DX}
\min_{D\in\mathbb{R}^{n\times r},X\in\mathbb{R}^{r\times n}}
\|\tilde{A}-DX\|^2_F,\;\st\,XX^T=I.
\end{equation}
The solution set includes $X^*=V_{1:r}$, $D^*=\tilde{A}V_{1:r}^T$,
and any rotation of them.

\begin{theorem}\label{theo:kernel ssr}
(\textbf{SSRk}) When the graph is nearly ideal, there is a sparse
representation
\begin{equation}\label{equ:kernel sparse A=DH=AHH}
\tilde{A}\approx \hat{D}H=\tilde{A}(H^TH).
\end{equation}
The approximation accuracy is optimal in the Frobenius-norm sense.
$H$ is the matrix of sparse codes, i.e., noisy indicator matrix,
which satisfies
\begin{equation}\label{equ:kernel X=RH}
V_{1:r}=RH,
\end{equation}
for some rotation matrix $R$. $\hat{D}$ is a dictionary defined as
\begin{equation}\label{equ:kernel D}
\hat{D}=\tilde{A}H^T=\tilde{A}V_{1:r}^TR=\begin{bmatrix}(\lambda_n
I-\Lambda_{1:r})^{\frac{1}{2}}R\\\mb{0}\end{bmatrix}\approx
\begin{bmatrix}\lambda_n^{\frac{1}{2}}R\\\mb{0}\end{bmatrix},
\end{equation}
which has property $\hat{D}^T\hat{D}\approx \lambda_n I$, i.e., the
atoms of $\hat{D}$ are near-orthogonal and have similar lengths. The
Gram matrix of codes $H^TH\in\mathbb{R}^{n\times n}$ reflects the linear
relationship of data, which has property
\begin{equation}\label{equ:1HTH}
\mb{1}^TH^TH=\mb{1}^TV_{1:r}^TV_{1:r}=\mb{1}^T,
\end{equation}
i.e., the sum of each column is one.
\end{theorem}

\begin{proof}
The key lies in (\ref{equ:kernel X=RH}), the others are easy to
obtain. Assume the graph is ideal, and the underlying normalized
indicator matrix is $H^*$. According to the properties of Laplacian
matrix in Section~\ref{sec:ratio cut}, $H^*$ spans the eigenspace of
eigenvalue 0, and vector $\mb{1}$ always belongs to this space.
Define $\tilde{V}=RH^*$, then $\tilde{V}$ remains to be the ideal
eigenvectors of eigenvalue 0, where $R$ is a rotation matrix so that
$\tilde{V}_1=\frac{1}{\sqrt{n}}\mb{1}^T$. Assume the graph becomes
noisy. According to the matrix perturbation theory, the eigenvalues
become $\Lambda_{1:r}\approx \mb{0}$, and the eigenvectors $V_{1:r}$
becomes a perturbed version of $\tilde{V}$, i.e., $V_{1:r}=RH^*+E$,
where $E$ is some noise. Then we obtain the relationship
(\ref{equ:kernel X=RH}), where $H=H^*+R^TE$ is the noisy indicator
matrix.

If $V_{1:r}$ and $D^*=\tilde{A}V_{1:r}^T$ are a solution of
(\ref{equ:kernel sparse A-DX}), then the rotated version
$H=R^TV_{1:r}$ and $\hat{D}=D^*R=\tilde{A}H^T$ are also a solution.
Thus we obtain SSRk (\ref{equ:kernel sparse A=DH=AHH}). The
approximation accuracy is optimal in the Frobenius-norm sense, as
indicated by the Eckart-Young theorem. Substituting the definition
of $\tilde{A}$ into $\tilde{A}V_{1:r}^TR$, we obtain the third
equality of (\ref{equ:kernel D}). The fourth approximation holds,
due to $\Lambda_{1:r}\approx \mb{0}$. Finally, since $V_{1:r}$ is
orthonormal and $V_1=\frac{1}{\sqrt{n}}\mb{1}^T$, (\ref{equ:1HTH})
is obtained.
\end{proof}

SSRk has the following interpretations (it would be better to
compare with those of PCA and K-means in Section~\ref{sec:pca} and
Section~\ref{sec:k-means} respectively):
\begin{enumerate}
\item The data can be sparsely represented by the dictionary.
$\tilde{A}_i\approx \hat{D}H_i$, i.e., $\tilde{A}_i$
is approximated by a linear combination of a few atoms
of dictionary $\hat{D}$.

\item The dictionary comes from the data clusters.
$\hat{D}_k=\tilde{A}H_k^T$, implying $\hat{D}_k$ mainly comes from the linear combination of samples in cluster $C_k$.
$\hat{D}_k$ can be seen as a quasi cluster center. However, compared with K-means, firstly, the weights are not
uniform. They distribute according to the relevance of the data to the center. Moreover, the weights include small
negative values, which imply least relevance. Secondly, since $\hat{D}^T\hat{D}\approx \lambda_n I$, the dictionary is
incoherent, a desirable property in compressed sensing \cite{bruckstein2009sparse}, and the mutual coherence
$\mu(\hat{D})=\max_{i\neq j}\,|\hat{D}_i^T\hat{D}_j|/(\|\hat{D}_i\|_2\cdot\|\hat{D}_j\|_2)$ is about zero.

\item The data are eventually represented by the data themselves according to relevance.
$\tilde{A}_i\approx \tilde{A}(H^TH_i)$, implying $\tilde{A}_i$ can be
represented by a linear combination of the relevant samples. The sum
of the weights, including negative values, is always 1. This
coincides with K-means.
\end{enumerate}

There is another representation form analogous to the un-normalized indicator representation form of K-means,
which we will call un-normalized SSR:
\begin{equation}\label{equ:unnormalize ssr}
\tilde{A}\approx D_UH_U,
\end{equation}
where $D_U=\hat{D}\,S_H^{-1}$, $H_U=S_HH$, and $S_H=\diag(H\mb{1})$. $H_U$ is the analogy of un-normalized
indicator matrix, for $\mb{1}^TH_U=\mb{1}^T$ as can be verified by (\ref{equ:1HTH}). $D_U$ are the proper cluster
centers, as $D_U=\tilde{A}(H^T\,S_H^{-1})$ and $\mb{1}^TH^T\,S_H^{-1}=\mb{1}^T$, i.e., the weights sum to 1.

\subsection{Linear Version (SSRl): Original Data as Input}\label{sec:sparse original
version}

After turning PCA (\ref{equ:pca A-DX}) to LE (\ref{equ:pca-lap}), we turn it back to the dictionary representation
form:\footnote{First turn (\ref{equ:pca-lap}) back to the
second line of (\ref{equ:pca to laplacian}), and then by
(\ref{equ:Pythagorean}), it leads to (\ref{equ:original sparse
A-DX}).}
\begin{equation}\label{equ:original sparse A-DX}
\min_{D\in\mathbb{R}^{(p+1)\times r},X\in\mathbb{R}^{(r-1)\times n}}
\|\tilde{A}-D\tilde{X}\|^2_F,\,\st\,\tilde{X}=\begin{bmatrix}
\frac{1}{\sqrt{n}}\mb{1}^T\\
X\end{bmatrix},\,\tilde{X}\tilde{X}^T=I.
\end{equation}
By the SVD of $\tilde{A}$, (\ref{equ:svd tildeA}), the
solution set includes $\tilde{X}^*=\tilde{V}_{1:r}=\begin{bmatrix}
\frac{1}{\sqrt{n}}\mb{1}^T\\
V_{1:r-1}\end{bmatrix}$, $D^*=\tilde{A}\tilde{V}_{1:r}^T$, and any
rotation of them.

\begin{theorem}\label{theo:original ssr}
(\textbf{SSRl}) When the graph is nearly ideal, there is a sparse
representation
\begin{equation}\label{equ:original sparse A=DH=AHH}
A\approx \hat{D}H=A(H^TH).
\end{equation}
The approximation accuracy is optimal in the Frobenius-norm sense.
$H\in\mathbb{R}^{r\times n}$ ($r\leq \rank(A)+1$) is the matrix of
sparse codes, i.e., noisy indicator matrix, which satisfies
\begin{equation}\label{equ:original X=RH}
\begin{bmatrix}
\frac{1}{\sqrt{n}}\mb{1}^T\\
V_{1:r-1}\end{bmatrix}=RH,
\end{equation}
for some rotation matrix $R$. $\hat{D}\in\mathbb{R}^{p\times r}$ is
a dictionary defined as
\begin{equation}\label{equ:original D}
\hat{D}=AH^T=A[\frac{1}{\sqrt{n}}\mb{1},
V_{1:r-1}^T]R=U_{1:r-1}\Sigma_{1:r-1}R_{2:r},
\end{equation}
where $R_{2:r}$ denotes the second to the $r$th rows of $R$. The
Gram matrix of codes $H^TH\in\mathbb{R}^{n\times n}$ reflects the
linear relationship of data, which has property
$\mb{1}^TH^TH=\mb{1}^T\tilde{V}_{1:r}^T\tilde{V}_{1:r}=\mb{1}^T$,
i.e., the sum of each column is one.
\end{theorem}

\begin{proof}
Following similar reasoning of SSRk, (\ref{equ:original X=RH}) can
be obtained, and there is a sparse representation of the
\emph{augmented data}: $\tilde{A}\approx \tilde{D}H=\tilde{A}H^TH$,
where by the SVD of $\tilde{A}$, (\ref{equ:svd tildeA}), and
(\ref{equ:original X=RH}), $\tilde{D}=\tilde{A}H^T=\begin{bmatrix}
1 &\\
&U_{1:r-1}\end{bmatrix}\begin{bmatrix}
\sqrt{\beta}n &\\
&\Sigma_{1:r-1}\end{bmatrix}R$.

We now focus on (\ref{equ:original sparse A=DH=AHH}), where the
representation is for the \emph{original data}. First, we reduce
(\ref{equ:original sparse A-DX}) to a form that involves only the
original data. Substituting $D=\tilde{A}\tilde{X}^T$ into
(\ref{equ:original sparse A-DX}), note that the first row of
$\tilde{A}\tilde{X}^T\tilde{X}$ is always $\sqrt{\beta}\mb{1}^T$,
equal to the first row of $\tilde{A}$. Thus we can remove this
artificial component:
\begin{equation}\label{equ:original sparse A-DX compact}
\min_{\bar{D}\in\mathbb{R}^{p\times r},X}
\|A-\bar{D}\tilde{X}\|^2_F,\,\st\,\tilde{X}=\begin{bmatrix}
\frac{1}{\sqrt{n}}\mb{1}^T\\
X\end{bmatrix},\,\tilde{X}\tilde{X}^T=I.
\end{equation}
The solution of the codes remains the same. Hence, we have $A\approx
\bar{D}^*\tilde{X}^*=A\tilde{X}^{*T}\tilde{X}^*$. Combining with the
solution of (\ref{equ:original sparse A-DX}), and (\ref{equ:original
X=RH}), we obtain SSRl (\ref{equ:original sparse A=DH=AHH}). Finally
it is easy to verify that (\ref{equ:original D}) holds.
\end{proof}

Though augmenting $X$ with a constant component
$\frac{1}{\sqrt{n}}\mb{1}$, (\ref{equ:original sparse A-DX compact})
is equivalent to PCA (\ref{equ:pca A-DX}), due to
$A\mb{1}=\mb{0}$ and $\|A(I-\tilde{X}^{T}\tilde{X})\|^2_F
=\|A(I-\frac{1}{\sqrt{n}}\mb{1}\frac{1}{\sqrt{n}}\mb{1}^T)(I-X^{T}X)\|^2_F
=\|A(I-X^{T}X)\|^2_F$. However, PCA cannot explicitly reveal the
cluster structure, because in the absence of the redundant
$\frac{1}{\sqrt{n}}\mb{1}$, the sparse codes cannot be recovered
through rotating $V_{1:r-1}$ only.

The interpretations are similar to those of SSRk and an un-normalized SSRl exists, except that the dictionaries
are usually not near-orthogonal. This is because the singular values are not uniform, which usually decay very
fast. The more fundamental reason underlying this phenomenon is that, as argued before, the ideal graph condition
essentially requires that after translating along a new dimension the clusters become orthogonal, which can hardly
be met in practice, except perhaps for some high-dimensional data. Nevertheless, linear models, as elementary
components in the family of machine learning, usually lay down important parts of the theoretical foundation and
provide support for more advanced methods, therefore should not be underestimated.

Finally, we mention that the virtual data in the kernel version and the augmented data including the auxiliary
constant $\beta$ in the linear version would not play actual roles in practice, their functions are to
establish theory and facilitate understanding. However, the constant vector $\mb{1}$ discarded by dimensionality
reduction indeed is indispensable for revealing cluster structure and therefore plays essential roles in SSR and
cluster analysis.

\subsection{SSR for Out-of-sample Data}
Within unsupervised learning domain, SSR solves sparse codes for a given sample set. When new data (called out-of-sample data) come, solving the sparse codes of them is not necessarily straightforward, especially for the kernel version \cite{bengio2004out}. To make SSR fully useful, we have to be able to address this problem.

\subsubsection{Kernel Version}
There are some prior work dealing with the out-of-sample problem for Ncut
\cite{bengio2004out, bengio2004learning}. However, the kernel function of
LE/Rcut is different, we will solve the problem in our context.

Given a new point $b$, we denote its out-of-sample LE codes by $V_b\in\mathbb{R}^{r}$, and its sparse codes by
$H_b$. By the rotation relation between sparse codes and codes of LE (\ref{equ:kernel X=RH}), if we can get
$V_{b}$, then $H_b$ can be obtained using the same rotation matrix:
\begin{equation}\label{equ:Hb}
H_b=R^TV_b.
\end{equation}
First, we study the ideal case, and solve $V_b$.
\begin{theorem}
Given a new point $b$ with similarities to the data set $W_b=[W_{1,b},\cdots,W_{n,b}]^T$, and denoting the sum of
$W_b$ by $s_b$, if the ideal graph condition is satisfied and $b$ is not an overlapping point, then $b$'s LE codes
 are
\begin{equation}\label{equ:SSRk V_b Lambda}
V_b=(s_bI-\Lambda_{1:r})^{-1}V_{1:r}W_b,
\end{equation}
or, since $\Lambda_{1:r}=\mb{0}$,
\begin{equation}\label{equ:SSRk V_b}
V_b=V_{1:r}(W_b/s_b).
\end{equation}
\end{theorem}

\begin{proof}
By assumption and properties of Laplacian matrix, $\tilde{H}=[H,H_b]$ are eigenvectors of the augmented Laplacian
matrix:
\begin{equation}\label{equ:agumented L}
\tilde{L}=
\begin{bmatrix}
\tilde{S}-W&-W_b\\
-W_b^T& s_b
\end{bmatrix}.
\end{equation}
That is $\tilde{L}\tilde{H}^T=\tilde{H}^T\Lambda_{1:r}$. $\Lambda_{1:r}$ are virtually zero, so we also have
$\tilde{L}\tilde{H}^TR^T=\tilde{H}^TR^T\Lambda_{1:r}$, which implies $\tilde{V}_{1:r}=R\tilde{H}=[V_{1:r},V_{b}]$
are also eigenvectors. We see that the eigenvectors of sample set, $V_{1:r}$, are preserved during the extension.
Substituting (\ref{equ:agumented L}) into $\tilde{L}\tilde{V}_{1:r}^T=\tilde{V}_{1:r}^T\Lambda_{1:r}$ and
considering the last row, we have $-W_b^TV_{1:r}^T+s_bV_{b}^T=V_{b}^T\Lambda_{1:r}$. Therefore, (\ref{equ:SSRk
V_b Lambda}) is obtained.
\end{proof}

When the graph is nearly ideal, we will still estimate the LE codes by (\ref{equ:SSRk V_b Lambda}).\footnote{When
$s_b=0$, as the sample set of kernel PCA satisfy, (\ref{equ:SSRk V_b Lambda}) reduces to the out-of-sample codes
of kernel PCA \cite{scholkopf1998nonlinear}. Besides, similar results can be obtained by studying the normalized
Laplacian matrix $S^{-1}W$ (or optionally applying the Nystr\"{o}m formula \cite{Williams2001Using,
bengio2004learning}), since the eigenvectors of $L$ with eigenvalue zero (smallest) are the eigenvectors of
$S^{-1}W$ with eigenvalue one (largest) \cite{von2007tutorial}. (\ref{equ:SSRk V_b Lambda}) will be replaced with
$V_b=(I-\Lambda_{1:r})^{-1}V_{1:r}(W_b/s_b)$, while (\ref{equ:SSRk V_b Lambda}) remains unchanged.} Further, since
$\Lambda_{1:r}\approx \mb{0}$, especially orders of magnitude smaller than $s_b$ in practice, we can still use
(\ref{equ:SSRk V_b}) as approximation. By (\ref{equ:Hb}), finally we have:
\begin{theorem}
The sparse codes of $b$ can be estimated as
\begin{equation}\label{equ:SSRk H_b}
H_b=H(W_b/s_b).
\end{equation}
\end{theorem}

Both (\ref{equ:SSRk V_b}) and (\ref{equ:SSRk H_b}) have very clear interpretation: the out-of-sample codes are
obtained by the weighted combination of the codes of sample set, and the weights are nonnegative and sum to 1.
Besides, note that when the graph is ideal, replacing the new data with sample set, (\ref{equ:SSRk V_b}) and
(\ref{equ:SSRk H_b}) lead to exactly the sample-set codes, $V_{1:r}$ and $H$. This is due to the properties of
normalized Laplacian matrix $S^{-1}W$, cf. footnote~12. Finally, as the derivation does not depend on the
row-length of $H$, if we use un-normalized SSR (\ref{equ:unnormalize ssr}), there is an alternative:
$H_b=H_U(W_b/s_b)$, and $\mb{1}^TH_b=1$ too.

\subsubsection{Linear Version}
This version is simpler, since the out-of-sample codes of PCA is easy to
obtain, and a rotation of it leads to sparse codes. Nevertheless, we investigate it systematically for deeper
understanding.

First, note the solution of the linear regression problem:
\begin{equation}\label{equ:DT}
\min_{D_T\in\mathbb{R}^{r\times (p+1)}} \|\tilde{X}-D_T\tilde{A}\|^2_F.
\end{equation}
$(D_T)^*=\tilde{X}\tilde{A}^{\dag}$. When $\tilde{X}$ is the normalized PCs, $\tilde{V}_{1:r}$, by the SVD of
$\tilde{A}$ (\ref{equ:svd tildeA}), $(D_T)^*=\tilde{\Sigma}_{1:r}^{-1}\tilde{U}_{1:r}^T=(D^*)^{\dag}$ and
$\tilde{X}=(D_T)^*\tilde{A}$, where $D^*$ is the solution of (\ref{equ:original sparse A-DX}). Similarly, when
$\tilde{X}$ is the sparse codes, $H$, we have solution
$\hat{D}_T=R^T\tilde{\Sigma}_{1:r}^{-1}\tilde{U}_{1:r}^T=\hat{D}^{\dag}$, and $\tilde{X}=\hat{D}_T\tilde{A}$,
where $\hat{D}$ is the dictionary of SSRl. The two dictionaries are related by a rotation $\hat{D}_T=R^T(D_T)^*$,
which is the counterpart of $\hat{D}=D^*R$. The above principle suggests that
\begin{theorem}
Given a new point $b$ (with mean removed as $A$), and denoting its augmented data by $\tilde{b}=[\sqrt{\beta},b^T]^T$, then its augmented
PCA codes and sparse codes can be obtained by
\begin{equation}\label{equ:SSRl V_b}
\tilde{V}_b=(D^*)^{\dag}\tilde{b}=\begin{bmatrix}
\frac{1}{\sqrt{n}}\\
\Sigma_{1:r-1}^{-1}U_{1:r-1}^Tb\end{bmatrix},
\end{equation}
\begin{equation}\label{equ:SSRl H_b}
H_b=\hat{D}^{\dag}\tilde{b}=R^T\tilde{V}_b.
\end{equation}
\end{theorem}
The proof of the final expression of $\tilde{V}_b$ involves tedious expansion of $\tilde{A}$'s SVD (\ref{equ:svd
tildeA}), which we will omit. Note that, $\Sigma_{1:r-1}^{-1}U_{1:r-1}^Tb$ is exactly the PCA codes of $b$, and
the auxiliary constant $\beta$ does not play actual role. Besides, by the orthonormality of $\tilde{V}_{1:r}$ and
$H$, we have:
\begin{corollary}
\begin{equation}
\tilde{V}_b=\begin{bmatrix}
\frac{1}{\sqrt{n}}\\
V_b\end{bmatrix}.
\end{equation} The out-of-sample codes can also be written in terms of sample-set codes:
$\tilde{V}_b=\tilde{V}_{1:r}(\tilde{V}_{1:r}^T\tilde{V}_b)$, $H_b=H(\tilde{V}_{1:r}^T\tilde{V}_b)$, where vector
$\tilde{V}_{1:r}^T\tilde{V}_b=\frac{1}{n}\mb{1}+V_{1:r}^TV_b$ and $\mb{1}^T(\tilde{V}_{1:r}^T\tilde{V}_b)=1$.
\end{corollary}
The corollary implies that the out-of-sample codes is a weighted combination of the codes of sample set. The
weights are similarities defined by the inner product of PCs, and they sum to 1. These expressions are consistent
with those of kernel version. Original codes can be recovered by replacing the new data with sample set, and an
alternative by normalizing $H_b$ to sum 1 exists. Finally, in view of (\ref{equ:SSRl V_b}) and (\ref{equ:SSRl
H_b}), the codes can be directly obtained by applying a linear transform to the data. It implies that SSRl is both
synthesis SR and analysis SR/sparsifying transform \cite{Elad2007Analysis,Nam2013The,Ravishankar2013Learning}.

\subsection{NSCrt: to Solve Sparse Codes}\label{sec:NSCrt}

To make SSR practical, we have to find the rotation matrices in (\ref{equ:kernel X=RH}) and (\ref{equ:original
X=RH}) accurately. (\ref{equ:kernel X=RH}) or (\ref{equ:original X=RH}) essentially requires to find a rotation
matrix and sparse codes such that the sparse codes match the normalized PCs after the rotation. We employ a
modified version of SPCArt \cite{Hu2016Sparse} to accomplish this task. SPCArt is a sparse PCA algorithm designed
to solve sparse loadings. It finds a rotation matrix and sparse loadings such that the sparse loadings approximate
the PCA loadings after the rotation. Replacing the PCA loadings with the normalized PCs, SPCArt meets our need.
Considering our sparse codes are noisy indicators, the dominant values of which are nonnegative, we additionally
impose nonnegative constraint on the codes. The modified algorithm is called NSCrt (nonnegative sparse coding via
rotation and truncation).

Given a row-wise orthonormal matrix $X\in \mathbb{R}^{r\times n}$,
the objective of NSCrt is
%\footnote{This version is based on
%$\ell_0$-norm penalty. Extensions to other versions that are based
%on e.g., $\ell_1$-norm penalty, $\ell_0$-norm constraint are
%trivial. Empirically, we observed that the current version performs
%best.}
\begin{equation}\label{equ:NSCrt}
\min_{R\in\mathbb{R}^{r\times r},\bar{H}\in \mathbb{R}^{r\times n}}
\|X-R\bar{H}\|^2_F+\lambda^2\|\bar{H}\|_0,\,\st\,R^TR=I,\,\bar{H}_{ij}\geq
0,\forall i,j,
\end{equation}
where $R$ is the rotation matrix, $\bar{H}$ is the matrix of
truncated sparse codes, $\|\bar{H}\|_0$ counts the number of nonzero
entries of $\bar{H}$, $\lambda$ is a small threshold in $(0,1)$.
Since we want to find an equivalence relation between $X$ and $RH$,
as (\ref{equ:kernel X=RH}) and (\ref{equ:original X=RH}) require,
rather than an approximation relation, after $R$ is solved, we
obtain sparse codes as $H=R^TX$.

% algorithm NSCrt ------------------------
\begin{algorithm}[h]
   \caption{NSCrt for solving sparse codes}
   \label{alg:NSCrt}
\begin{algorithmic}[1]
   \REQUIRE ~~\\
   row-wise orthonormal eigenvectors $X\in\mathbb{R}^{r\times n}$, threshold
   $\lambda\in (0,1)$
   \ENSURE ~~\\
   sparse codes $H\in\mathbb{R}^{r\times
   n}$ and its truncated version $\bar{H}$, rotation matrix $R\in\mathbb{R}^{r\times
   r}$

   \STATE Initialize $R$: $R\leftarrow I$

   \REPEAT
   \STATE Rotation: $H\leftarrow R^TX$
   \STATE Truncation: $\forall i,j$, $\bar{H}_{ij}\leftarrow H_{ij}$ if $H_{ij}\geq
   \lambda$,
and $\bar{H}_{ij}\leftarrow 0$ otherwise
   \STATE Update $R$: compute SVD of $X\bar{H}^T$: $U\Sigma V$, then $R\leftarrow UV$
   \UNTIL{convergence}
   \STATE Obtain final sparse codes: $H\leftarrow R^TX$
\end{algorithmic}
\end{algorithm}
%----------------------------------------

% algorithm SSR-K ------------------------
\begin{algorithm}[htb]
   \caption{SSR of kernel version (SSRk)}
   \label{alg:SSR-K}
\begin{algorithmic}[1]
   \REQUIRE ~~\\
   similarity matrix $W\in\mathbb{R}^{n\times
   n}$, code dimension $r$, threshold
   $\lambda$, out-of-sample similarity vector $W_b\in\mathbb{R}^{n}$ (optional)
   \ENSURE ~~\\
   sparse codes $H\in\mathbb{R}^{r\times
   n}$, rotation matrix $R\in\mathbb{R}^{r\times
   r}$, out-of-sample sparse codes $H_b\in\mathbb{R}^{r}$ (optional)

   \STATE Compute Laplacian matrix: $L\leftarrow \diag(W\mb{1})-W$
   \STATE Compute the $r$ smallest eigenvectors of $L$: $V\in\mathbb{R}^{r\times
   n}$
   \STATE Solve sparse codes: $\{H,R\}\leftarrow$ NSCrt($V,\lambda$)
   \STATE If $W_b$ is provided, $H_b\leftarrow HW_b/(\mb{1}^TW_b)$
\end{algorithmic}
\end{algorithm}
%----------------------------------------

% algorithm SSR-O ------------------------
\begin{algorithm}[htb]
   \caption{SSR of linear version (SSRl)}
   \label{alg:SSR-O}
\begin{algorithmic}[1]
   \REQUIRE ~~\\
   mean-removed data set $A\in\mathbb{R}^{p\times
   n}$, code dimension $r$ ($r\leq \rank(A)+1$), threshold
   $\lambda$, mean-removed out-of-sample data $b\in\mathbb{R}^{p}$ (optional)
   \ENSURE ~~\\
   sparse codes $H\in\mathbb{R}^{r\times
   n}$, rotation matrix $R\in\mathbb{R}^{r\times
   r}$, out-of-sample sparse codes $H_b\in\mathbb{R}^{r}$ (optional)

   \STATE Compute rank $r-1$ SVD of $A$: $U\Sigma V$
   \STATE Solve sparse codes: $\{H,R\}\leftarrow$ NSCrt($\begin{bmatrix}
\frac{1}{\sqrt{n}}\mb{1}^T\\ V\end{bmatrix},\lambda$)
   \STATE If $b$ is provided, $H_b\leftarrow R^T\begin{bmatrix}
\frac{1}{\sqrt{n}}\\
\Sigma^{-1}U^Tb\end{bmatrix}$
\end{algorithmic}
\end{algorithm}
%----------------------------------------

Note (\ref{equ:NSCrt}) itself is a dictionary learning problem, but the characteristics of its operands make the
solution elegant. Following SPCArt, a local optimum can be solved by alternately optimizing $R$ and $\bar{H}$.
When initializing $R=I$, the solution process results into operations of alternately rotating and truncating
$X$.

1) Fixing $R$, (\ref{equ:NSCrt}) becomes
\begin{equation}\label{equ:NSCrt H}
\min_{\bar{H}}
\|R^TX-\bar{H}\|^2_F+\lambda^2\|\bar{H}\|_0,\,\st,\,\bar{H}_{ij}\geq
0,\forall i,j.
\end{equation}
Denote $H=R^TX$, note that $H$ is a rotation of $X$, which is
orthonormal and spans the same subspace as $X$. It is not hard to
see that the solution of $\bar{H}$ is: $\bar{H}_{ij}^*=H_{ij}$ if
$H_{ij}\geq \lambda$, and $\bar{H}_{ij}^*=0$ otherwise, i.e., it is
obtained by truncating small values of $H$ that are below $\lambda$.

2) Fixing $\bar{H}$, (\ref{equ:NSCrt}) becomes a Procrustes problem
\begin{equation}\label{equ:NSCrt R}
\min_{R} \|X\bar{H}^T-R\|^2_F,\,\st\,R^TR=I.
\end{equation}
Let $X\bar{H}^T=U\Sigma V$ be the SVD, then $R^*=UV$.

The NSCrt algorithm is presented in Algorithm~1, and those of SSRk and
SSRl are presented in Algorithm~2 and Algorithm~3 respectively. The
time complexity of NSCrt is $O(nr^2)$, which scales linearly with
the data size. It is efficient if $r$ is not too large.

%According to the performance-guarantee analysis in SPCArt, we may
%choose $\lambda$ around $1/\sqrt{n}$. In our case, we found
%$\lambda=0.6/\sqrt{n}$ consistently performs well.

% algorithm sparse cut ------------------------
\begin{algorithm}[htbp]
   \caption{Sparse cut (Scut)}\label{alg:sparse cut}
\begin{algorithmic}[1]
   \REQUIRE ~~\\
   sparse codes $H\in\mathbb{R}^{r\times n}$ output by SSR (or out-of-sample sparse codes)
   \ENSURE ~~\\
   cluster labels $c\in\mathbb{N}^{1\times n}$
   \STATE $c_i\leftarrow \mathop{\arg\max}_{k=1,\dots,r}\,H_{ki}$, $i=1,\dots,n$
\end{algorithmic}
\end{algorithm}
%----------------------------------------

\subsection{Sparse Cut (Scut): Application of SSR in Clustering}\label{sec:scut}
As an application of SSR, the sparse codes $H$ can be directly used
for clustering. Since the sparse codes are noisy indicator vectors,
where usually only one dominant value appears in each column, we can
check the maximal entry in each column and assign its index as the
cluster label:\footnote{There is a more concrete interpretation for
Scut in the case of SSRk, details are included in
Appendix~\ref{app:scut}.}
\begin{equation}
c_i \leftarrow \mathop{\arg\max}_k\,H_{ki},\,1\leq i\leq n.
\end{equation}

We call this SSR based clustering, sparse cut (Scut), shown in Algorithm~4. In brief, combining with SSR, Scut performs
the following steps: 1) compute the normalized PCs of data (SSRl),
or the eigenvectors of Laplacian matrix (SSRk), 2) employ NSCrt to
recover the noisy indicator vectors from the eigenvectors, 3) finish
clustering by checking the maximal entries.

There is another alternative with minor difference: normalizing each column of $H$ to sum 1 first, and then $c_i
\leftarrow \mathop{\arg\max}_k\,{H_U}_{ki}$, where $H_U$ is the code matrix of the un-normalized SSR
(\ref{equ:unnormalize ssr}). The quasi-probability interpretation of $H_U$ justifies this scheme. The difference is:
the first scheme weights the smaller clusters more, since $1/\sqrt{n_i}<1/\sqrt{n_j}$ if $n_i>n_j$, while the latter
scheme treats them equally, since all columns sum to 1. This paper adopts the first scheme.

\section{Experiments}\label{sec:experiments}
The experiments consist of: 1) validating the ability of NSCrt in recovering rotation matrix, 2) illustration of the
cluster structure revealed by sparse codes, 3) a comparison between the linear version and the kernel version,
in terms of the ideal graph condition and clustering performance, 4) investigation of the relations between ideal graph
condition, sparsity, and clustering accuracy, 5) the performance of kernel Scut, 6) comparison of the clustering
performance between SSR and OSRs.

\begin{table*}[h]
\caption{Data sets. }\label{tab:dataset} \vskip 0.1in
\begin{center}
%\vskip -0.1in
\begin{scriptsize} %
\begin{tabular}{|m{1.1cm}||m{7.4cm}|l|m{1.4cm}|m{1.9cm}|}
\hline Data set & Description & \#Classes & Size & Sizes of classes\\\hline\hline

G1,G2,G3 & three artificial Gaussian data with more and more heavy
overlaps, shown in Figure~\ref{fig:gaussClassPoints}  & 3 &
2$\times$150 & 50,50,50\\\hline

onion & an artificial data of unbalanced classes, shown in
Figure~\ref{fig:onionClassPoints} & 3 & 2$\times$75 &
5,20,50\\\hline

iris & Fisher's iris flower data set & 3 & 4$\times$150 &
50,50,50\\\hline

wdbc & breast cancer Wisconsin (diagnostic) data set & 2 &
30$\times$569 & 212,357\\\hline

Isolet &  spoken (English) letter recognition data set, subset of
the first set, excluding classes ``c,d,e,g,k,n,s'' & 19 & 617$\times
$1140 & each 60\\\hline

USPS3, USPS8, USPS10 & three subsets of the training set of United
States Postal Service (USPS) handwritten digit database. USPS3:
``0''-``2'', USPS8: excludes ``5'' and ``9'' (as ``3'' and ``5'',
``4'' and ``9'', ``7'' and ``9'' heavily overlap), USPS10:
``0''-``9'' & 3,8,10 & 256$\times$2930, 256$\times$6091,
256$\times$7291 & 1194,1005,731, 658,652,556, 664,645,542,644\\\hline

4News & four groups \{2,9,10,15\} of 20 Newsgroups documents, tf-idf
sparse features are used& 4 & 26214$\times$2372 & 389,398,397,394\\\hline

TDT2 & the largest 30 categories of NIST Topic Detection and
Tracking corpus, documents appearing in more than one categories are removed, tf-idf
sparse features are used& 30 & 36771$\times$9394 &
1844,1828,1222, 811,...,52\\\hline

polb & a relation network (sparse) of books about US Politics
(``liberal'', ``conservative'', or ``neutral''), edges represent
copurchasing of books by the same buyers & 3 & 105$\times$105 &
43,49,13\\\hline

\end{tabular}
\end{scriptsize}
\end{center}
\vskip -0.1in
\end{table*}

The data sets we used, shown in
Table\ref{tab:dataset},\footnote{The sources of the public data sets are listed below.\\
iris: \url{http://archive.ics.uci.edu/ml/datasets/Iris} \\
Isolet: \url{http://archive.ics.uci.edu/ml/datasets/ISOLET}\\
wdbc:
\url{http://archive.ics.uci.edu/ml/machine-learning-databases/}\\\url{breast-cancer-wisconsin/}\\
USPS: \url{http://www-i6.informatik.rwth-aachen.de/~keysers/usps.html}\\
TDT2: \url{http://www.nist.gov/speech/tdt98/tdt98.htm}\\
4News: \url{http://qwone.com/~jason/20Newsgroups/} \\
polb: \url{http://networkdata.ics.uci.edu/data.php?id=8}} contain class labels for each sample, which serve as
ground-truth for the evaluation. To make the evaluation with the
class labels reasonable, the data sets are chosen so that the underlying clusters of the data
are in good accord with the man-assigned class labels. The data sets come from various domains and are of
diverse nature. G1, G2, and G3 have more and more heavy
cluster-overlaps. The class sizes of onion and TDT2 are highly
unbalanced. The number of samples in USPS10 and TDT2 are large.
4News and TDT2 are high-dimensional data sets with $p>n$. 4News and
TDT2 are sparse data. The number of classes in Isolet and TDT2 are
relatively large. Polb is a relational data with only similarity
matrix provided.

Clustering performance is measured by four criteria: 1) accuracy
(percentage of total correctively classified points), 2) normalized
mutual information (NMI) \cite{manning2008introduction}, 3) rand
index (RI) \cite{manning2008introduction}, 4) time cost.
For accuracy, the matching between the output clusters and the
labeled classes is established by the Hungarian algorithm
\cite{kuhn1955hungarian}.

%Similarity matrices of the kernel version (except polb) are built by
%4NN graph (4News use 9NN) \cite{von2007tutorial} and Gaussian
%kernel $\phi(A_i,A_j)=\exp\{-0.5\|A_i-A_j\|_2^2/\upsilon\}$ with
%$\upsilon$. \footnote{$\upsilon$ is set to be the mean of class
%variances computed using the class labels. It is set in this
%convenient way to avoid the boring tuning problem. It would not
%affect the conclusions, so long as all the algorithms receive the
%same input. But for large data sets USPS10, 4News, and TDT2, we
%apply the self-tuning method of \cite{zelnik2004self}, since it is
%faster, though less reasonable.}
The algorithms are implemented using MATLAB, run on a PC with
2.93GHz duo core CPU, 2GB memory. The similarity matrices of the
kernel version (except polb) are built by 4NN graph (4News uses 9NN
to avoid isolated points) with the self-tuning method of
\cite{zelnik2004self}. For NSCrt, we set
$\lambda=0.6/\sqrt{n}$,\footnote{According to the
performance-guarantee analysis in SPCArt \cite{Hu2016Sparse}, it is
recommended to set $\lambda$ around $1/\sqrt{n}$. In our case, we
found $\lambda=0.6/\sqrt{n}$ consistently performed well.} 200
maximal iterations, and $\|R^{(t)}-R^{(t-1)}\|_F/\sqrt{r}\leq 0.01$
to be the convergence condition. We observed it usually converged
within a dozen iterations.

\begin{figure*}[htbp]
\centering{
\subfigure[$r=2$]{\label{fig:acc_NSCrt:r2}\includegraphics[width=3.5cm]{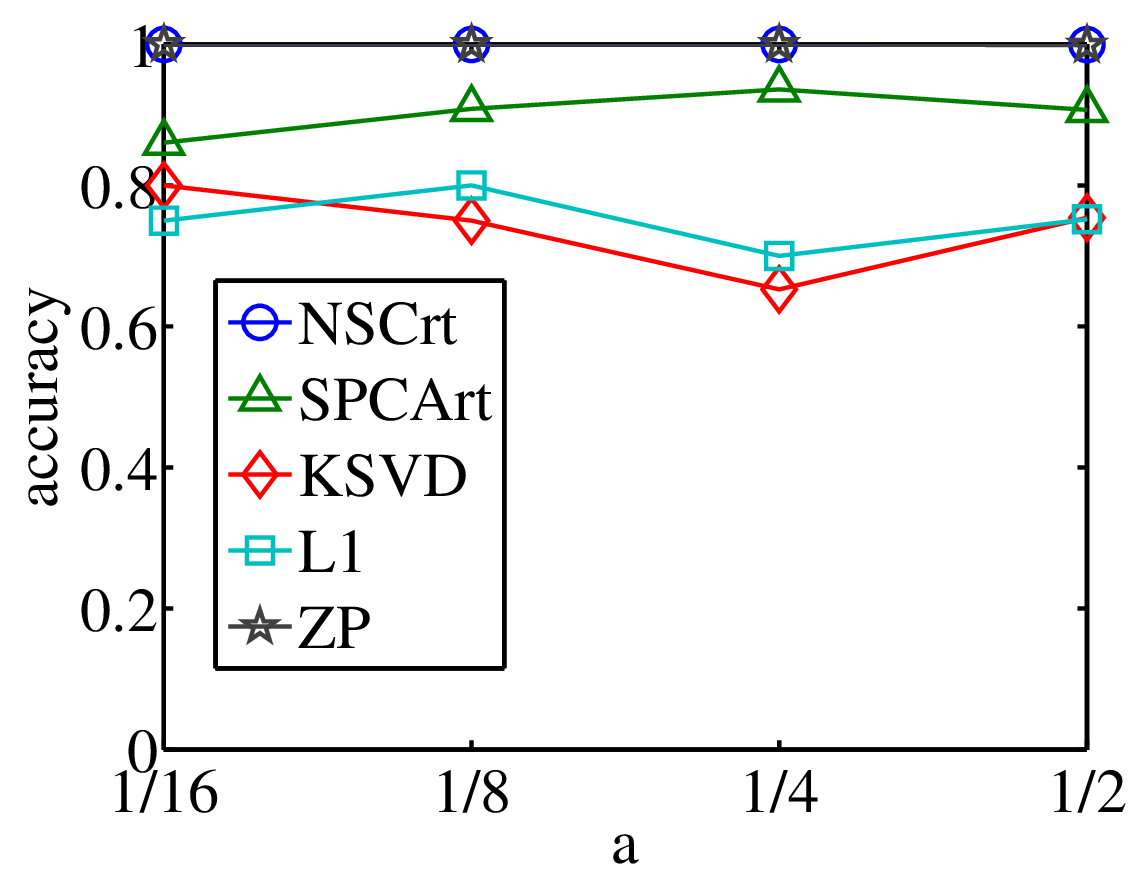}}\hspace{2mm}
\subfigure[$r=16$]{\label{fig:acc_NSCrt:r16}\includegraphics[width=3.5cm]{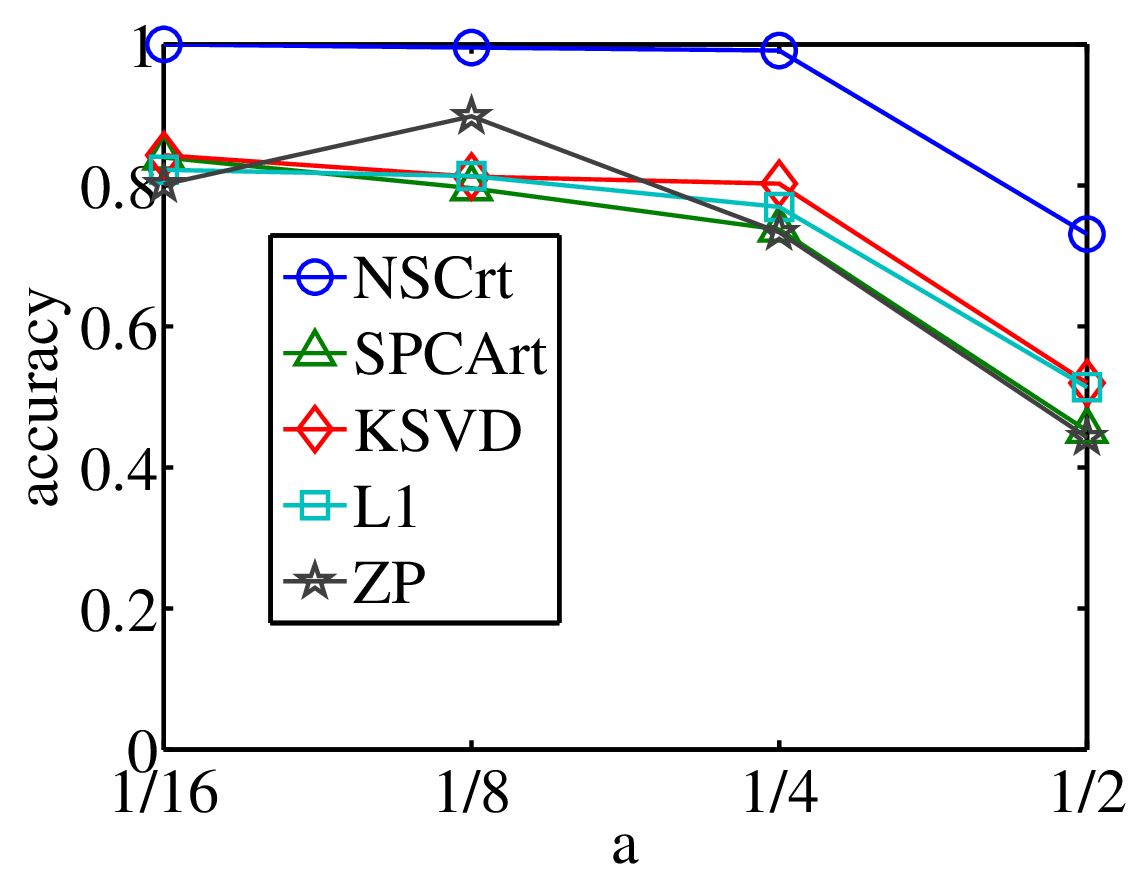}}\hspace{2mm}
\subfigure[$r=128$]{\label{fig:acc_NSCrt:r128}\includegraphics[width=3.5cm]{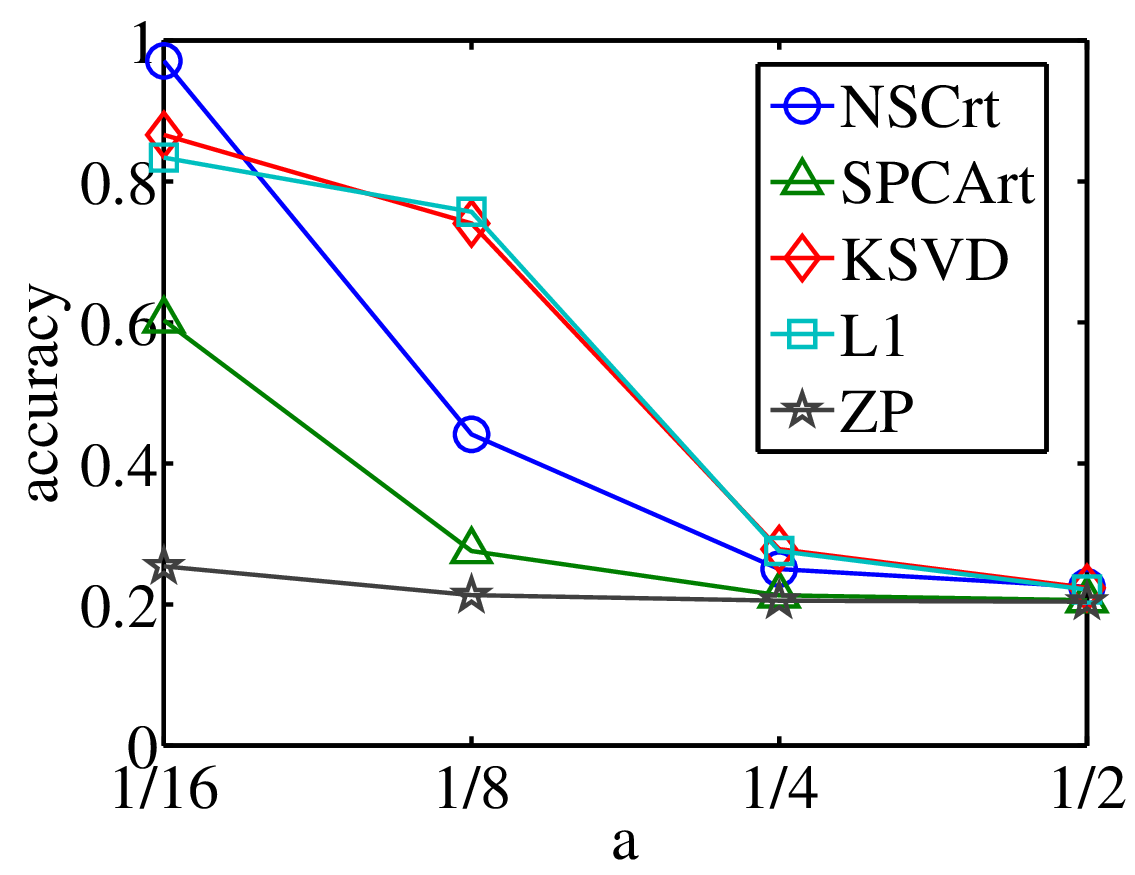}}\hspace{2mm}
\subfigure[$r=9$]{\label{fig:acc_NSCrt:r9}\includegraphics[width=3.5cm]{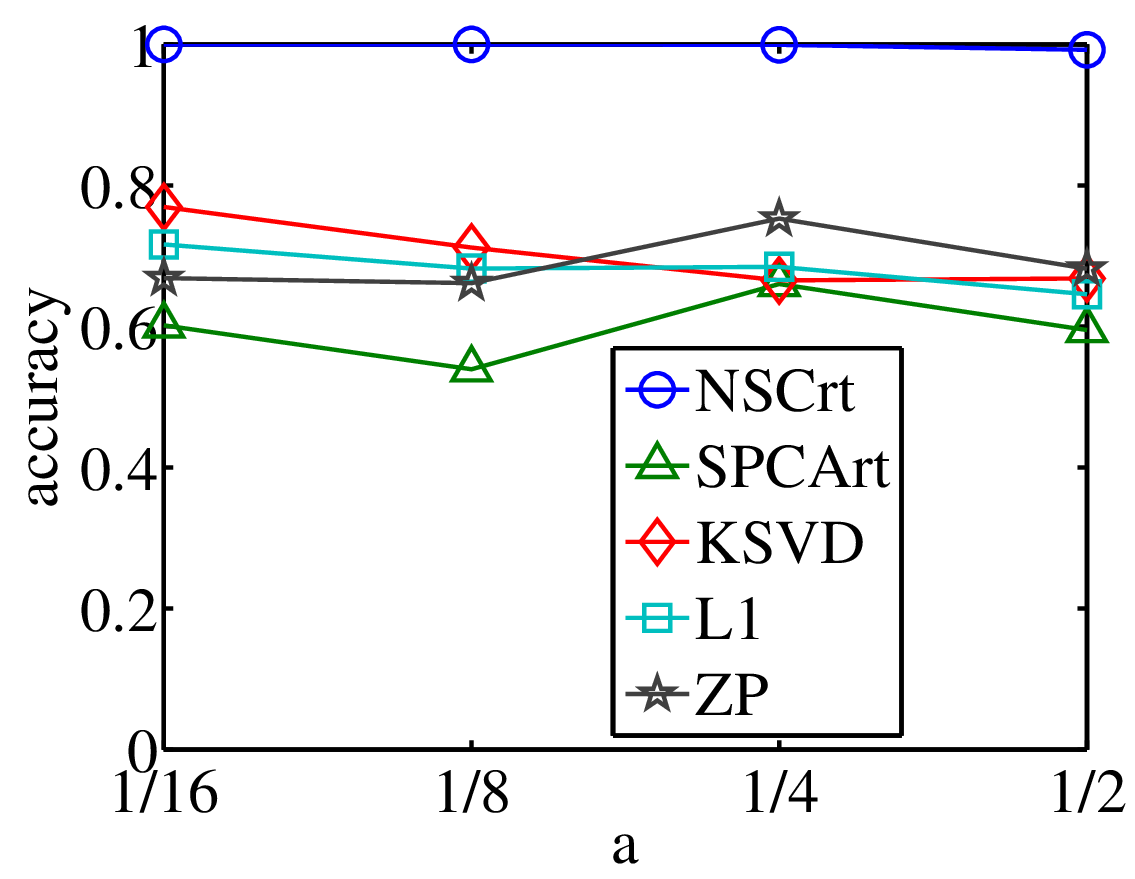}}
} \caption{Accuracy of NSCrt in recovering rotation matrix. The
horizontal axis indicates noise level. (a)-(c) data of uniform
cluster-sizes. (d) data of exponential
cluster-sizes.}\label{fig:acc_NSCrt}
\end{figure*}

\subsection{Accuracy of NSCrt in Recovering Rotation Matrix}\label{sec:acc NSCrt}
SSR is practical and Scut is feasible only if we can find the right
rotation matrix, so we test the accuracy of NSCrt in recovering the
rotation matrix first. Randomly generated rotation matrices and
sparse codes are used for the test. The performance is compared with
those of four other algorithms: $\ell_0$-norm based SPCArt \cite{Hu2016Sparse}, ZP \cite{zelnik2004self},
KSVD \cite{aharon2006img} and an $\ell_1$-norm based KSVD, denoted by ``L1''.\footnote{Codes of KSVD and ZP are downloaded
from \\\url{http://www.cs.technion.ac.il/~ronrubin/software.html}
and\\
\url{http://www.vision.caltech.edu/lihi/Demos/SelfTuningClustering.html}
respectively. KSVD and L1 are dictionary learning methods. They
participate in the comparison for the recovery of $R$ can be
formulated as a dictionary learning problem under sparse
representation framework, as NSCrt does. In KSVD, the $\ell_0$-norm
of each sparse code-vector is constrained to be 1. For L1, the
dictionary update step follows KSVD, while the sparse coding step
adopts the $\ell_1$-norm based SLEP \cite{Liu:2009:SLEP:manual}.}

First, a rotation matrix
$R\in\mathbb{R}^{r\times r}$, a normalized indicator matrix
$H^*\in\mathbb{R}^{r\times 1024}$, and Gaussian noise $E$ of mean zero are randomly generated.
Then, data $X$ is synthesized via $X\leftarrow R(H^*+E)$, where
$H^*+E$ simulates the sparse codes. We input $X$ into the algorithms
and test their accuracies in recovering $R$. Two kinds of data sets
are generated. 1) Data of uniform cluster-sizes. Each row of $H^*$
has the same number of nonzeros $n_k=n/r$, $1\leq k\leq r$, with
entry value $1/\sqrt{n_k}$. Three number of clusters are tested: $r=2, 16, 128$. 2) Data of exponential
cluster-sizes. $r=9$. The $k$th row ($k<9$) of $H^*$ has $n_k=2^k$
nonzero entries with value $1/\sqrt{n_k}$, the last row has
$n_9=(n-\sum_{k=1}^{8} n_k)$ nonzero entries with value
$1/\sqrt{n_9}$. Note that the smallest cluster has only 2 members,
while the largest one has 514 members, the cluster sizes are highly
unbalanced. For each case above, the algorithms are tested under
increasing Gaussian noise $\sigma=a\min_k\,1/\sqrt{n_k}$,
where $a$ is a factor relative to the smallest value of codes, $a=1/16, 1/8, 1/4, 1/2$. On the highest level, the standard deviation of Gaussian has magnitude up to half of that of data.

The accuracy is measured by the mean of cosines between the
estimated rotation $\hat{R}$ and $R$: $1/r\sum_{k=1}^r
|\hat{R}_k^TR_k|$, where the matching of columns is established by
the Hungarian algorithm \cite{kuhn1955hungarian}. The mean
accuracies over 20 runs are shown in Figure~\ref{fig:acc_NSCrt}. We
see that NSCrt outperforms the others. The improvement is
most significant when the cluster sizes are unbalanced
(Figure~\ref{fig:acc_NSCrt:r9}), where the margin is more than 20\%.
Besides, NSCrt frequently obtains mean accuracies over 98\%, which
indicates the standard deviations are very small, in other words,
the performance of NSCrt is stable. Note that although NSCrt intends
to find local minima, in moderate noise level, it recovers the
underlying solutions with high accuracy.

\begin{figure*}[htbp]
\centering{
\subfigure[G1]{\label{fig:gaussClassPoints:g1}\includegraphics[width=4.3cm]{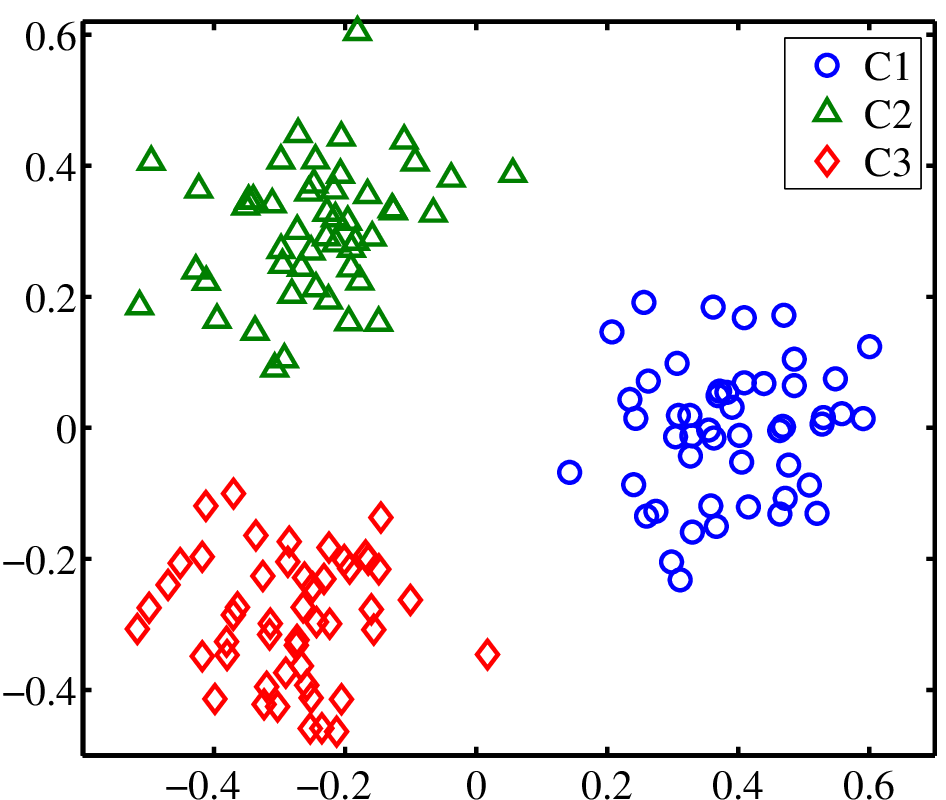}}\hspace{8mm}
\subfigure[G2]{\label{fig:gaussClassPoints:g2}\includegraphics[width=4.3cm]{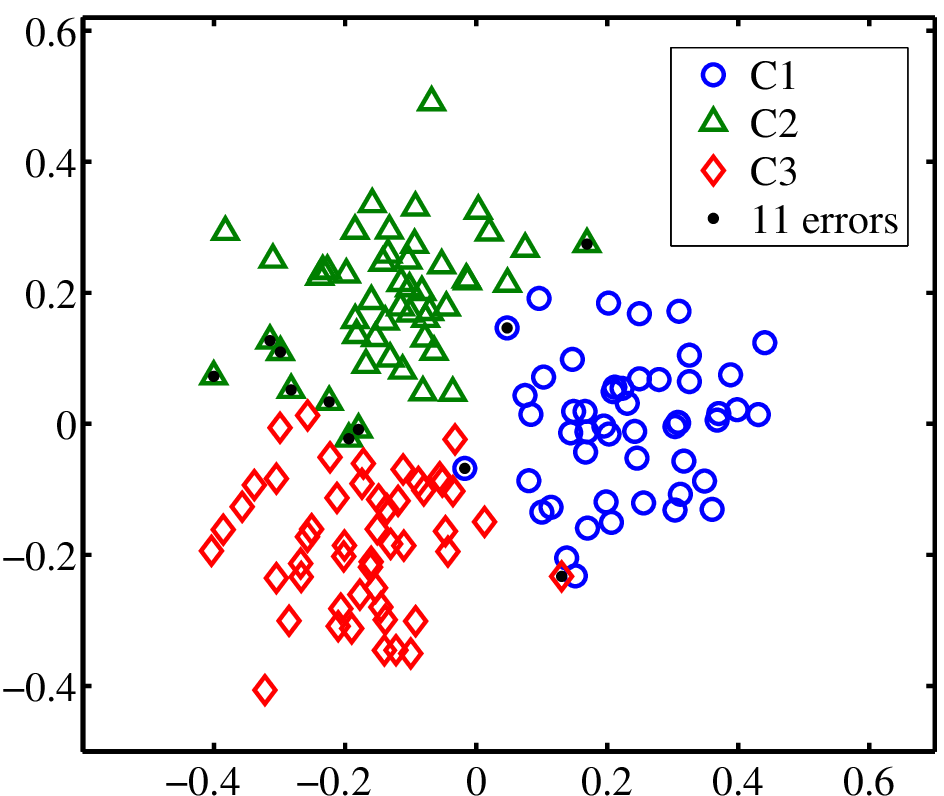}}\hspace{8mm}
\subfigure[G3]{\label{fig:gaussClassPoints:g3}\includegraphics[width=4.3cm]{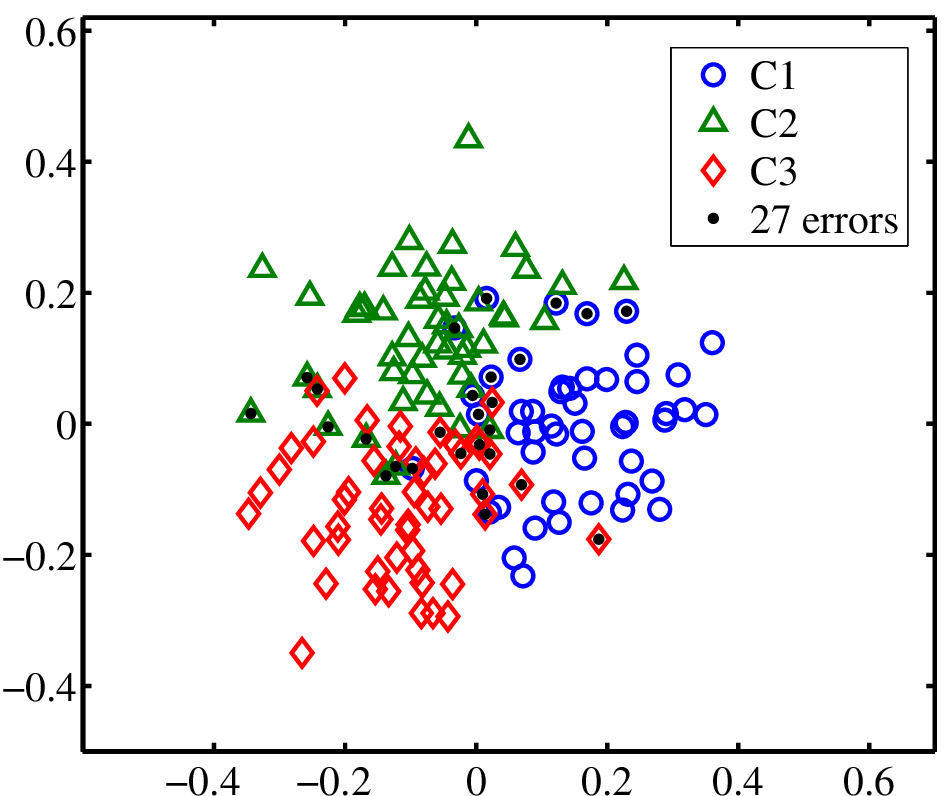}}
} \caption{Gaussian data G1-G3 of different overlaps. Black dots
indicate points misclassified by kernel Scut (see
Section~\ref{sec:cluster info}). These points lie on the overlapping
regions.}\label{fig:gaussClassPoints}
\end{figure*}

\begin{figure*}[htbp]
\centering{ \subfigure[SSRk,
G1]{\label{fig:gausseigenvectorLap:g1}\includegraphics[width=5.0cm]{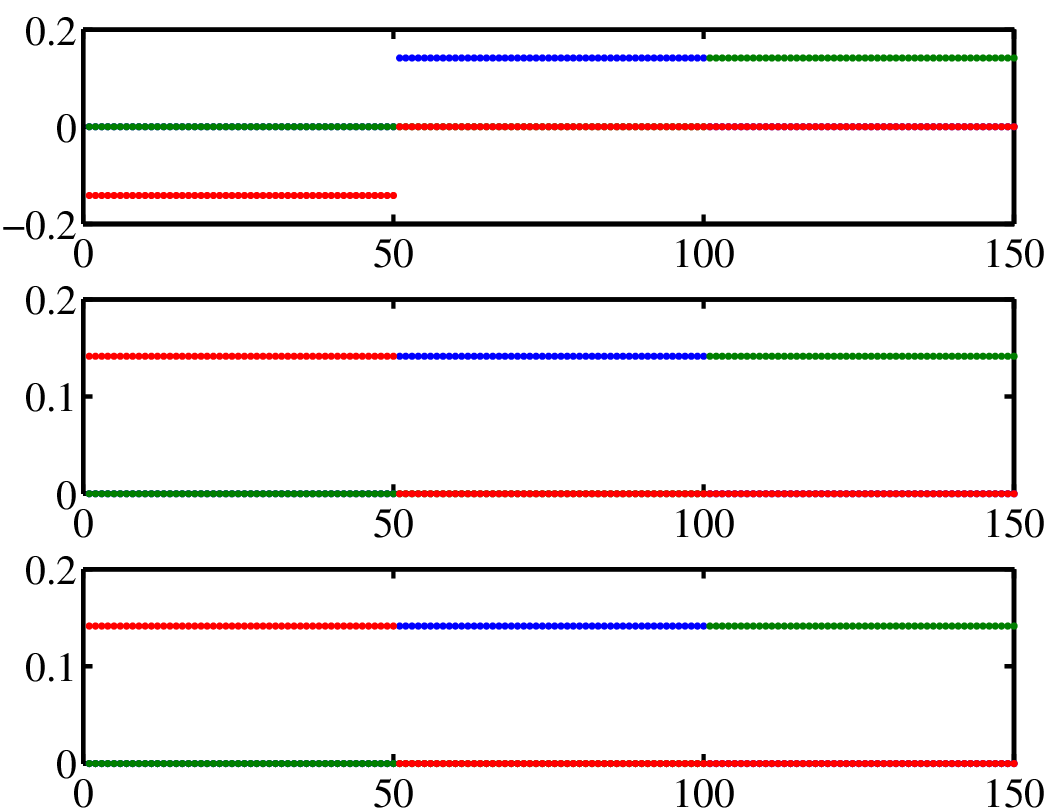}}
\subfigure[SSRk,
G2]{\label{fig:gausseigenvectorLap:g2}\includegraphics[width=5.0cm]{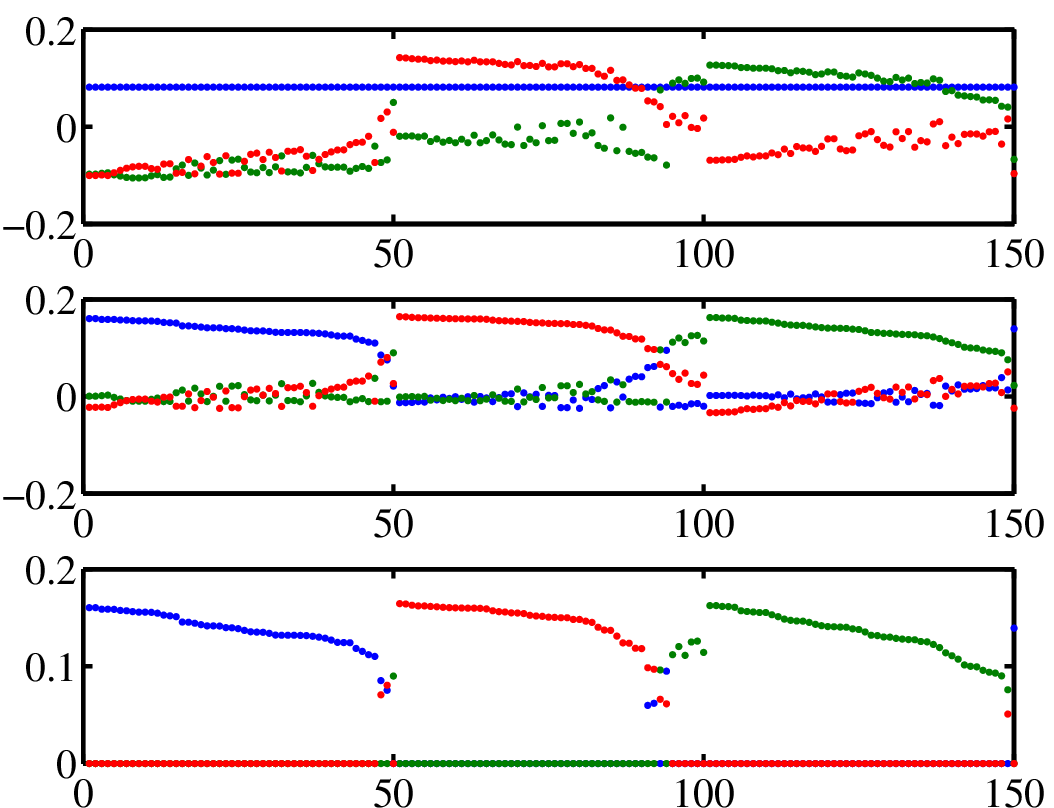}}
\subfigure[SSRk,
G3]{\label{fig:gausseigenvectorLap:g3}\includegraphics[width=5.0cm]{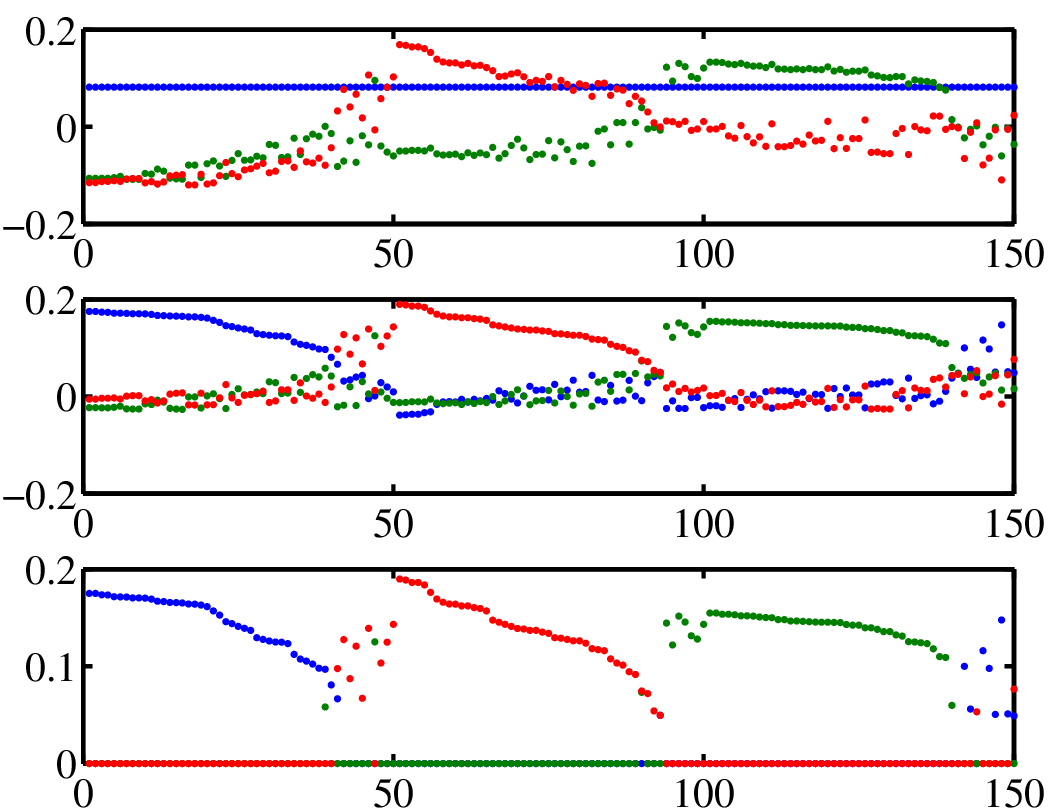}}\\
\subfigure[SSRl,
G1]{\label{fig:gausseigenvectorOri:g1}\includegraphics[width=5.0cm]{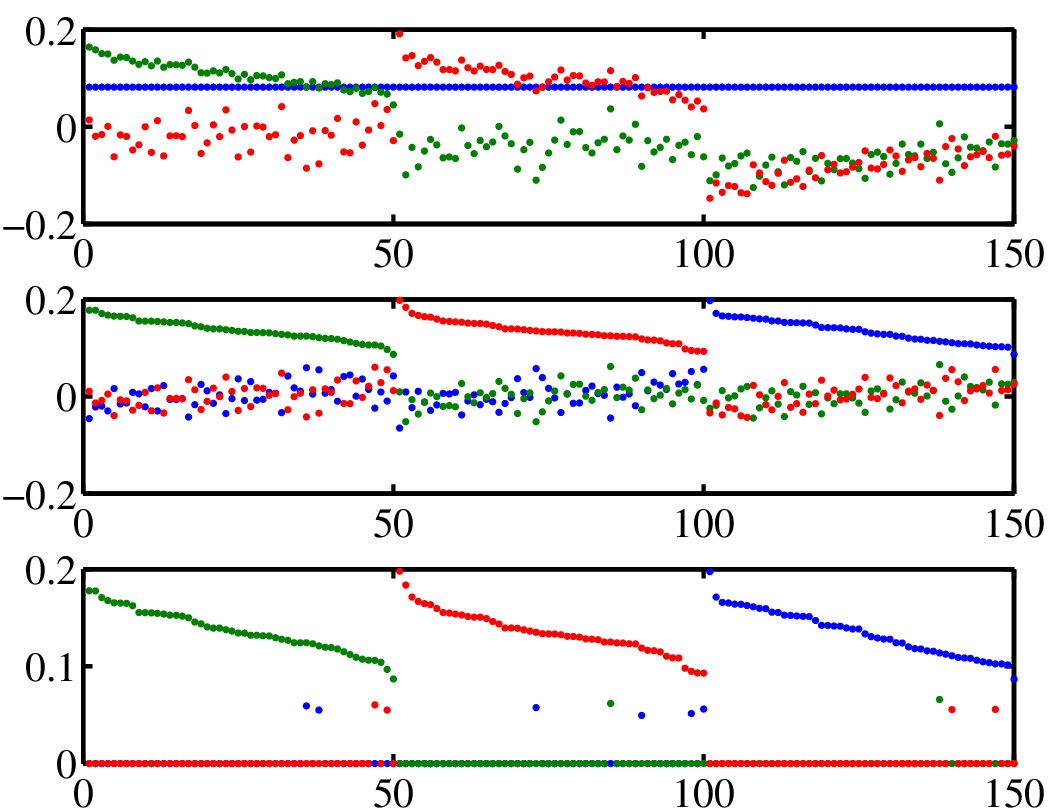}}
\subfigure[SSRl,
G2]{\label{fig:gausseigenvectorOri:g2}\includegraphics[width=5.0cm]{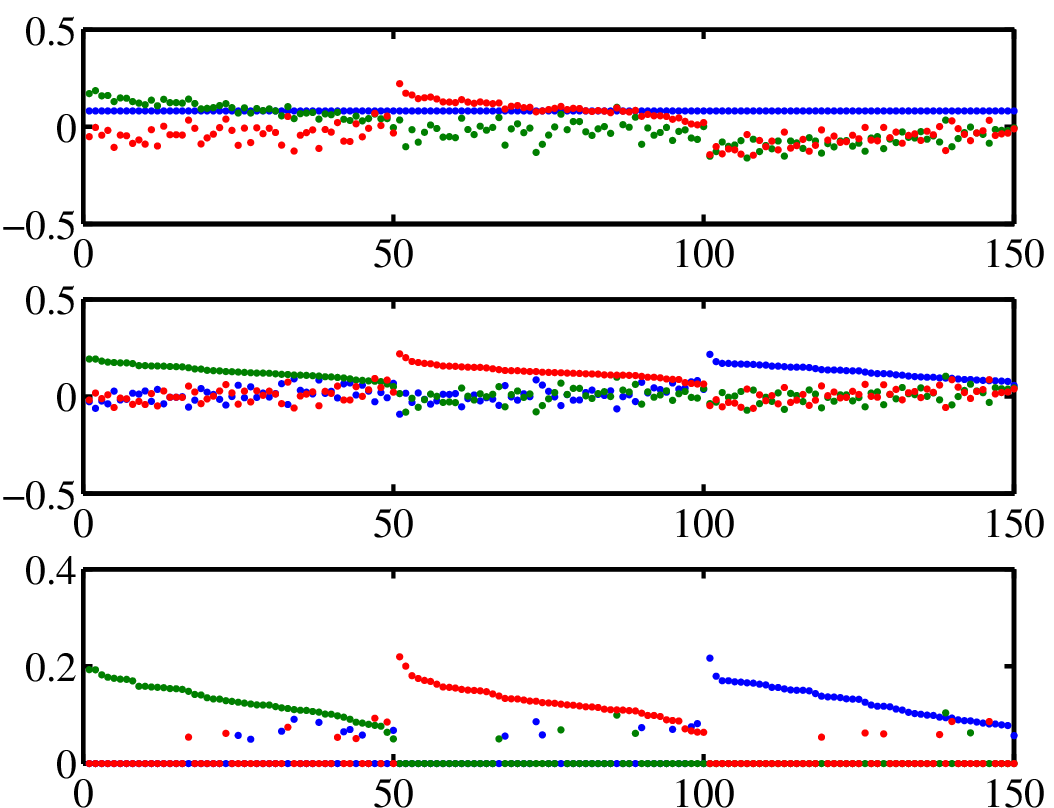}}
\subfigure[SSRl,
G3]{\label{fig:gausseigenvectorOri:g3}\includegraphics[width=5.0cm]{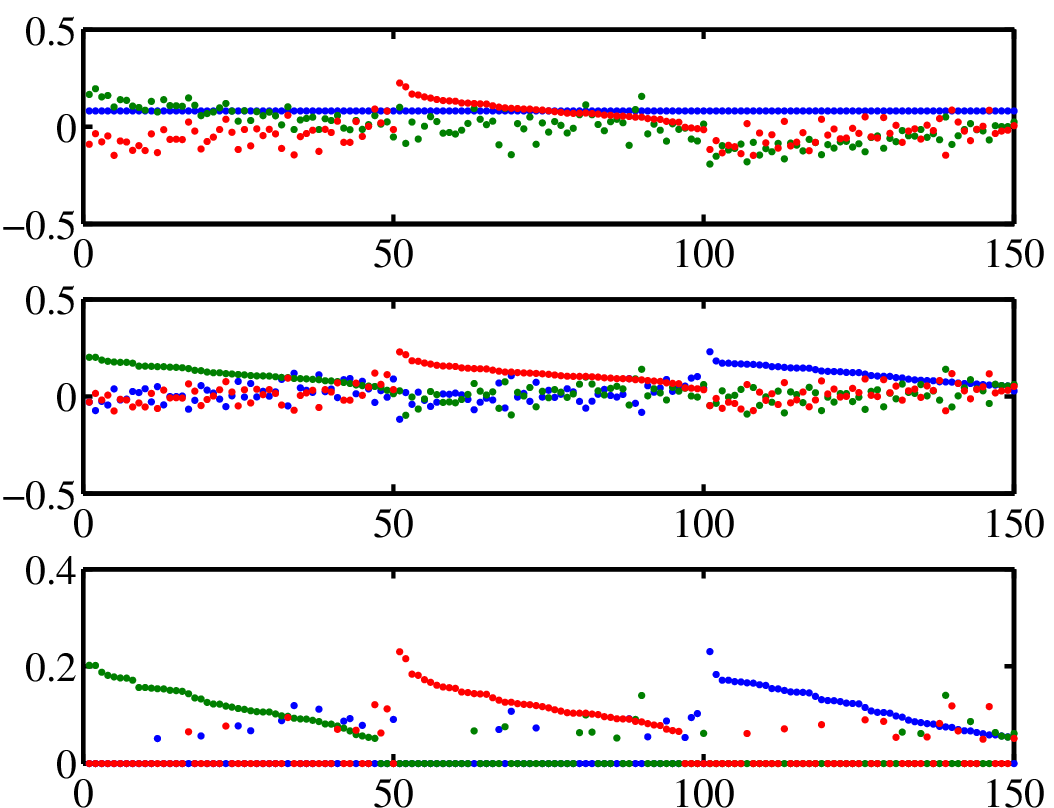}}
} \caption{Row vectors of sparse codes of G1-G3. In each
subfigure, top: eigenvectors $V\in \mathbb{R}^{3\times 150}$ (one
color corresponds to one vector), middle: sparse codes $H=R^TV$ obtained by NSCrt,
bottom: truncated sparse codes $\bar{H}$. The three segments 1-50,
51-100, 101-150 on horizontal axis correspond to the three
ground-truth clusters. Within each segment, samples are
rearranged according to the confidence to its truth cluster. }\label{fig:gausseigenvector}
\end{figure*}

\begin{figure}
  \begin{minipage}[t]{0.33\linewidth}
    \centering
    \includegraphics[width=4.7cm]{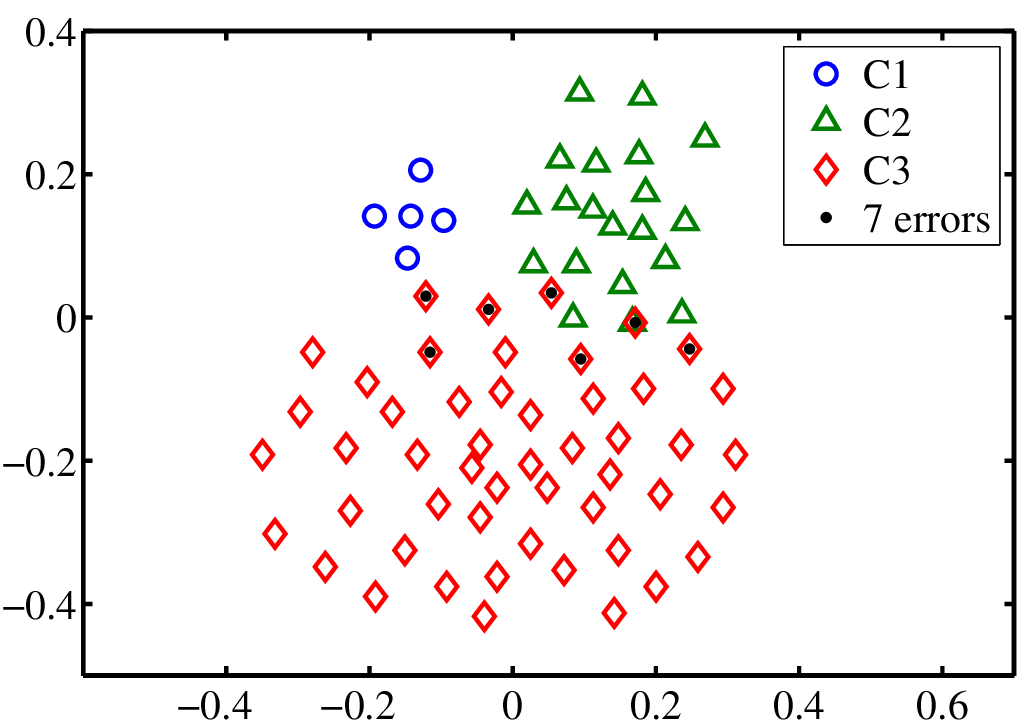}
    \caption{onion data: artificial data with unbalanced classes.}
    \label{fig:onionClassPoints}
  \end{minipage}%
  \hspace{2mm}
  \begin{minipage}[t]{0.67\linewidth}
    \centering
    \subfigure[SSRk]{\label{fig:onioneigenvector:lap}\includegraphics[width=5.0cm]{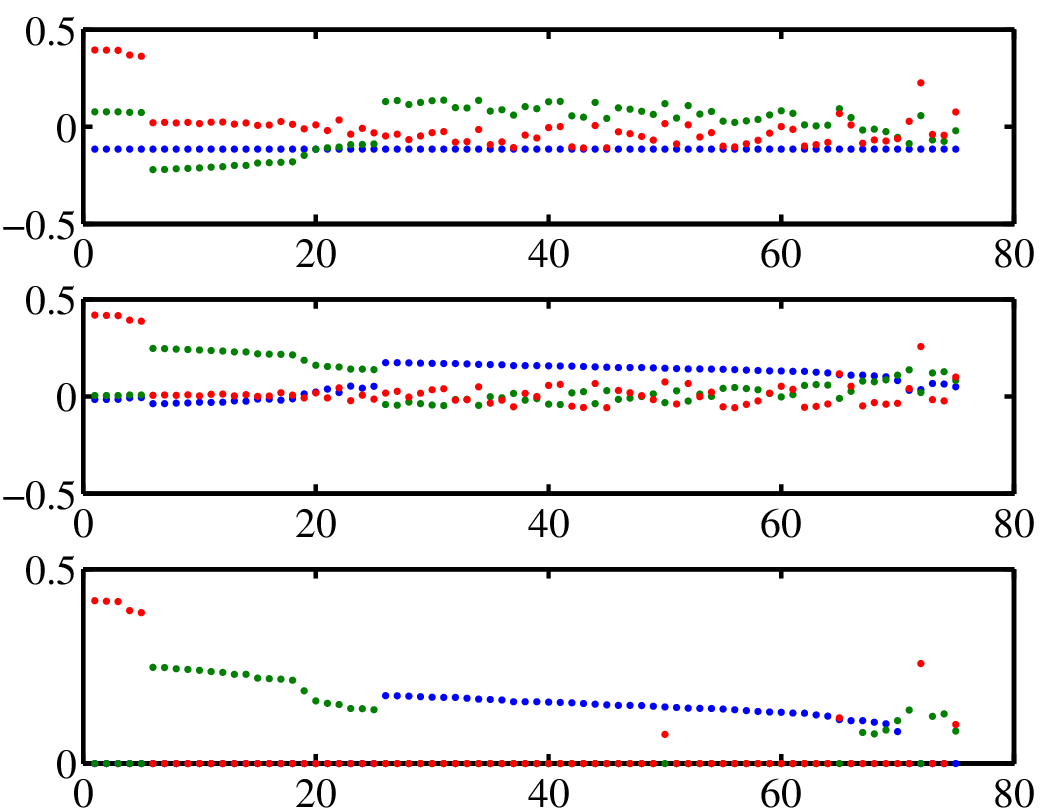}}
    \subfigure[SSRl]{\label{fig:onioneigenvector:ori}\includegraphics[width=5.0cm]{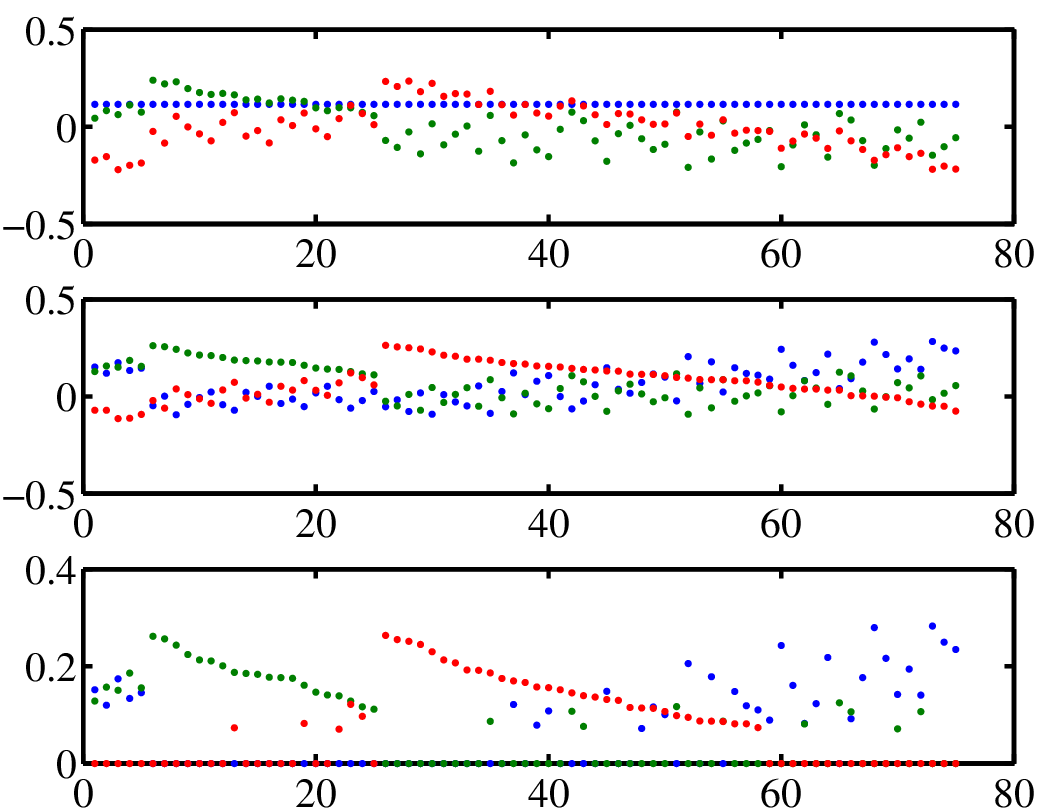}}
    \caption{Row vectors of sparse codes of onion data. Interpretation is similar to that in Figure~\ref{fig:gausseigenvector}.}
    \label{fig:onioneigenvector}
  \end{minipage}
\end{figure}

\begin{figure*}[h]
\centering{ \subfigure[SSRk]{\label{fig:eigenvector_points:ker}\includegraphics[width=10cm]{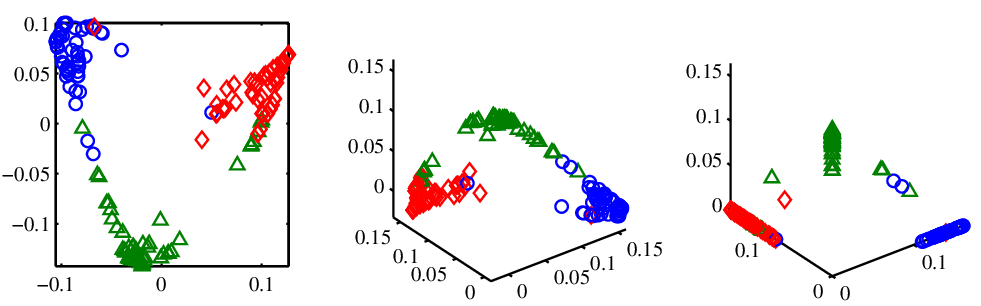}}\\
\subfigure[SSRl]{\label{fig:eigenvector_points:ori}\includegraphics[width=10cm]{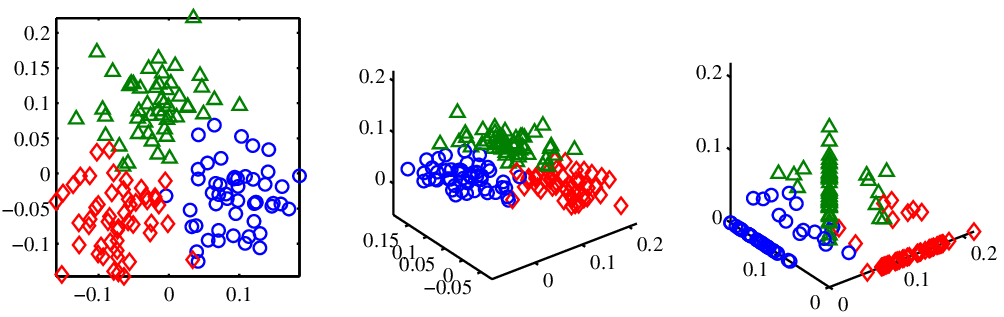}}}
\caption{Column vectors of sparse codes $H\in \mathbb{R}^{3\times
150}$ of G2 data. One point corresponds to one sample. Left:
intrinsic 2D manifold, i.e., $V_{2:3}$, middle: $H$, right:
truncated sparse codes $\bar{H}$.} \label{fig:eigenvector_points}
\end{figure*}

\begin{figure*}[h]
\centering{
\subfigure[SSRk]{\label{fig:gaussGram:ker}\includegraphics[width=7cm]{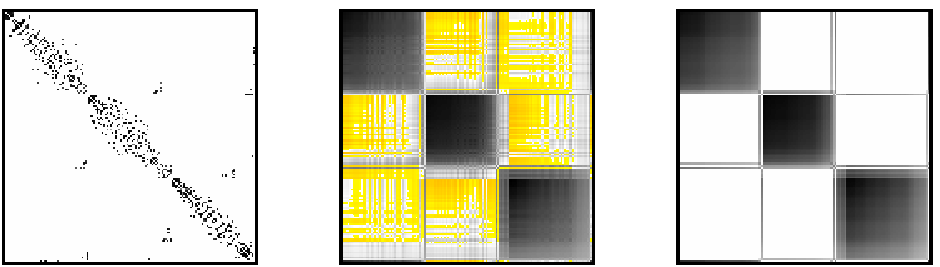}}\hspace{10mm}
\subfigure[SSRl]{\label{fig:gaussGram:ori}\includegraphics[width=7cm]{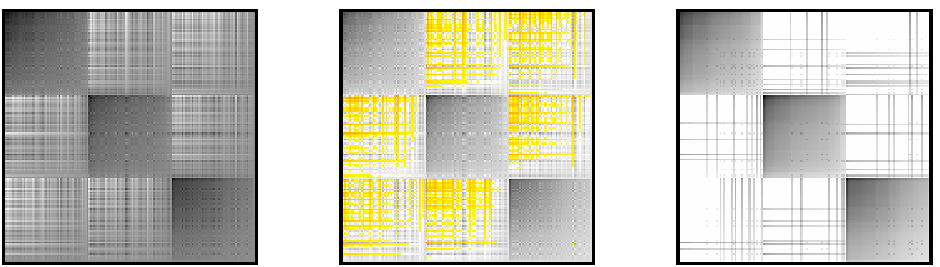}}}
\caption{Gram matrices of G2 data. Yellow color indicates negative
value (the heavier the color, the smaller the value). In each
subfigure, left: similarity matrix/Gram matrix of augmented data,
i.e., $W$, middle: $H^TH$, right: $\bar{H}^T\bar{H}$. The data is
rearranged as Figure~\ref{fig:gausseigenvector}.}
\label{fig:gaussGram}
\end{figure*}

\subsection{Illustrations of Cluster Structure Revealed by Sparse
Codes of SSR}\label{sec:cluster info}

It has been qualitatively analyzed how the sparse codes reveal
cluster structure in Section~\ref{sec:SSR}. We now illustrate it
from row perspective, column perspective, and Gram matrix
perspective. The Gaussian data G1, G2, G3, and onion data are taken
as examples. Both SSRk and SSRl are shown.

The row vectors of sparse codes $H$ are demonstrated in
Figure~\ref{fig:gausseigenvector} and
Figure~\ref{fig:onioneigenvector}. Usually, on each column of $H$,
only one value dominates, which indicates the
cluster membership. The cross sections of these vectors at the end
of each segment indicate the overlapping regions. In
Figure~\ref{fig:gausseigenvector}, the cross sections become more
and more significant from (a) to (c), faithfully reflecting the
overlapping status. Samples in these cross sections will
be misclassified by Scut, which are plotted as black dots in
Figure~\ref{fig:gaussClassPoints} and
Figure~\ref{fig:onionClassPoints}.

The column vectors of sparse codes, taking G2 as example, are
illustrated in Figure~\ref{fig:eigenvector_points}. For SSRk, compared
with the original data in Figure~\ref{fig:gaussClassPoints:g2}, the
codes make the cluster structure prominent. Besides, in contrast to
the 0-1 discrete codes of K-means, the sparse codes are continuous,
which can describe the clusters (if one entry dominates) and the overlaps
(if several comparable entries exist). The truncated sparse codes
try to resolve the overlaps, being closer to the K-means's codes.

The Gram matrix $H^TH$ reveals the linear relation of data: each
sample can be approximated by the linear combination of all samples,
with weights provided by the corresponding column of $H^TH$. Taking
G2 as example, the matrices are shown in Figure~\ref{fig:gaussGram}.
We see the weights in each column of $H^TH$ distribute according to
the relevance of the corresponding sample to all samples.
$\bar{H}^T\bar{H}$ removes the small weights, especially the
negative ones, making the main relation clearer.
For SSRk, the averages of column sums in $\bar{H}^T\bar{H}$ on G1-G3 and onion are 1, 0.9866, 0.9249, and 0.9784 respectively, still close to 1, implying the distortions due to truncations are small.

%Finally, the mutual coherence $\mu(\hat{D})$ of the dictionaries for
%G1-G3 and onion are: $5.5\times 10^{-17}$, 0.0024, 0.0026 and 0.0043
%which imply good incoherent property.

\begin{figure*}[h]
\centering{
\subfigure[G2]{\label{fig:W:G2}\includegraphics[width=4.6cm]{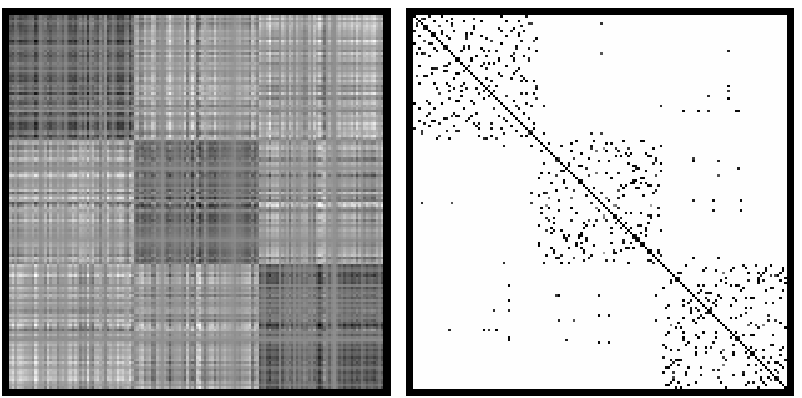}}\hspace{2mm}
\subfigure[USPS3]{\label{fig:W:uspsK3}\includegraphics[width=4.6cm]{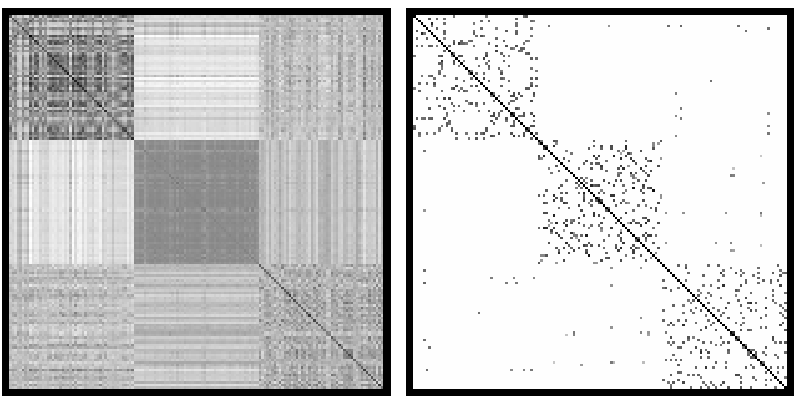}}\hspace{2mm}
\subfigure[iris]{\label{fig:W:iris}\includegraphics[width=4.6cm]{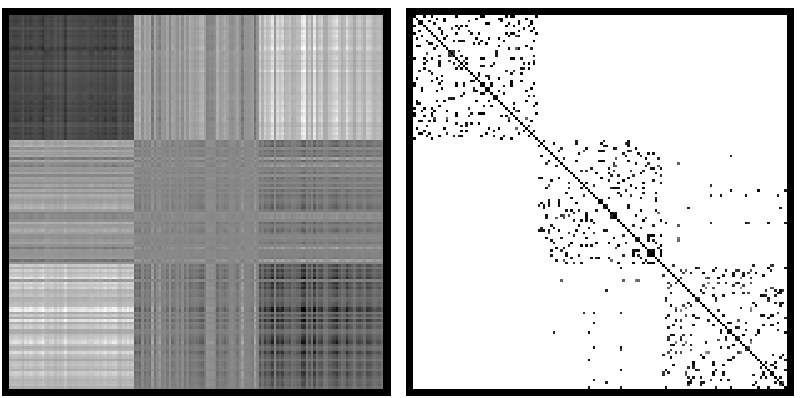}}\\
\subfigure[onion]{\label{fig:W:onion}\includegraphics[width=4.6cm]{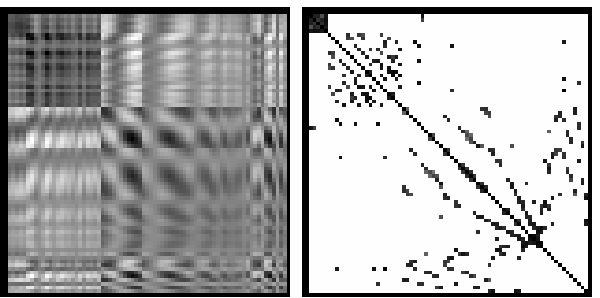}}\hspace{2mm}
\subfigure[wdbc]{\label{fig:W:wdbc}\includegraphics[width=4.6cm]{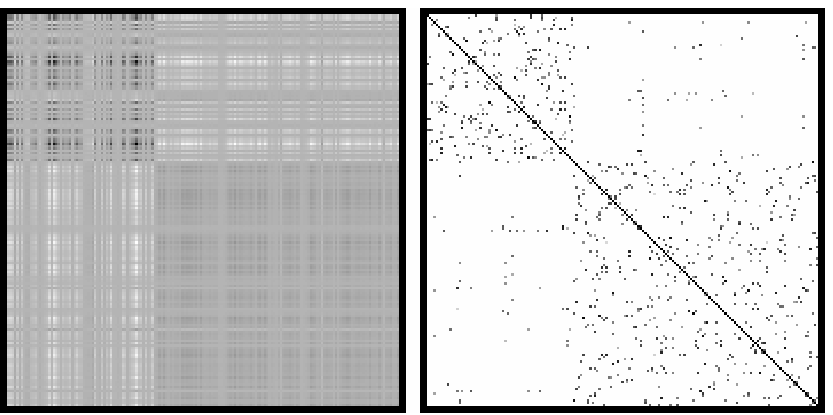}}\hspace{2mm}
\subfigure[USPS8]{\label{fig:W:uspsK8}\includegraphics[width=4.6cm]{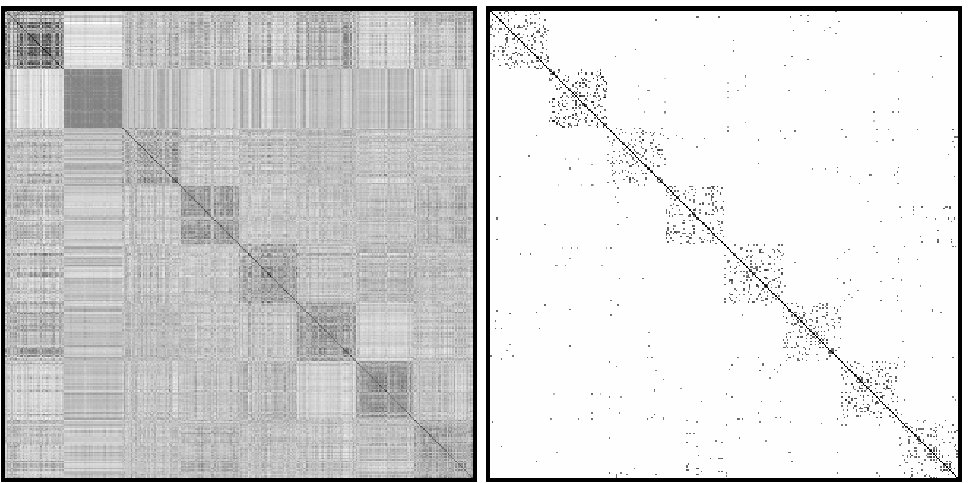}}
} \caption{Linear version v.s. kernel version: a comparison
of the similarity matrices. In each subfigure, left: linear
version, right: kernel version.}\label{fig:W}
\end{figure*}

\begin{figure*}[h]
\centering{
\subfigure[USPS3]{\label{fig:2d:USPS3}\includegraphics[width=4cm]{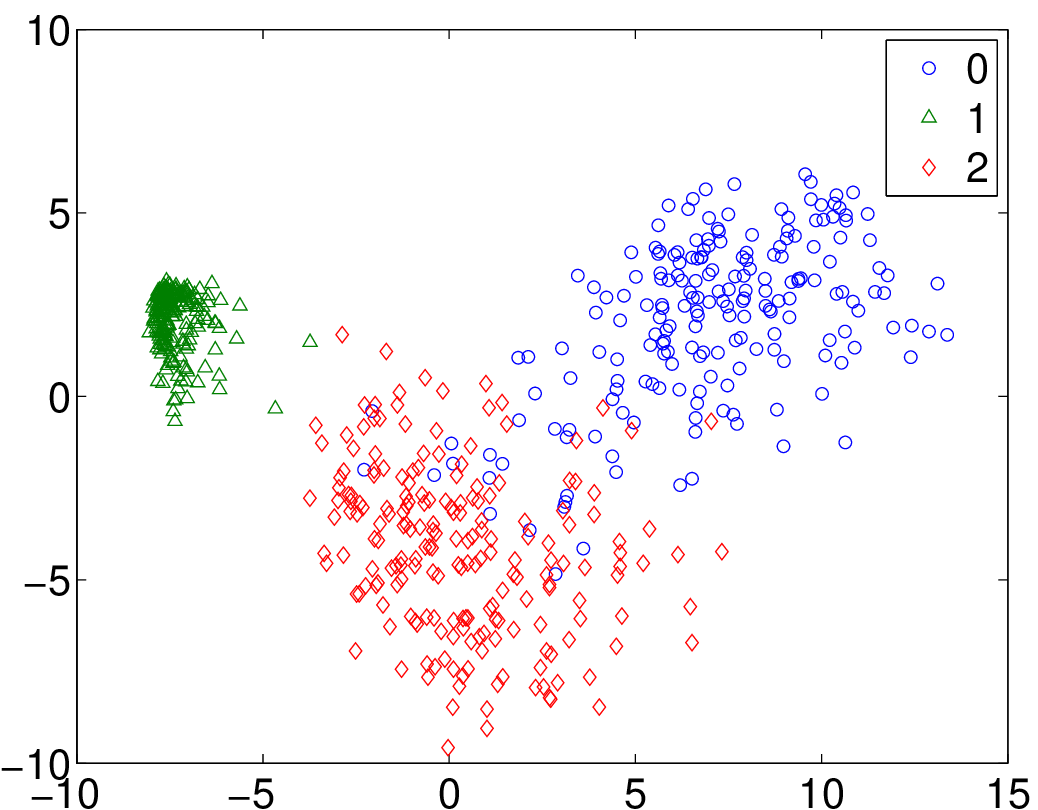}}\hspace{8mm}
\subfigure[iris]{\label{fig:2d:iris}\includegraphics[width=4cm]{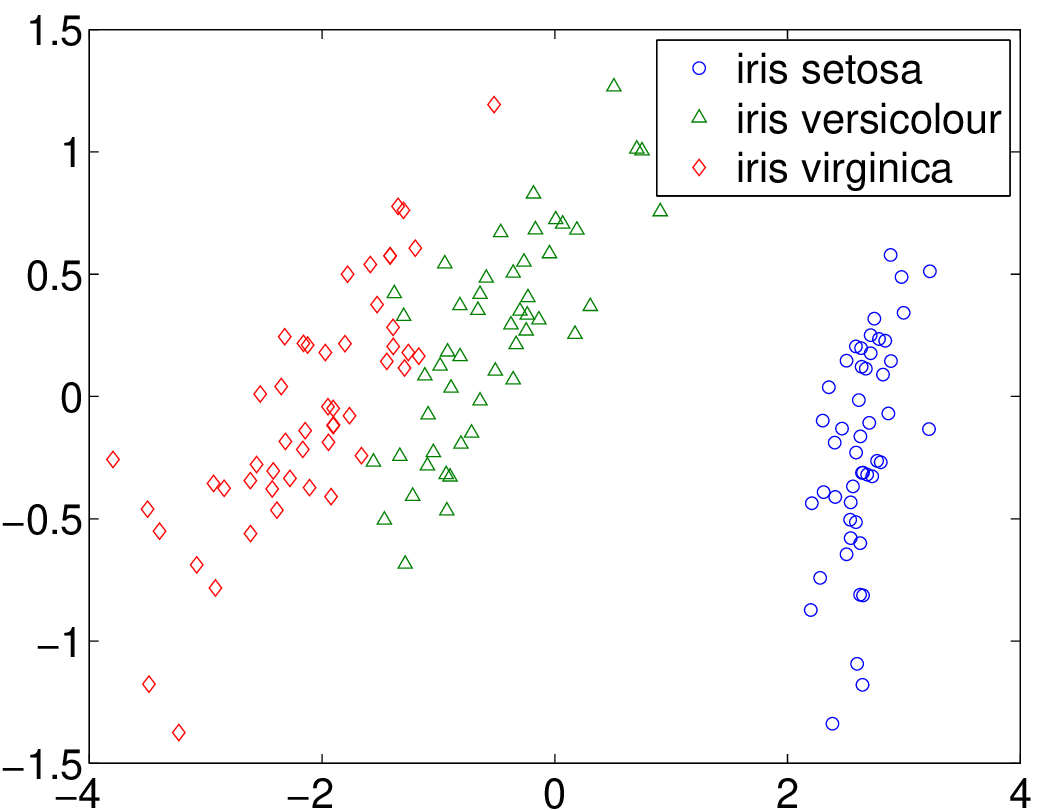}}
} \caption{2D views of USPS3 and iris (the first two PCs of
PCA).}\label{fig:usps3 iris}
\end{figure*}

\begin{table*}[h]
\caption{Linear version v.s. kernel version: $\rho$ values
($0\leq \rho \leq 1$) of the similarity matrices. A higher value
indicates the ideal graph condition is met better.}\label{tab:rho}
\vskip 0.1in
\begin{center}
%\vskip -0.1in
\begin{scriptsize} %
\begin{tabular}{|c||*{12}{c|}}
\hline $\rho$ (\%) & G1 & G2 & G3 & onion & iris & wdbc & Isolet &
USPS3 & USPS8 & USPS10 & 4News & TDT2\\\hline\hline

SSRl & 20.1 & 15.8 & 12.7 & 15.2 & 1.3 & 13.2 & 0.1 & 2.2 & 1.0 &
0.1 & 4.0 & 0.04\\\hline

SSRk & 100.0 & 59.4 & 58.4 & 39.4 & 63.2 & 67.7 & 16.8 & 60.5 & 5.1
& 6.8 & 1.9 & 4.7\\\hline
\end{tabular}
\end{scriptsize}
\end{center}
\vskip -0.1in
\end{table*}

\begin{table*}[h]
\caption{Linear version v.s. kernel version: clustering performance. K-PC: K-means applied on PCs
$\Sigma_{1:r-1}V_{1:r-1}$. Rcutl: K-means applied on $V_{1:r-1}$.
Scutl: Scut of SSRl. Rcut and Scut refer to kernel versions by
default. For the algorithms that depend on random initialization,
mean/standard deviation of the accuracy over 20 trials are reported.
The last row shows the mean over all data sets.
%The mean of time cost
%on two large data sets USPS10 and 4News are shown at the end. The
%time cost of K-PC, Rcutl, and Scutl is eigen-computation + K-means
%or NSCrt; that of Rcut and Scut is similarity-matrix construction +
%eigen-computation + K-means or NSCrt.
}\label{tab:clustering ori vs
ker} \vskip 0.1in
\begin{center}
%\vskip -0.1in
\begin{scriptsize} %
\begin{tabular}{|c||c|c|c|c||c|c|}
\hline

\multirow{2}{*}{Accuracy (\%)}

& \multicolumn{4}{c||}{Linear version} &
\multicolumn{2}{c|}{Kernel version}\\\cline{2-7}

& Kmeans & K-PC & Rcutl & Scutl & Rcut & Scut\\\hline\hline

G1  & 91.0 / 18.5 & 93.5 / 15.9 & 97.6 / 10.6 & \tb{100.0} & 91.4 /
17.6 & \tb{100.0}\\\hline

G2  & 95.3 / 0.0 & 95.3 / 0.0 & \tb{96.0} / 0.0 & 95.3 & 90.8 / 8.3
& 92.7\\\hline

G3  & 82.7 / 0.0 & 82.7 / 0.0 & 82.7 / 0.0 & \tb{83.3} & 82.0 / 0.0
& 82.0\\\hline

onion  & 61.7 / 6.9 & 62.5 / 7.2 & 62.6 / 5.2 & 61.3 & 90.5 / 6.6 &
\tb{90.7}\\\hline

iris  & 81.5 / 13.9 & 80.8 / 13.9 & 77.3 / 0.0 & 78.0 & 87.4 / 9.1 &
\tb{95.3}\\\hline

wdbc  & 85.4 / 0.0 & 85.4 / 0.0 & 85.4 / 0.0 & 87.5 & \tb{88.9} /
0.0 & 88.4\\\hline

Isolet  & 71.7 / 4.9 & 70.8 / 5.6 & 71.9 / 4.6 & 65.0 & 65.0 / 7.5 &
\tb{82.0}\\\hline

USPS3  & 92.0 / 0.0 & 78.3 / 0.0 & 71.7 / 0.0 & 72.6 & 97.2 / 8.6 &
\tb{99.2}\\\hline

USPS8  & 74.7 / 6.3 & 74.3 / 3.5 & 74.2 / 3.8 & 76.3 & 74.1 / 10.6 &
\tb{91.0}\\\hline

USPS10  & 65.6 / 2.8 & 61.8 / 4.0 & 65.5 / 1.5 & 64.9 & \tb{66.6} /
8.8 & 66.4\\\hline

4News  & 88.3 / 10.4 & 88.7 / 6.7 & 82.3 / 15.0 & 94.9 & 91.8 / 10.0
& \tb{96.0}\\\hline\hline

Average & 80.9 / 5.8 & 79.5 / 5.2 & 78.8 / 3.7 & 79.9 & 84.2 / 7.9 &
\tb{89.4}\\\hline
%\hline
%
%Time cost (s) & & & & & &\\\hline
%
%USPS10 & 8.68 & \tabincell{c}{1.60\\(1.38+0.22)} &
%\tabincell{c}{1.69\\(1.38+0.31)} &
%\tabincell{c}{\tb{1.51}\\(1.38+0.13)} &
%\tabincell{c}{7.94\\(7.41+0.43+0.10)} &
%\tabincell{c}{7.96\\(7.41+0.43+0.12)}\\\hline
%
%4News & 584.69 & \tabincell{c}{8.46\\(8.44+0.02)} &
%\tabincell{c}{8.45\\(8.44+0.01)} & \tabincell{c}{8.46\\(8.44+0.02)}
%& \tabincell{c}{\tb{1.42}\\(1.03+0.38+0.01)} &
%\tabincell{c}{1.46\\(1.03+0.38+0.05)}\\\hline
\end{tabular}
\end{scriptsize}
\end{center}
\vskip -0.1in
\end{table*}

\subsection{Linear Version v.s. Kernel Version}\label{sec:ori vs ker}
We first compare how well the ideal graph condition is met for the
two versions qualitatively from the view of the structure of
similarity matrix, and then propose a measure for quantitative
evaluation. Finally we compare the clustering performance of the two versions.

First, a clearer block-diagonal structure of the similarity matrix
indicates the condition is met better. The similarity matrices on
some representative data sets are shown in Figure~\ref{fig:W}. We
see that so long as the data has clear clusters, the similarity
matrix of the kernel version will meet the condition well, but that
of the linear version does not necessarily be so. In
Figure~\ref{fig:W} (a)-(c), the condition is deviated more and
more for the linear version, as can be expected from the data
distributions of these data sets (see
Figure~\ref{fig:gaussClassPoints:g2} and Figure~\ref{fig:usps3
iris}).

Next, based on the
eigengap (Davis-Kahan) theorem of matrix perturbation (see Theorem~7 of \cite{von2007tutorial}), we propose a $\rho$ value,
measuring how well the ideal graph condition is met. The $\rho$
value of a similarity matrix $W\in\mathbb{R}^{n\times n}$ for $1\leq
r<n$ clusters is defined as
\begin{equation}\label{equ:rho}
\rho=(\lambda_{r+1}-\lambda_{r})/\lambda_{r+1},
\end{equation}
where $\lambda_{i}$ is the $i$th smallest eigenvalue of the
Laplacian matrix. If $\lambda_{r+1}=0$, $\rho$ is defined to be 0.
It has the following properties:
\begin{theorem}
$0\leq \rho \leq 1$, and $\rho=1$ if and only if the ideal graph
condition is met exactly.
\end{theorem}

\begin{proof}
When the ideal graph condition is met, according to the properties
of Laplacian matrix in Section~\ref{sec:ratio cut}, $\lambda_{r}=0$
and $\lambda_{r+1}\neq 0$, so $\rho=1$. Conversely, if $\rho=1$,
then $\lambda_{r+1}\neq 0$ and $\lambda_{1}=\cdots=\lambda_{r}=0$,
implying $r$ connected components exist, so the ideal
graph condition is met.
\end{proof}
When the graph is nearly ideal, $\rho\approx 1$. The better the
condition is met, the higher the $\rho$ is. A comparison of the
$\rho$ values between the linear version and kernel version
are shown in Table~\ref{tab:rho}. It is clear that generally the
kernel version meets the condition better than the linear
version.

Finally, the clustering performance are compared in Table~\ref{tab:clustering ori vs ker}. The
results on TDT2 and polb are absent, since the linear version
fails to run on them. %Concerning accuracy,
The kernel version generally performs better, for it meets the ideal
graph condition better. The contrasts are most apparent on onion,
iris, and wdbc, as implied by Figure~\ref{fig:W}. But, the results
are reversed on G2 and G3, maybe the linear version meets the
condition well (see Figure~\ref{fig:W:G2}).
%Besides, we have the following
%additional observations. 1) The time cost of kernel version on
%USPS10, which is mainly contributed by the construction of
%similarity matrix, is comparable to that of Kmeans while larger than
%those of K-PC, Rcutl, and Scutl. However, when dealing with sparse
%data, kernel version may be more efficient, as 4News shows. This is
%because the eigen-computation of linear version applies on
%the mean-removed data, which is no longer sparse, and processing
%dense data is much costly. 2) There is no significant accuracy
%difference between the methods within linear version.
%However, K-PC, Rcutl, and Scutl, which work with eigenvectors, are
%more efficient than K-means.

\begin{figure}
  \begin{minipage}[t]{0.55\linewidth}
    \centering
    \includegraphics[width=6cm]{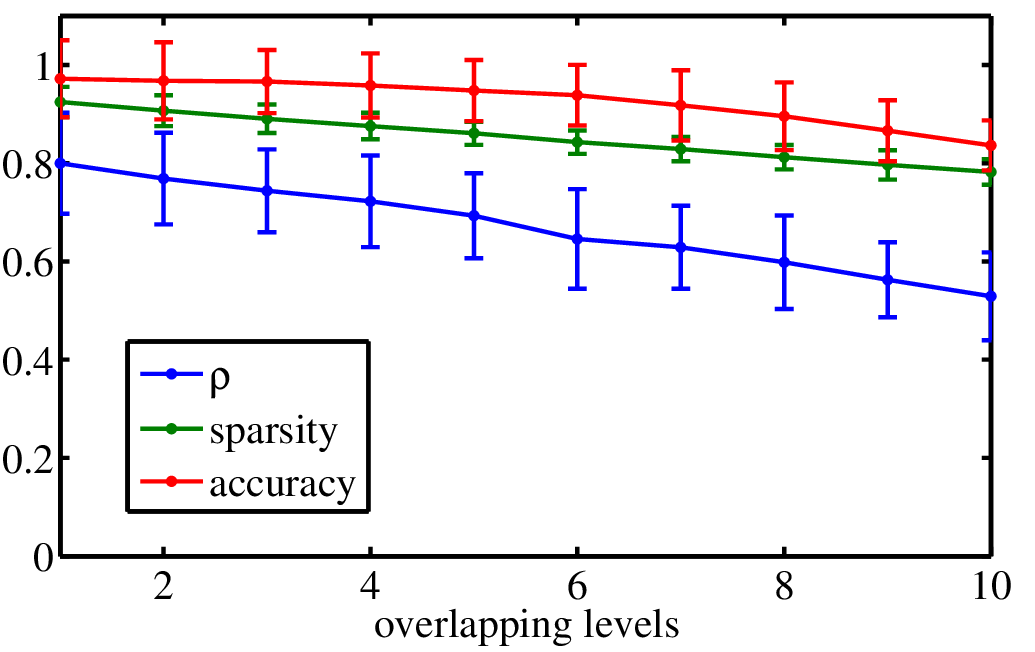}
    \caption{The relations between $\rho$, sparsity of SSRk, and
clustering accuracy of kernel Scut on a set of random Gaussian data with more and more heavy overlaps. The curve and
the bar show the mean and the standard derivation over 50 trials respectively.}\label{fig:rho gauss}
  \end{minipage}%
  \hspace{3mm}
  \begin{minipage}[t]{0.42\linewidth}
    \centering
    \includegraphics[width=5cm]{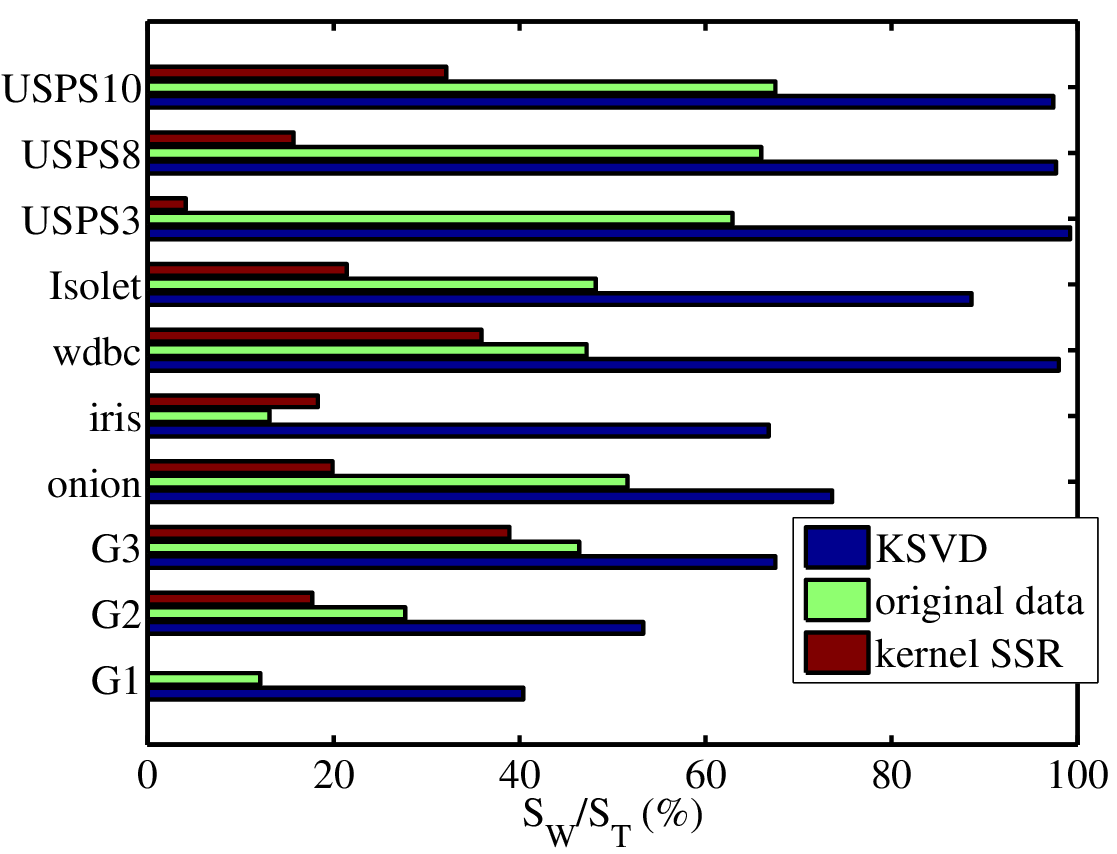}
    \caption{Comparison of the ratios of within-class distance between
the original data, sparse codes of KSVD, and sparse codes of SSRk. On G1, the ratio of $\text{SSRk}$ is
zero.}\label{fig:within class bar}
  \end{minipage}
\end{figure}

\subsection{Relations between $\rho$, Sparsity of SSR, and Clustering Accuracy}

In principle, $\rho$ reflects the separability of the clusters, it
thus relates to the clustering accuracy. On the other hand, it
implies the noise level in the indicator vectors, which is related
to the sparsity of codes. Therefore, there are proportional relations between them. We validate the relations by a set of
Gaussian data with more and more heavy overlaps (like G1-G3).
The result is shown in Figure~\ref{fig:rho gauss},
clear proportional relations can be observed. The relations imply that, potentially, the $\rho$ value, which can be computed
once the data is given, may provide us an estimation of the sparsity
and the clustering performance before SSR and Scut are applied or
when the ground-truth class labels are not available.

\begin{table*}[htbp]
\caption{Comparison of the clustering results between kernel Scut
and Graclus \cite{dhillon2007weighted}, Ncut
\cite{yu2003multiclass}, ZP \cite{zelnik2004self}, GMM-V (Gaussian
mixture model, applied on $V_{1:r}$ of Laplacian matrix), GMM
(applied on original data), Kmeans, NJW \cite{ng2002spectral}, and
Rcut. For algorithms that depend on random initialization (from
GMM-V to Rcut), the score when the objective function obtains the
best value over 20 trials is reported. The last row shows the
average score over all data sets.}\label{tab:scut}\vskip 0.1in
\begin{center}
\begin{scriptsize} %

\subtable[Accuracy]{
\begin{tabular}{|c|*{9}{|c}|}
\hline Accuracy (\%) & Graclus & Ncut & ZP & GMM-V & GMM & Kmeans &
NJW & Rcut & Scut
\\\hline\hline

G1  & \tb{100.0} & \tb{100.0} & \tb{100.0} & \tb{100.0} & \tb{100.0}
& \tb{100.0} & \tb{100.0} & \tb{100.0} & \tb{100.0}\\\hline

G2  & 92.7 & 93.3 & 92.7 & 93.3 & \tb{96.0} & 95.3 & 92.7 & 92.7 &
92.7\\\hline

G3  & \tb{84.0} & 81.3 & 81.3 & 81.3 & 54.7 & 82.7 & 81.3 & 82.0 &
82.0\\\hline

onion  & 62.7 & 90.7 & 90.7 & 73.3 & \tb{93.3} & 57.3 & 90.7 & 92.0
& 90.7\\\hline

iris  & 87.3 & 87.3 & 88.0 & 84.0 & 52.7 & 89.3 & 90.0 & 90.0 &
\tb{95.3}\\\hline

wdbc  & 88.4 & 77.7 & 88.2 & 83.5 & 85.1 & 85.4 & 81.9 & \tb{88.9} &
88.4\\\hline

Isolet1  & 76.8 & 80.9 & 70.7 & 73.5 & 55.6 & 81.4 & 80.5 &
\tb{82.2} & 82.0\\\hline

USPS3  & \tb{99.5} & 60.9 & 99.2 & 99.2 & 95.7 & 92.0 & 99.2 & 99.1
& 99.2\\\hline

USPS8  & 80.2 & 47.9 & 92.1 & 85.7 & 67.8 & 72.2 & 90.8 & \tb{93.3}
& 91.0\\\hline

USPS10  & \tb{78.9} & 67.3 & 64.6 & 75.8 & 49.4 & 67.8 & 66.5 & 66.8
& 66.4\\\hline

4News  & 95.6 & \tb{96.1} & 95.7 & 71.3 & - & 92.2 & \tb{96.1} &
95.8 & 96.0\\\hline

TDT2  & 54.4 & 88.0 & 84.4 & 71.6 & - & - & 70.7 & 76.6 &
\tb{88.5}\\\hline

polb  & 83.8 & 82.9 & 84.8 & 81.0 & - & - & 82.9 & \tb{87.6} &
84.8\\\hline\hline

Average & 83.4 & 81.1 & 87.1 & 82.6 & 75.0 & 83.2 & 86.4 & 88.2 &
\tb{89.0}\\\hline
\end{tabular}}\label{tab:scut:accuracy}

\subtable[NMI]{
\begin{tabular}{|c|*{9}{|c}|}
\hline NMI (\%) & Graclus & Ncut & ZP & GMM-V & GMM & Kmeans & NJW &
Rcut & Scut
\\\hline\hline

G1  & \tb{100.0} & \tb{100.0} & \tb{100.0} & \tb{100.0} & \tb{100.0}
& \tb{100.0} & \tb{100.0} & \tb{100.0} & \tb{100.0} \\\hline

G2  & 76.4 & 78.5 & 75.5 & 78.5 & \tb{83.9} & 80.9 & 75.5 & 75.5 &
75.5 \\\hline

G3  & \tb{56.9} & 49.4 & 49.4 & 49.6 & 26.2 & 48.6 & 49.4 & 50.8 &
50.4 \\\hline

onion  & 44.6 & 68.4 & 68.4 & 37.9 & \tb{74.4} & 46.4 & 68.4 & 71.4
& 68.4 \\\hline

iris  & 75.0 & 75.0 & 75.6 & 72.2 & 65.4 & 75.8 & 77.8 & 77.8 &
\tb{84.6} \\\hline

wdbc  & 49.4 & 35.2 & 49.0 & 42.5 & 37.1 & 46.7 & 39.2 & \tb{49.9} &
49.4 \\\hline

Isolet1  & 79.8 & 88.0 & 82.0 & 83.6 & 62.7 & 84.8 & 87.8 & 87.6 &
\tb{88.2} \\\hline

USPS3  & \tb{97.2} & 58.5 & 95.8 & 95.7 & 84.6 & 80.0 & 95.8 & 95.4
& 95.8 \\\hline

USPS8  & 85.8 & 66.9 & 86.8 & 83.7 & 55.7 & 70.5 & 85.6 & \tb{87.2}
& 86.2 \\\hline

USPS10  & \tb{81.0} & 80.1 & 77.6 & 79.0 & 41.0 & 64.0 & 78.6 & 78.8
& 78.4 \\\hline

4News  & 84.3 & \tb{85.7} & 84.2 & 65.1 & - & 81.4 & \tb{85.7} &
84.8 & 85.4 \\\hline

TDT2  & 74.2 & \tb{85.5} & 83.1 & 76.1 & - & - & 80.8 & 83.5 & 85.0
\\\hline

polb  & 55.4 & 54.2 & 58.6 & 57.8 & - & - & 54.2 & \tb{65.1} & 58.6
\\\hline\hline

Average & 73.8 & 71.2 & 75.8 & 70.9 & 63.1 & 70.8 & 75.3 & \tb{77.5}
& 77.4\\\hline
\end{tabular}}\label{tab:scut:nmi}

\subtable[RI]{
\begin{tabular}{|c|*{9}{|c}|}
\hline RI (\%) & Graclus & Ncut & ZP & GMM-V & GMM & Kmeans & NJW &
Rcut & Scut
\\\hline\hline

G1  & \tb{100.0} & \tb{100.0} & \tb{100.0} & \tb{100.0} & \tb{100.0}
& \tb{100.0} & \tb{100.0} & \tb{100.0} & \tb{100.0}\\\hline

G2  & 90.8 & 91.6 & 90.8 & 91.6 & \tb{94.8} & 94.0 & 90.8 & 90.8 &
90.8\\\hline

G3  & \tb{81.7} & 79.0 & 79.0 & 78.6 & 56.7 & 79.7 & 79.0 & 79.2 &
79.6\\\hline

onion  & 66.7 & 85.3 & 85.3 & 63.8 & \tb{88.7} & 64.1 & 85.3 & 86.5
& 85.3\\\hline

iris  & 86.2 & 86.2 & 86.8 & 83.7 & 72.2 & 88.0 & 88.6 & 88.6 &
\tb{94.2}\\\hline

wdbc  & 79.5 & 65.3 & 79.2 & 72.4 & 74.5 & 75.0 & 70.3 & \tb{80.3} &
79.5\\\hline

Isolet1  & 96.7 & \tb{97.5} & 95.6 & 96.2 & 93.1 & 97.3 & 97.4 &
97.2 & \tb{97.5}\\\hline

USPS3  & \tb{99.4} & 73.9 & 99.0 & 98.9 & 95.0 & 90.8 & 99.0 & 98.9
& 99.0\\\hline

USPS8  & 94.6 & 83.1 & 96.4 & 95.7 & 88.1 & 91.5 & 96.2 & \tb{96.8}
& 96.0\\\hline

USPS10  & \tb{94.6} & 93.5 & 92.9 & \tb{94.6} & 78.8 & 91.7 & 93.3 &
93.2 & 93.2\\\hline

4News  & 95.7 & \tb{96.2} & 95.8 & 83.6 & - & 92.5 & \tb{96.2} &
95.9 & 96.1\\\hline

TDT2  & 91.3 & \tb{96.3} & 95.3 & 93.7 & - & - & 93.8 & 95.1 &
96.2\\\hline

polb  & 83.6 & 83.1 & 85.0 & 81.4 & - & - & 83.1 & \tb{86.6} &
85.0\\\hline\hline

Average & 82.4 & 87.0 & 90.9 & 87.2 & 84.2 & 87.7 & 90.2 & 91.5 &
\tb{91.7}\\\hline
\end{tabular}}\label{tab:scut:ri}

\end{scriptsize}
\end{center}
\vskip -0.1in
\end{table*}

\begin{table*}[htbp]
\caption{Time cost on the three largest data sets, USPS10
(256$\times$7291, 10 classes), 4News (26214$\times$2372, 4 classes),
and TDT2 (36771$\times$9394, 30 classes). The
total time cost includes the construction of similarity matrix,
computation of eigenvectors, and post-processing of eigenvectors.
}\label{tab:time} \vskip 0.1in
\begin{center}
%\vskip -0.1in
\begin{scriptsize} %

\subtable[Total time cost]{
\begin{tabular}{|c|*{9}{|c}|}
\hline Time (s) & Graclus & Ncut & ZP & GMM-V & GMM & Kmeans & NJW &
Rcut & Scut
\\\hline\hline
USPS10  & \tb{7.5} & 8.2 & 12.0 & 11.0 & 101.0 & 8.7 & 7.9 & 7.9 &
8.0\\\hline

4News  & \tb{1.1} & 1.1 & 1.5 & 1.5 & - & 584.7 & 1.1 & 1.4 &
1.5\\\hline

TDT2  & \tb{17.7} & 19.7 & 187.8 & 47.1 & - & - & 22.8 & 22.7 &
19.3\\\hline
\end{tabular}}\label{tab:time:total}

%\subtable[Time cost excluding the construction of similarity
%matrix]{
%\begin{tabular}{|c|*{9}{|c}|}
%\hline Time (s) & Graclus & Ncut & ZP & GMM-V & GMM & Kmeans & NJW &
%Rcut & Scut
%\\\hline\hline
%USPS10  & \tb{0.1} & 0.8 & 4.6 & 3.6 & 101.0 & 8.7 & 0.4 & 0.5 &
%0.5\\\hline
%
%4News  & \tb{$<$0.1} & 0.1 & 0.4 & 0.4 & - & 584.7 & 0.1 & 0.4 &
%0.4\\\hline
%
%TDT2  & \tb{0.3} & 2.3 & 170.4 & 29.7 & - & - & 5.3 & 5.3 &
%1.9\\\hline
%\end{tabular}}\label{tab:time:simi}

\subtable[Time cost of post-processing eigenvectors]{
\begin{tabular}{|c|*{6}{|c}|}
\hline Time (s) & ZP & GMM-V & NJW & Rcut & Scut
\\\hline\hline
USPS10  & 4.20 & 3.13 & \tb{0.08} & 0.10 & 0.12\\\hline

4News  & 0.05 & 0.06 & \tb{0.01} & 0.01 & 0.05\\\hline

TDT2  & 168.80 & 28.10 & 2.66 & 3.69 & \tb{0.34}\\\hline
\end{tabular}}\label{tab:time:simi eig}
\end{scriptsize}
\end{center}
\vskip -0.1in
\end{table*}

\subsection{Clustering via Kernel Scut}\label{sec:scut performance}
Since the superiority of kernel version has been demonstrated in previous section, now we focus
on kernel Scut and compare it with some popular algorithms.

The results of accuracy, NMI, and RI are shown in
Table~\ref{tab:scut}.\footnote{Codes of Graclus and Ncut are
downloaded from
\url{http://www.cs.utexas.edu/users/dml/Software/graclus.html} and
\url{http://www.cis.upenn.edu/~jshi/software/} respectively. GMM and
Kmeans are provided by MATLAB toolbox. The remaining methods are
implemented by us using MATLAB. GMM cannot run on 4News and TDT2 due
to $p>n$, and Kmeans fails to run on TDT2.} According to the mean
scores of the three criteria, Scut performs best overall. Although
the scores of Rcut are comparable to those of Scut, they are the
best ones picked over 20 trials, the mean scores of Rcut are much
worse, and the variances are large, as shown in
Table~\ref{tab:clustering ori vs ker}. The difference between Rcut
and Scut is that Rcut employs Kmeans to recover the indicators,
while Scut employs NSCrt. In view of the comparable results of Scut
to the optimal results of Rcut, it probably indicates
that NSCrt has accurately recovered the underlying indicators. ZP and
Ncut address similar problems too, especially ZP, which tries to
recover indicators via rotation as NSCrt does. However, the results
show that they do not perform as well as Scut. Graclus occasionally
does best on some data sets, however, it performs poorly on some
well-separable data sets, e.g., onion, iris, and TDT2.

The time cost on the three largest data sets, USPS10, 4News, and
TDT2, is shown in Table~\ref{tab:time}. GMM and Kmeans are much
slower. Scut is comparable to the most efficient method, Graclus.
Designed to improve the computational cost of traditional
spectral clustering methods, Graclus avoids the computation of
eigenvectors. However, the total time cost of those methods working
with similarity matrix, including Scut and Graclus, is dominated by
the construction of similarity matrix. As a consequence, the cost of
Scut is still close to that of Graclus. Table~\ref{tab:time}(b) highlights
the cost of the post-processing of eigenvectors. On the largest data
set TDT2, Scut is significantly faster than the others,
since the time complexity of Scut is $O(nr^2)$ scaling linearly with data
size.

Finally, as popular application of spectral clustering, image segmentation is compared between Rcut, Graclus, Ncut,
and Scut. Results on two typical images are shown in Figure~\ref{fig:imgseg}. The images and their similarity matrices
are obtained from the web site of Ncut: \url{http://www.cis.upenn.edu/~jshi/software/}.

\begin{table*}[h]
\caption{KSVD and SSC v.s. SSRl on clustering. KSVD has the same
size of dictionary and codes as SSRl, and non-maximum suppression is
applied for clustering. The mean accuracy of KSVD over 20 trials is
shown. }\label{tab:ksvd vs
Scutl} \vskip 0.1in
\begin{center}
%\vskip -0.1in
\begin{scriptsize} %
\begin{tabular}{|c||*{10}{c|}}
\hline Accuracy (\%) & G1 & G2 & G3 & onion & iris & wdbc & Isolet &
USPS3 & USPS8 & USPS10 \\\hline\hline

% mean err / std
%KSVD       & 8.7 / 0 & 28.2 / 6.3 & 48.7 / 4.9 & 47.9 / 3.0 & 51.3 /
%2.6 & 46.2 / 1.4 & 45.6 / 4.9 & 32.5 / 4.3 & 43.9 / 3.7 & 51.4 / 3.7
%& 25.7 / 10.9\\\hline
KSVD       & 95.9& 86.0& 67.5& 64.7& 62.6& 66.5& 56.3& 67.9& 57.2& 49.2\\\hline

SSC        & \tb{100.0} & 66.7 & 53.3 & 60.0 & \tb{78.0} & 87.3 &
\tb{78.8} & 59.0 & - & - \\\hline

Scutl      & \tb{100.0} & \tb{95.3} & \tb{83.3} & \tb{61.3} &
\tb{78.0} & \tb{87.5} & 65.0 & \tb{72.6} & \tb{76.3} & \tb{64.9} \\\hline
\end{tabular}
\end{scriptsize}
\end{center}
\vskip -0.1in
\end{table*}

\begin{table*}[h]
\caption{KSVD+Kmeans v.s. kernel Scut. For KSVD, the optimal
accuracy with respect to objective function of Kmeans over 20 trials
is shown. }\label{tab:ksvd+kmeans vs scut} \vskip 0.1in
\begin{center}
%\vskip -0.1in
\begin{scriptsize} %
\begin{tabular}{|c||c|c|c|c|c|c|c|c|c|c|}
\hline Accuracy (\%) & G1 & G2 & G3 & onion & iris & wdbc & Isolet &
USPS3 & USPS8 & USPS10\\\hline\hline

KSVD       & 91.3& \tb{93.3}& 74.7& 58.7& 82.0& 67.0& 23.9& 43.1&
19.7& 17.5 \\\hline

Scut       & \tb{100.0}& 92.7& \tb{82.0}& \tb{90.7}& \tb{95.3}&
\tb{88.4}& \tb{82.0}& \tb{99.2}& \tb{91.0}& \tb{66.4}\\\hline
\end{tabular}
\end{scriptsize}
\end{center}
\vskip -0.1in
\end{table*}

\subsection{OSRs v.s. SSR on clustering}
We compare the two most related OSRs: KSVD and SSC \cite{elhamifar2013sparse}, with SSR on clustering. Below, absence
of results on 4News and TDT2 are due to running out of memory of KSVD and SSC.

First, for KSVD, we learn a dictionary and codes with the same sizes
as SSRl, and apply the same clustering strategy as Scut.\footnote{The maximal
$\ell_0$-norm of KSVD is set to be $\lceil r/3 \rceil$. Considering KSVD depends on random
initialization, it is repeated 20 times, and the mean accuracy is
reported, since there is no suitable criterion to pick the optimal result. Codes of SSC is the ADMM version
downloaded from \url{http://www.cis.jhu.edu/~ehsan/code.htm}. Independent
affine subspaces is assumed for better performance.}
The results are shown in Table~\ref{tab:ksvd vs Scutl}.
Scutl outperforms KSVD and SSC significantly. Next, we learn an over-complete KSVD and apply K-means
on the codes.\footnote{The number of atoms and the maximal $\ell_0$-norm are set to be
$m=\min(2p, \lceil n/5 \rceil)$ and $\min(8,\lceil m/3 \rceil)$ respectively.} Comparing with kernel Scut, the
clustering results are shown in Table~\ref{tab:ksvd+kmeans vs scut}. Scut still performs better.

To investigate the reason of the inferior performance of KSVD
against SSR, we compare how well the cluster structure of data is
enhanced by the two kinds of sparse codes. This is measured by the
ratio of within-class distance, $S_W/S_T=\sum_{k=1}^K \sum_{H_i\in
C_k}\|H_i-D_k\|_2^2/\sum_{j=1}^n \|H_j-d\|_2^2$, where $D_k$ is the
mean of class $C_k$, and $d$ is the mean of the whole set. The
ratio is between 0 and 1. Smaller value implies better separability. In the same setting as last experiment, the ratios are
shown in Figure~\ref{fig:within class bar}. We see that the ratios
of SSR are generally smaller than those of original data, which
means SSR enhances the cluster structure. However, those of KSVD are
significantly larger, implying the codes of KSVD are dispersed, which
may not suitable for clustering purpose.

%\begin{figure}[htbp]
%\centering{
%\includegraphics[width=6cm]{fig/winclassBar.eps}
%} \caption{Comparison of the ratios of within-class distance between
%the original data, sparse codes of KSVD, and sparse codes of SSRk.
%On G1, the ratio of SSRk is zero.}\label{fig:within class bar}
%\end{figure}

\begin{figure}[htbp]
\centering{ \subfigure[]{\label{fig:baby:ori}\includegraphics[width=2cm]{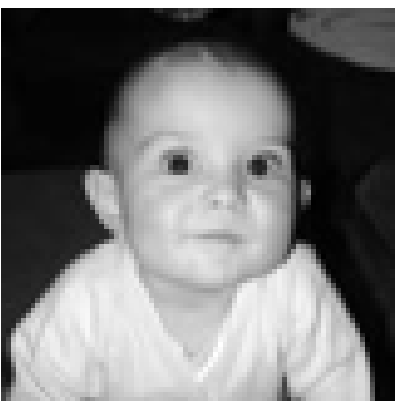}}
\subfigure[]{\label{fig:baby:seg}\includegraphics[width=5cm]{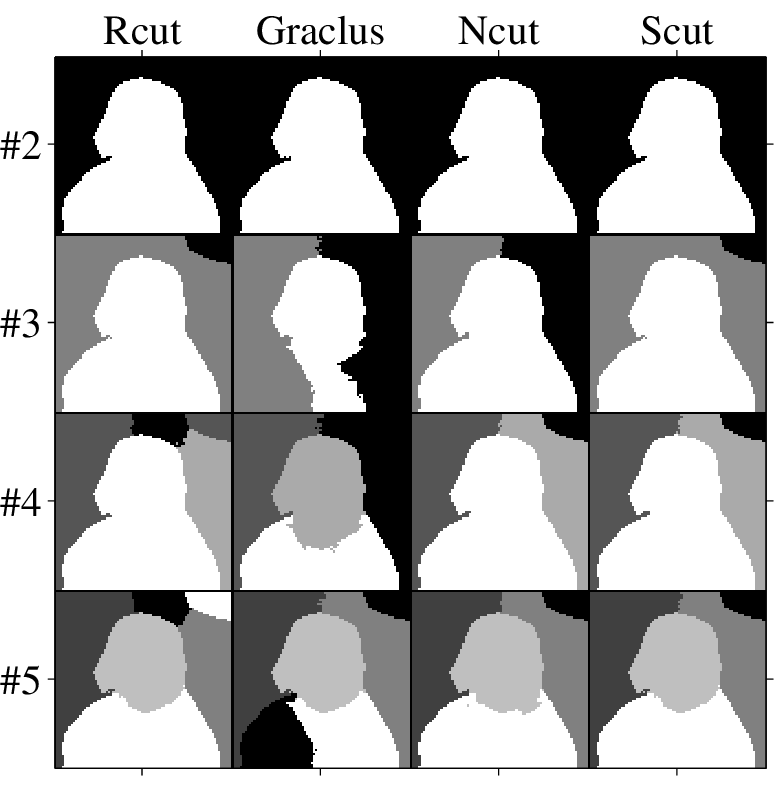}}
\subfigure[]{\label{fig:tiger:ori}\includegraphics[width=2cm]{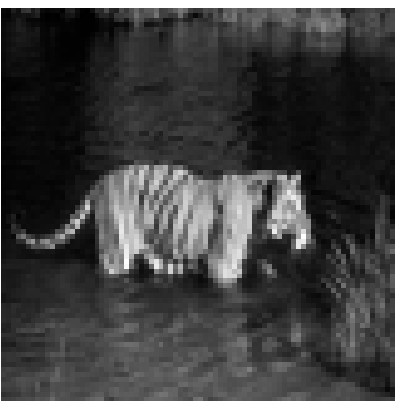}}
\subfigure[]{\label{fig:tiger:seg}\includegraphics[width=5cm]{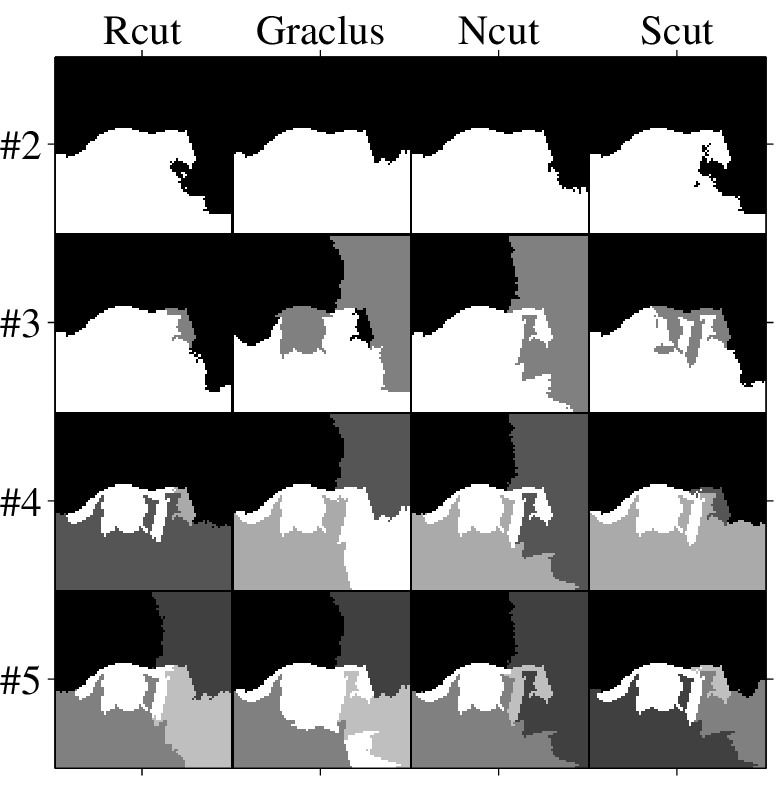}} } \caption{Examples of image segmentation
results. (a), (c) Original $80\times 80$ gray images. (b), (d) Segmentation results. The rows correspond to number of
clusters. }\label{fig:imgseg}
\end{figure}

\section{Extended Relations to Other Methods}\label{sec:discussions}
We have established the inherent relations between PCA/LE, K-means/Rcut, and SSR so far. In this section, we will further link the framework to OSRs, kernel PCA, manifold learning, and subspace clustering.

\subsection{Relation between OSRs and SSRl}\label{sec:OSR vs SSR}

We focus on two classical OSRs: \cite{olshausen1996emergence} and KSVD \cite{aharon2006img}, and compare them with
SSRl. The relation will be explored mainly from the row-space view, rather than the column-space view that is
sample-wise. It will be shown that the commonality of the methods is to find a sparse basis that covers the row space
of data as much as possible. The sparse basis consists of the row vectors of the code matrix. The basis size can be
chosen differently, whose range lies in $[0, n]$. OSRs and SSRl occupy different sections of the range, that is where
the difference comes from, and it leads to distinct solution strategies.

The objectives of \cite{olshausen1996emergence} and KSVD are
\begin{equation}\label{equ:spareL1}
\min_{D\in\mathbb{R}^{p\times r},X\in\mathbb{R}^{r\times
n}}\,\|A-DX\|^2_F+\lambda\|X_i\|_1,\,i=1,\dots,n,
\end{equation}
\begin{equation}\label{equ:ksvd}
\min_{D\in\mathbb{R}^{p\times r},X\in\mathbb{R}^{r\times
n}}\,\|A-DX\|^2_F,\,\st\,\|X_i\|_0\leq L,\,i=1,\dots,n,
\end{equation}
respectively, where, $p<r\leq n$, i.e., over-complete. We assume in this discussion $\rank(A)=p$.

Due to implicit sparsity constraint, the objective of SSRl can be expressed as
\begin{equation}\label{equ:original sparse A-DH compact}
\min_{D\in\mathbb{R}^{p\times r},X\in\mathbb{R}^{r\times n}}
\|A-DX\|^2_F,\,\st\,XX^T=I,\,X\,\text{sparse},
\end{equation}
where $r\leq p+1$, i.e., usually under-complete.

Fixing $X$, $D^*=AX^T(XX^T)^{-1}$. Substituting it into the
objectives, all of them reduce to
\begin{equation}\label{equ:general sparse A-AHH}
\min_{X} \|A-AX^T(XX^T)^{-1}X\|^2_F,\,\st\,X\,\text{sparse},
\end{equation}
with various sparse constraints on $X$. Since $X^T(XX^T)^{-1}X$ is a
projection matrix formed by the row vectors of $X$, the above
objective can be interpreted as: to find a sparse basis so that it
covers the row space of data as much as possible. This is the common
point of the three SRs.

The sparse basis of SSRl is obtained via rotating the normalized PCs (attached with $\mb{1}^T/\sqrt{n}$), showing a
close tie to PCA. SSRl is also related to cluster analysis, due to the one-one correspondence between the normalized
PCs and the eigenvectors of a linear Laplacian matrix. The basis size of SSRl is limited by the number of PCs, or
equivalently the effective eigenvectors of Laplacian matrix. The over-completeness prevents OSRs from the reach of
spectral graph theory as well as PCA, when $r>p+1$. We do not know which vectors to draw from the void subspace that is
complementary to the data row-space, so that attaching to the PCs and via rotation, they turn into sparse basis.

The dilemma is circumvented by resorting to the dictionary representation. Treating the dictionary and codes
independently, OSRs obtain solution by alternate optimization. This strategy depends on a random initialization of
the dictionary, which can be seen as replacing drawing vectors from the complementary row-subspace. The initialization
is usually done via randomly drawing samples from the data set. Expressing $D=AX^T(XX^T)^{-1}$, it is still not clear
what the corresponding sparse codes are. However, if the dictionary is initialized with cluster centers, e.g., of
K-means, the sparse codes are obviously the corresponding indicator vectors. Then following the alternate
optimization, $D$ is fixed, the extreme sparse codes are relaxed and updated to soft ones, via some sample-wise sparse
coding algorithms, e.g., OMP \cite{pati1993orthogonal}, BP \cite{chen1998atomic}, and then $D$ is updated and so on.\footnote{In KSVD, when the maximal $\ell_0$-norm of codes is constrained to 1, KSVD reduces to vector quantization
\cite{Gersho1992Vector}, and if further the nonzero value is constrained to be 1, it reduces to exactly K-means, as
\cite{aharon2006img} had analyzed. The reverse process generalizes K-means to KSVD \cite{aharon2006img}.} Notice that,
K-means itself belongs to the SSR framework, while its solution adopts the alternate optimization. Unlike Rcut,
K-means can obtain clusters more than the number of PCs. The cluster centers of K-means are usually initialized with random samples. Combining the two solution processes together, the above strategy is in fact K-means+OSR,\footnote{Similar
strategies of clustering+OSR have been used in the proofs of exact dictionary recovery \cite{Agarwal2013Exact}, showing
some underlying relation between OSR and cluster analysis. } in hindsight, it is interesting to find their counterpart
PCA+SSR. At this point, the close ties between these methods are clear. A diagram is illustrated in
Figure~\ref{fig:diagram}.

Further more, the four methods can be eventually unified at an extreme point, $r=n$. At $r=n$ there is no choice, the
whole complementary subspace has been included. A rotation of the full basis can lead to the ultimate sparse basis: the
identity matrix. In this case, the dictionary consists of the whole data sets, with each sample serving as an
atom/cluster center. They must be the optimal solution of OSRs and K-means too. This seemly useless extreme sparse
representation had found its application in face recognition \cite{wright2009robust}.

Finally, we comment that attributing to the random initialization, the formal relation of OSRs to cluster analysis is
less obvious than that of SSRl, except a special case of KSVD; and the clustering performance of KSVD is not
satisfactory as the experiments demonstrated. Nevertheless, the basis size SSRl can handle is limited, and OSRs had
demonstrated impressive performance on signal denoising and reconstruction \cite{elad2010role}. As sparse representation
methods, SSRl and OSRs occupy different ranges and complement each other.

\subsection{Relations between Kernel PCA, Manifold Learning, and SSR}
First, we will show that kernel PCA (KPCA) \cite{scholkopf1998nonlinear} can be incorporated into the framework of
SSR. The underlying reason is simple: PCA constitutes a part of SSR framework,
and KPCA can be seen as a PCA, with original data
replaced by high-dimensional features. Next, the relation is extended to manifold learning, including multidimensional scaling (MDS) \cite{cox2001multidimensional}, Isomap \cite{tenenbaum2000global}, and locally linear embedding (LLE) \cite{roweis2000nonlinear}.

Let $G$ be a kernel matrix, and assume to find normalized PCs, the objective of KPCA is
\[
\max_{X}\,
\tr\{X(I-\frac{1}{n}\mb{1}\mb{1}^T)G(I-\frac{1}{n}\mb{1}\mb{1}^T)X^T\},\;\st\,XX^T=I.
\]
Then we have
\begin{theorem}
There is a sparse representation of KPCA: $\begin{bmatrix}
\frac{1}{\sqrt{n}}\mb{1}^T\\
V_{1:r-1}\end{bmatrix}=RH$, where $V_{1:r-1}$ are the eigenvectors
of KPCA.
\end{theorem}

\begin{proof}
Assume LE uses the same kernel matrix,
and the original data $A$ is not mean-removed, then PCA, KPCA,
and LE can be seen as working with the following three similarity
matrices respectively: $W_{PCA}=\beta
\mb{1}\mb{1}^T+\bar{A}^T\bar{A}$, $W_{KPCA}=\beta'
\mb{1}\mb{1}^T+G_{KPCA}$, and $W_{LE}= G$, where $\bar{A}=
A(I-\frac{1}{n}\mb{1}\mb{1}^T)$, $G_{KPCA}=
(I-\frac{1}{n}\mb{1}\mb{1}^T)G(I-\frac{1}{n}\mb{1}\mb{1}^T)$, and
$\beta$, $\beta'$ ensure the nonnegativity of the corresponding
similarity matrices. The solution of PCA is obtained from the
eigenvectors of $\bar{A}^T\bar{A}$, while that of KPCA is
obtained from the eigenvectors of $G_{KPCA}$. By the correspondence
between $\bar{A}^T\bar{A}$ and $G_{KPCA}$, following the conversion
from PCA to SSRl, the assertion is established.
\end{proof}

It implies KPCA is inherently related to cluster analysis too. We may wonder to what extent does KPCA reveal cluster structure, and how the clustering performance will be when using Scut. These questions can be answered by the ideal graph condition. Following the analysis of SSRl, we expect the similarity matrix used by KPCA meets the condition better than that used by PCA while worse than that used by LE. Though not reported, our experiments validated this. Besides, we mention that a significant drawback of KPCA is that it demands a large storage when dealing with large data sets, due to the dense kernel matrix.

Once KPCA has jointed the SSR framework, a series of manifold learning \cite{Ham2004A} or graph embedding
\cite{Yan2007Graph} methods, including MDS, Isomap, and LLE, subsequently share extended relations to SSR. It is
because, under a kernel view, these methods can be written into the form of KPCA, as is well-known in the literature
\cite{Ham2004A, bengio2004out, Yan2007Graph}. In particular, it had been proven that classical MDS using Euclidean
distance is exactly equivalent to PCA \cite{williams2002connection}. For metric MDS, Isomap (a special metric MDS that
uses geodesic manifold distances), and LLE, relations to KPCA exist, although less rigorous, due to the kernel
functions of their constructed ``kernel matrices'' are ambiguous \cite{Ham2004A}.

Finally, it should be clarified that when we talk about the relations, it does not mean that dimensionality reduction or manifold learning is always related to SSR or cluster analysis. When the data do not exhibit cluster structure, these methods work as usual, and the codes preserve the data proximity. When clusters emerge in the data distribution, noisy or not, the codes are intrinsically embedded with cluster structure. Depending on two things, the structure becomes explicit: the attachment of vector $\mb{1}$ and a proper rotation. We believe that for most real-world data, some underlying clusters naturally exist, more or less, and in a hierarchical way. Therefore, the codes must generally contain some cluster information, and the link of dimensionality reduction or manifold learning methods to SSR and cluster analysis should be commonplace. This explains why in so many paper focusing on manifold learning, the codes, when reduced to two dimensions for visualization, frequently exhibit a triangle shape (three clusters, cf. Figure~\ref{fig:eigenvector_points}). Plenty of examples can be found, e.g., in
\cite{Saul2003Think, Silva2003Global, saul2006spectral}.

\subsection{Relations between Subspace Clustering, LLE, and SSRl}
A series of work in subspace clustering, represented by SSC \cite{elhamifar2013sparse}, low-rank representation (LRR) \cite{liu2013robust}, and the manifold learning method LLE \cite{roweis2000nonlinear} center around the seemingly trivial data representation form $A\approx AZ$, which has been repeatedly interpreted in PCA, K-means, and SSR. We will show that SSC, LRR, LLE, and SSRl are related to each other, and the ideal graph condition is a special case of the ideal working condition of subspace clustering.

In cluster analysis, besides the clusters formed by spatial proximity, there is another circumstance that the clusters are formed by a union of linear subspaces, with one subspace corresponding to one cluster. The latter case is the subject of subspace clustering \cite{Vidal2011Subspace}, which has found applications in, e.g., motion and image segmentation, face recognition \cite{elhamifar2013sparse, liu2013robust}. SSRl/Scutl may belong to this domain too. The ideal graph condition of linear version requires that the clusters should be orthogonal.\footnote{In this section, for SSRl, we consider the augmented data. For SSC and LRR, we focus on their basic models.} This is a special case of the ideal working condition of subspace clustering, i.e., the subspaces being linearly independent.

The first step of SSC, LRR, LLE, and SSRl can be viewed as finding some code matrix to represent the
data matrix with respect to a dictionary composed of the data matrix itself, i.e., solving $\|A-AZ\|_F^2$ with some constraint on $Z$. The problem is generally under-determined, with infinitely many solutions achieving zero error, e.g., the identity matrix and the Gram matrix $V^TV$, where $V$ is the full normalized PCs of data. SSC finds a sparse solution with the constraint that the diagonal being zero, i.e., forbidding self-representation. LRR finds a low-rank solution via minimizing the nuclear norm of $Z$, which turned out to be $V^TV$ uniquely. LLE finds a sparse solution such that the nonzero entries of each column are constrained to the neighborhood of the associated point, and they sum to 1, i.e., each point is approximated locally by its neighborhood. For SSRl, we have $A\approx A(H^TH)$. Note $H^TH=V_{1:r}^TV_{1:r}$, which is the Gram matrix of $r$ PCs. SSC and LRR can achieve zero reconstruction error while SSRl is optimal within rank-$r$ limitation.

In the second step, SSC, LRR, and SSRl aim at clustering, while LLE is oriented to dimensionality reduction. When the subspaces are independent, \cite{elhamifar2013sparse} (excluding minor factors) and \cite{liu2013robust} proved their solutions are block-diagonal, if the data are arranged orderly. That is each point can be exactly reconstructed by its subspace. When the subspaces are orthogonal as the ideal graph condition requires, it is easy to see that all code matrices of the four methods become block-diagonal. Block-diagonal property is desirable, since it suits for the clustering purpose. In the second step of SSC and LRR, a similarity matrix is heuristically built based on $Z$, e.g., $W= |Z|+|Z|^T$ in SSC, then spectral clustering is employed to finish clustering. Since the code matrices are block-diagonal, so are the similarity matrices. It implies an ideal graph condition of kernel version is satisfied, making spectral clustering easy. For SSRl, this step can be omitted and clustering can be done by applying Scutl on $H$, since the code matrix $H^TH$ itself is a qualified similarity matrix, and the eigenvectors of the associated Laplacian matrix is exactly $H$. In the second step of LLE, it seeks low-dimensional codes $X\in \mathbb{R}^{r\times n}$, such that it replaces original data and the neighborhood structure of data is preserved: $\min_X$ $\|X-XZ\|_F^2$, s.t. $XX^T=I$. For SSRl, this step can be omitted too, since $H$ itself is the codes and also the solution of LLE's objective when $Z=H^TH$.

Although LLE is designed for dimensionality reduction, by the same principle of SSC and LRR, it can be used for clustering too. Conversely, following LLE, SSC and LRR can be used for dimensionality reduction, regardless of the clusters. But only SSRl covers both domains in a natural way. We again witness the ambiguous borders between dimensionality reduction, sparse representation, and cluster analysis.

\section{Conclusion and Future Work}\label{sec:future work}
In this paper, we have established a spectral graph-based framework unifying dimensionality reduction, cluster analysis, and sparse representation. It includes the inherent relations between PCA, K-means, LE, Rcut, SSR, and extended relations to OSRs, manifold learning, and subspace clustering. Based on the unification of PCA, K-means, LE, Rcut, a new sparse representation, SSR, has been derived. An efficient algorithm, NSCrt, has been developed to solve the sparse codes of SSR. As an application of SSR, the new clustering algorithm Scut achieved the state-of-the-art performance in the spectral clustering family.

The future work may include as follows.
First, the theoretical analysis of whether NSCrt has the ability to recover underlying solution remains an open
question. It is essentially a dictionary recovery problem. Recent results have shed light on it, e.g.,
\cite{Spielman2012Exact,Agarwal2013Exact,Arora2014New}. Especially, \cite{Spielman2012Exact} deals with square
dictionary, and they proved that in noise-free case the dictionary can be recovered with high probability when
each column of the codes has at most $O(\sqrt{r})$ nonzeros and there are at least $\Omega(r^2\log^2 r)$ samples.
It matches our case closely. Second, this paper remains in the unsupervised learning domain. It is interesting to find out whether the supervised classification can be incorporated into the spectral graph framework. Finally, in view of the relations of SSR
to many other methods, as well as the nice properties it has, we envision that SSR may have other potential applications, especially the out-of-sample case.

% Acknowledgements should go at the end, before appendices and references

%\acks{We would like to acknowledge support for this project from the
%National 973 Program (2013CB329500), and the Program for New Century
%Excellent Talents in University (NCET-13-0521). }


%\newpage

\appendix

\section{An Interpretation for Kernel Scut}
\label{app:scut}

In SSRk, Scut can be interpreted as clustering data according to the
smallest included angle (or distance) between the cluster centers
and the virtual data. By Theorem \ref{theo:kernel ssr}, the atoms of
$\hat{D}$ are cluster centers that are near-orthogonal and have
similar lengths. In this case, it can be proved that
\begin{equation}
\begin{aligned}
&\mathop{\arg\max}_k\,\cos\,(\theta(\hat{D}_k,\tilde{A}_i))\\
=&\mathop{\arg\max}_k\,\hat{D}_k^T\tilde{A}_i/(\|\hat{D}_k\|_2\cdot\|\tilde{A}_i\|_2)\\
\approx &\mathop{\arg\max}_k\, \hat{D}_k^T\tilde{A}_i\qquad\qquad\qquad\qquad\quad\, (\text{since } \|\hat{D}_k\|_2\approx \lambda_n^{\frac{1}{2}})\\
\approx &\mathop{\arg\max}_k\,
[\lambda_n^{\frac{1}{2}}R_k^T,\mb{0}](\lambda_n
I-\Lambda)^{\frac{1}{2}}V_i\qquad (\text{by (\ref{equ:kernel D}) and definition of $\tilde{A}$}) \\
=&\mathop{\arg\max}_k\, \lambda_n^{\frac{1}{2}}R_k^T(\lambda_n
I-\Lambda_{1:r})^{\frac{1}{2}}(V_{1:r})_i\\
\approx &\mathop{\arg\max}_k\,\lambda_n R_k^T RH_i\qquad\qquad\qquad\quad\;\, (\text{by $\Lambda_{1:r}\approx \mb{0}$ and $V_{1:r}=RH$})\\
= &\mathop{\arg\max}_k\,H_{ki},
\end{aligned}
\end{equation}
which implies Scut performs clustering based on the smallest
included angle.

Besides, since
$\mathop{\arg\min}_k\,\|\tilde{A}_i-\hat{D}_k\|_2^2=\|\tilde{A}_i\|_2^2+\|\hat{D}_k\|_2^2-2\hat{D}_k^T\tilde{A}_i\approx
\mathop{\arg\max}_k\,\hat{D}_k^T\tilde{A}_i$, Scut can also be
interpreted as clustering data according to the smallest distance to
the cluster centers.

%{\small
\bibliographystyle{plain}
\bibliography{sparse}
%}

\end{document}